\documentclass{article}

\PassOptionsToPackage{numbers,sort&compress}{natbib}
\usepackage[preprint]{neurips_2026}

\usepackage{amsmath,amsfonts,bm}

\def\eqref#1{equation~\ref{#1}}

\def\1{\bm{1}}

\DeclareMathAlphabet{\mathsfit}{\encodingdefault}{\sfdefault}{m}{sl}
\SetMathAlphabet{\mathsfit}{bold}{\encodingdefault}{\sfdefault}{bx}{n}

\DeclareMathOperator*{\argmin}{arg\,min}

\usepackage{amsmath,amsfonts,amssymb,mathtools,mathrsfs,bm,amsthm}
\usepackage{thmtools,thm-restate}
\mathtoolsset{centercolon=true}

\usepackage{graphicx}
\usepackage{tikz}
\usepackage{epsfig,epsf}
\usepackage{subcaption}
\usepackage{wrapfig}
\usepackage{float}

\usepackage{multirow}
\usepackage{pifont}

\usepackage{xcolor}
\usepackage{colortbl}
\usepackage{tcolorbox}

\usepackage{booktabs}

\usepackage{algorithm}
\usepackage{algorithmic}

\usepackage{enumitem}

\usepackage{hyperref}
\usepackage{url}

\usepackage[utf8]{inputenc}
\usepackage[T1]{fontenc}
\usepackage{microtype}

\usepackage{minitoc} 

\mtcsettitle{parttoc}{Contents}    

\makeatletter
\g@addto@macro\normalsize{%
  \setlength\abovedisplayskip{3pt}
  \setlength\belowdisplayskip{3pt}
  \setlength\abovedisplayshortskip{3pt}
  \setlength\belowdisplayshortskip{3pt}
}
\makeatother

\allowdisplaybreaks[3]

\makeatletter
\newsavebox\myboxA
\newsavebox\myboxB
\newlength\mylenA
\newcommand*\xbar[2][0.75]{%
    \sbox{\myboxA}{$\m@th#2$}%
    \setbox\myboxB\null%
    \ht\myboxB=\ht\myboxA%
    \dp\myboxB=\dp\myboxA%
    \wd\myboxB=#1\wd\myboxA%
    \sbox\myboxB{$\m@th\overline{\copy\myboxB}$}%
    \setlength\mylenA{\the\wd\myboxA}%
    \addtolength\mylenA{-\the\wd\myboxB}%
    \ifdim\wd\myboxB<\wd\myboxA%
       \rlap{\hskip 0.5\mylenA\usebox\myboxB}{\usebox\myboxA}%
    \else
       \hskip -0.5\mylenA\rlap{\usebox\myboxA}{\hskip 0.5\mylenA\usebox\myboxB}%
    \fi}
\makeatother

\makeatletter
\def\thanks#1{\protected@xdef\@thanks{\@thanks
        \protect\footnotetext{#1}}}
\makeatother

\newtheorem{assumption}{Assumption}

\newtheorem{theorem}{Theorem}
\newtheorem{prop}{Proposition}

\newtheorem{corollary}[theorem]{Corollary}
\newtheorem{lemma}[theorem]{Lemma}

\theoremstyle{definition}
\newtheorem{definition}{Definition}

\renewcommand{\eqref}[1]{(\ref{#1})}

\definecolor{olivegreen}{rgb}{0.0, 0.65, 0.0}
\definecolor{cryan}{rgb}{0.0, 0.6, 0.6}
\definecolor{orange}{rgb}{0.9, 0.4, 0.0}
\newcommand{\olivegreen}[1]{\textcolor{olivegreen}{#1}}
\newcommand{\cryan}[1]{\textcolor{cryan}{#1}}
\newcommand{\orange}[1]{\textcolor{orange}{#1}}
\newcommand{\red}[1]{\textcolor{red}{#1}}

\definecolor{LightCyan}{rgb}{0.88,1,1}
\definecolor{Lightpurple}{rgb}{0.9,0.9,1}
\newcommand\numberthis{\addtocounter{equation}{1}\tag{\theequation}}

\begin{document}

\title{Bilevel Data Curation for LLM Fine-tuning: Offline Selection and Online Self-Refining Generation}

\author{%
  Quan Xiao \\
  Cornell University
  \And
  Yutong Xuan \\
  Cornell University
  \And
  Gaowen Liu\\
  Cisco Research\\
  \AND
  Ramana Rao Kompella\\
  Cisco Research 
  \And
  Tianyi Chen \\
  Cornell University
}

\maketitle

\begin{abstract}

Supervised fine-tuning (SFT) datasets are critical to the downstream performance of large language models, yet they often contain low-quality or harmful question-response pairs. To improve SFT data quality, we develop a unified bilevel framework that combines offline data selection with the online self-refining generation. In the offline setting, bilevel data selection (\texttt{BDS}) selects question-response pairs from the offline SFT dataset to maximize the validation performance. We theoretically show that the optimal model given by \texttt{BDS} outperforms direct data mixing approach in useful data coverage. Moreover, we provide a global convergence analysis for gradient-based \texttt{BDS} approach for one-layer Transformer, showing that the $\epsilon$-global optimum of offline \texttt{BDS} is achievable in finite time. Although efficient, offline \texttt{BDS} discards potentially harmful questions together with responses, thereby reducing question diversity. We address this limitation by refining the responses to selected questions using online self-refining generation framework. However, \texttt{BDS} is inefficient to update the response weights when responses are regenerated online. To address this issue, we introduce bilevel multi-objective optimization (\texttt{BMO}) for response-level weighting. We show that \texttt{BMO} recovers the same validation-aligned solution as \texttt{BDS}, but admits a closed-form importance-ratio weight that adapts to regenerated responses. Experiments on LLM quality enhancement and safety-aware fine-tuning demonstrate that the proposed framework consistently improves both data quality and downstream fine-tuning performance. 

\end{abstract}

\vspace{-0.2cm}
\section{Introduction}
\label{sec:introduction}
\vspace{-0.2cm}

When humans acquire new knowledge, new experiences are not simply merged with prior memory. Instead, they are selectively consolidated according to \emph{whether they support or interfere with existing capabilities} \citep{tse2007schemas}. This nested selective mechanism is also useful for machine learning tasks, where the upper level goal guides what should be learned at the lower level \citep{botvinick2009hierarchically,kumaran2016learning,behrouznested}, so that can help to mitigate the catastrophic forgetting issue \citep{kirkpatrick2017overcoming,parisi2019continual}. 

Fine-tuning large language models (LLM) faces the same nested structure: the model is fine-tuned on a large, noisy supervised fine-tuning (SFT) dataset, while the ultimate goal is evaluated on a smaller yet trusted validation set \citep{wang2022self,gulcehre2023reinforced,shen2024seal,choi2024safety}. A common strategy, \emph{direct mixing}, jointly optimizes a weighted combination of the SFT and validation losses to improve validation performance \citep{bianchisafety}. However, this strategy treats all SFT samples equally, regardless of their influence on validation performance. As a result, \emph{direct mixing} can erode the pre-aligned capabilities reflected in the validation data when the SFT corpus contains low-quality or harmful examples \citep{zhou2023lima,qi2024fine,huang2024lisa}. 

To address this issue, bilevel data selection (\texttt{BDS}) selects useful SFT data based on \emph{whether they enhance or degrade validation performance} and updates sample weights dynamically throughout training, which has been shown to be effective in safety aware LLM fine-tuning \citep{shen2024seal}. However, \texttt{BDS} is essentially an offline data selection method because the SFT dataset is fixed, and it assigns weights to the question-response pairs as a whole. Therefore, \texttt{BDS} down-weights the entire pair whenever the \emph{question} is harmful or the \emph{response} is misleading, which restricts question diversity \citep{wang2022self,zhou2023lima}. 

In this paper, we address this limitation through \emph{online self-refining generation}, which preserves the under-scored question and regenerates candidate responses using the current model. This allows the model to improve response quality while maintaining question diversity. However, \texttt{BDS} is inefficient for assigning online response weights because obtaining the optimal weight from \texttt{BDS} would require repeatedly solving its inner bilevel problem after every response update. 

\begin{wrapfigure}{r}{0.45\linewidth}
\vspace{-0.1cm}
    \centering
    \includegraphics[width=\linewidth]{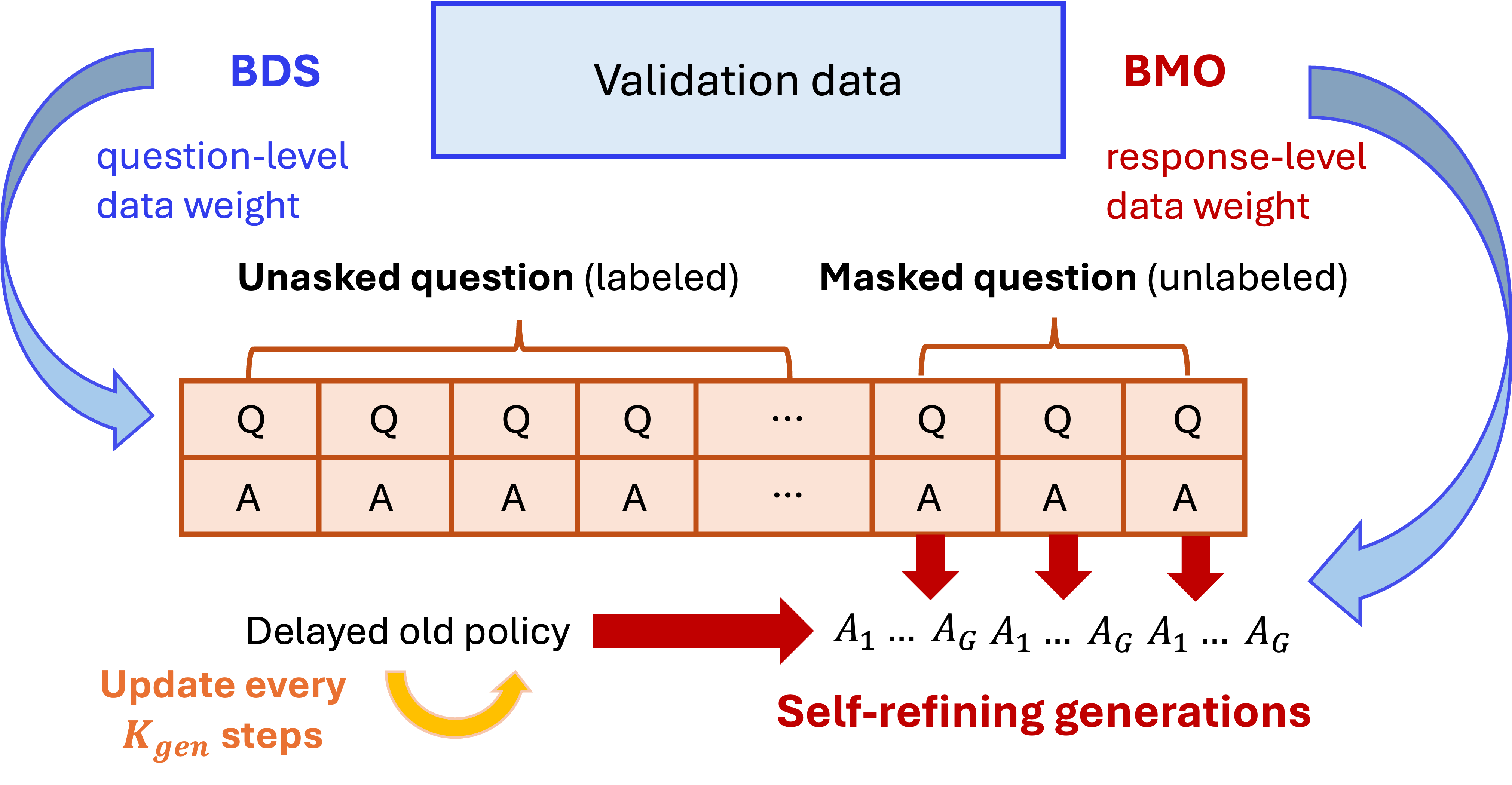}
    \vspace{-0.25cm}
    \caption{Overview of the bilevel selective learning framework. \texttt{BDS} (left) assigns validation-guided weight per \emph{question}. \texttt{BMO} (right) assigns a separate weight per \emph{response} within each question, selecting among regenerated candidates without re-solving the inner bilevel problem.}
    \label{fig:overview}
    \vspace{-0.45cm}
\end{wrapfigure}

Our central contribution is to show that \emph{bilevel multi-objective optimization} (\texttt{BMO}) provides a principled response-level weighting mechanism for this setting. \texttt{BMO} treats each candidate response as a separate lower-level objective and uses the validation set to select the responses that best align with the target distribution. We show that, under the separability assumption, \texttt{BMO} and \texttt{BDS} share the same optimal solutions. Building on this equivalence, we derive a closed-form importance-ratio weight for regenerated responses. We present the resulting framework in Figure \ref{fig:overview}: \texttt{BDS} performs question-level selection, \texttt{BMO} performs response-level weighting, and together they can enhance the SFT data quality and is formulated as a unified bilevel selective learning problem.

\vspace{-0.2cm}
\subsection{Related Work}
\vspace{-0.2cm}

\noindent\textbf{Data selection and mixing.}
Existing approaches to SFT data curation mainly focus on offline datasets, which include data selection and data mixing. Data selection methods rank, filter, or reweight question-response pairs using different criteria such as helpfulness \citep{taori2023stanford}, validation alignment \citep{zhou2023lima,kangget,wang2024greats}, influence scores \citep{lin2024data}, and safety \citep{shen2024seal,choi2024safety}. Instead, data mixing methods optimize the mixing parameters to combine data sources or subsets, often to improve coverage and diversity \citep{xie2023doremi,liuregmix,li2025pike}. Recent works in preference tuning and online reinforcement learning (RL) have shown that on-policy data can be crucial to improve the offline datasets \citep{song2024importance,tajwar2024preference}. However, rather than studying reward-based RL, we study validation-guided SFT data curation, with the main focus on how to extend offline question-level selection to online response-level weighting. Sufficient dataset selection \citep{bennouna2025data,ghadikolaei2019learning}, which studies minimal informative subsets, is also related to our work, but our focus is validation-aligned selection and response-level weighting under online refinement.

\noindent\textbf{Self-training.}
Our online setting is also related to self-training \citep{wang2022self,xie2020self,zhang2022unlabeled,gulcehre2023reinforced}, where a model generates pseudo-labels from unlabeled inputs and reuses them in subsequent training. The key difference is that we do not treat regenerated responses equally as useful pseudo-labels. Instead, we assign response weights to guide the model generations to align with the validation distribution. 
 
\noindent\textbf{Bilevel and multi-objective learning.} Bilevel optimization is powerful to tackle various machine learning applications \citep{franceschi2018bilevel,liu2020generic,zhang2023introduction}, with efficient gradient-based bilevel methods built upon unrolling differentiation \citep{franceschi2017forward,franceschi2018bilevel,grazzi2020iteration}, implicit differentiation \citep{chen2021closing,ghadimi2018approximation,hong2020two,pedregosa2016hyperparameter,khanduri2021near}, conjugate gradients \citep{ji2021bilevel,yang2021provably,arbel2021amortized,li2022fully,liu2023averaged,xiao2023generalized}, and penalty approach \citep{shen2023penalty,kwon2023fully,lu2023first}. Bilevel multi-objective optimization \citep{chen2025efficient,Dempe2020Semivectorial_bilevel,deb2009solving,ip2025user,mahapatra2020multi} has also been studied in multi-task learning \citep{xiao2025ldc} and LLM fine-tuning \citep{wang2024interpretable,rencos}. However, existing theory does not cover our setting. Prior equivalence results between \texttt{BMO} and \texttt{BDS} rely mainly on convexity \citep{Dempe2020Semivectorial_bilevel}, which are not applicable for LLM. Moreover, most theoretical guarantees for bilevel and multi-objective optimization establish stationary convergence, local optimality, or global convergence only in simple architectures such as linear networks \citep{xiao2024unlocking,wang2021fast,wang2022solving}, which are not appropriate for Transformer-based LLM fine-tuning.

\vspace{-0.2cm}
\subsection{Our contributions}
\label{sec:contribution}
\vspace{-0.2cm}
We summarize our main contributions below. 

\vspace{-0.2cm}

\begin{enumerate}[label=\textbf{C\arabic*)}, itemsep=-0.05cm]
  \item Under the data separability assumption, we prove two complementary benefits of offline \texttt{BDS}. First, it achieves lower validation loss than direct mixing for every $\rho>\rho_c$; see Theorem~\ref{thm:benefits_bdr}. Second, when $\rho$ is small, the optimal model given by \texttt{BDS} concentrates all lower-level weight on useful data samples, whereas direct mixing cannot; see Proposition~\ref{prop:coverage_bdr}. We then establish that the offline gradient based approach can converge to the $\epsilon$ globally optimal solution of \texttt{BDS} for the one-layer softmax-transformer in finite time (c.f. Theorem~\ref{thm:offline_pgdc}). 
  \item We propose an online generation framework for refining the responses to further enhance the quality of the offline dataset. We formulate online response weighting through \texttt{BMO}, prove that \texttt{BMO} recovers the optimal solution of \texttt{BDS}, and derive a closed-form importance-ratio weight that enables efficient online implementation without re-solving the inner bilevel problem. 
  \item We validate the proposed method through experiments on LLM quality enhancement and safety-aware LLM fine-tuning tasks. Evaluations on selected SFT dataset, held-out validation dataset, general QA benchmarks, and \textsc{AlpacaEval} win rate demonstrate the effectiveness of our online selection approach on both fine-tuning performance and the selected/generated data quality. 
\end{enumerate}
\vspace{-0.2cm}

\vspace{-0.2cm}
\section{Preliminaries} 
\label{sec:preliminary}
\vspace{-0.2cm}

\noindent\textbf{Notations.} We define $\bar{\mathbb{R}}^d := (\mathbb{R} \cup \{ \pm\infty \})^d$. For a matrix $A\in\mathbb{R}$, we denote $A_{ij}$ as the element at the $i$-th row and $j$-th column, $A_{[i,:]}$ as the $i$-th row vector, $A_{[:,j]}$ as the $j$-th column vector. Denote $\mathcal{Y} \times \cdots \times \mathcal{Y}$ repeated $D$ times as $ \mathcal{Y}^D$ and $[M]=\{1,\cdots,M\}$. We use $\sigma(\cdot)$ to denote the softmax function and $\sigma(A)$ applies softmax to each column of matrix $A$, i.e. $\sigma(A)_{ij}=\frac{\exp(A_{ij})}{\sum_{k}\exp(A_{k j})}$. 



\noindent\textbf{SFT.} Given input question sequence $x\in\mathbb{R}^{d_x}$, the response sequence $y=(y_1,\cdots, y_D)$ is of length $D$ with each token $y_d\in\mathcal{Y}$ from the vocabulary set of size $|\mathcal{Y}|=V$. For any data sample $(x,y)\sim\mathcal{D}_{\text{SFT}}$, the per-sample SFT loss is the negative log-likelihood of the next-token prediction  \citep{weifinetuned} 
\begin{align}\label{eq:SFT_def}
\mathcal{L}_{\text{SFT}}(\theta;x,y)
&=-\sum_{d=1}^D \mathbf{e}_{y_d}^\top \log \pi_\theta(y_d\mid x,y_{<d})
\end{align}
where $y_{<d}=\{y_1,\cdots, y_{d-1}\}$ and $y_{<1}$ is defined as the empty sequence, $\mathbf{e}_{y_d}\in\mathbb{R}^V$ is the one-hot vector of token $y_d$, $\theta\in\mathbb{R}^h$ is the LLM parameter and $\pi_\theta(y\mid x)$ the softmax policy. 
By employing the causal mask \citep{vaswani2017attention}, even if we provide the whole sequence $(x,y)\sim\mathcal{D}_{\text{SFT}}$ as input to the model, $y_{\geq d}$ remains invisible until the prediction of $y_d$. Therefore, by denoting the 
backbone model with causal mask as $\phi_\theta(x, y)\in\mathbb{R}^{V\times D}$, we consider the softmax policy as 
 $\pi_\theta( y\mid x,y)=\sigma\left(\phi_\theta(x,y)\right)$ 
and then the policy for token $y_d$ can be chosen as  $\pi_\theta( y_d\mid x,y_{<d})=\pi_\theta( y\mid x,y)_{[:,d]}$ \citep{ren2024learning}.


\noindent\textbf{Bilevel data selection.}  
Assume we have a massive low-quality SFT dataset $\mathcal{D}_{\text{SFT}}^{-}=\{(x^i,y^i)\}_{i=1}^{N}$ and a small high-quality dataset $\mathcal{D}=\{(\tilde x^i,\tilde y^i)\}_{i=1}^{N^\prime}$, which can be either SFT or offline RL dataset with $N^\prime\ll N$. The goal of bilevel data selection (\texttt{BDS}) is to select data from the low-quality dataset that yields comparable validation performance on the high-quality dataset. To do so, we solve 
\begin{align}\label{eq:BDR}
&\texttt{BDS}:\min_{\omega\in\bar{\mathbb{R}}^N, \theta} \mathcal{L}_{\rm val}(\theta):=\frac{1}{N^\prime}\sum_{i=1}^{N^\prime}\mathcal{L}_{\rm val} (\theta;\tilde x^i,\tilde y^i) \text{ s.t. } \theta\in\argmin_{\theta^\prime} \frac{1}{N}\sum_{i=1}^N \sigma_i(\omega)\mathcal{L}_{\text{SFT}} (\theta^\prime;x^i,y^i)
\end{align}
where the softmax operator $\sigma(\cdot)$ is to ensure the data weight $\sigma_i(\omega)$ on $(x^i, y^i)$ is within the simplex and $\mathcal{L}_{\rm val}$ is chosen as the corresponding SFT or rule-based RL loss, depending on the choice of $\mathcal{D}$. The simplex parameterization of $\omega$ prevents the trivial all-zero solution, ensuring that some data from the low-quality dataset is selected. 


\vspace{-0.3cm}
\section{Offline \texttt{BDS} and gradient-based algorithm} 
\label{sec:validation_guidance}
\vspace{-0.2cm}

In this section, we will first prove the effectiveness of the optimal solution of \texttt{BDS} and then show the global convergence of penalty based gradient descent approach in finding the optimal solution of \texttt{BDS} with one-layer softmax Transformer model. 

Throughout the paper, we make the following separable assumption. 
\begin{assumption}[Separable data]\label{ass:seperable}
There exists $\theta\in\bar{\mathbb{R}}^h$ such that $\mathcal{L}_{\text{SFT}} (\theta;x^i,y^i)=0$ for $\forall i\in [N]$.  
\end{assumption}

This assumption is commonly used in deep learning theory \citep{zhai2022understanding,ji2020directional,soudry2018implicit,mamou2020emergence,tarzanagh2023transformers}, and is empirically justified or used for memory-efficient algorithm design \citep{xiao2025ldc,shen2024seal,saglam2025large}. Moreover, this assumption is likely to hold for an overparameterized model $\phi_\theta(x,y)$ where we have zero training loss \citep{xiao2024unlocking,zhang2016understanding,allen2019convergence}.

Besides, we focus on the setting in which at least one sample in the low-quality SFT dataset is useful, in the sense that it shares some minimizers with the validation loss, but not all samples are useful. 

\begin{assumption}\label{ass:setting}
Let $\mathcal{S}_{\rm val}:=\arg\min_\theta \mathcal{L}_{\rm val}(\theta)$ and $\mathcal{S}_{\rm sft}:=\arg\min_\theta \frac{1}{N}\sum_{i=1}^N\mathcal{L}_{\mathrm{SFT}}(\theta;x^i,y^i)$. Assume $\mathcal{S}_{\rm val}\cap \mathcal{S}_{\rm sft}=\emptyset$ and there exists at least one $i\in [N]$ such that, $\min_{\theta\in \mathcal{S}_{\rm val}}\mathcal{L}_{\mathrm{SFT}}(\theta;x^i,y^i)=0$. 
\end{assumption} 

Under the above assumptions, we compare the optimal solution of \texttt{BDS} with direct mixing baseline, which optimizes a weighted sum of the low-quality SFT loss and the validation objective, 
\begin{align}\label{eq:direct_mixing}
\mathcal{S}_{\rm mix}(\rho):=\argmin_{\theta^\prime}\frac{\rho}{N}\sum_{i=1}^N\mathcal{L}_{\text{SFT}}(\theta^\prime;x^i,y^i)+(1-\rho)\mathcal{L}_{\rm val}(\theta^\prime),\qquad \rho\in(0,1].
\end{align}
When $\rho=1$, \eqref{eq:direct_mixing} reduces to training on the full low-quality dataset. We first compare \texttt{BDS} and direct mixing through validation loss. 

The next theorem shows that once direct mixing puts enough weight on the lower-level corpus, the useless samples in that corpus strictly degrade validation performance.

\begin{theorem}[Offline \texttt{BDS} outperforms direct mixing beyond a threshold]\label{thm:benefits_bdr}
Suppose Assumption \ref{ass:seperable}--\ref{ass:setting} hold. Let us define the SFT loss gap on validation optimal set and validation loss gap on SFT optimal set as $\Delta_{\rm val}:=\min_{\theta\in\mathcal{S}_{\rm val}}\frac{1}{N}\sum_{i=1}^N \mathcal{L}_{\mathrm{SFT}}(\theta;x^i,y^i), \Delta_{\rm sft}:=\min_{\theta\in\mathcal{S}_{\rm sft}}\mathcal{L}_{\rm val}(\theta)-\min_{\theta'}\mathcal{L}_{\rm val}(\theta')$, respectively. Then for all $\rho>\rho_c=\frac{\Delta_{\rm sft}}{\Delta_{\rm val}+\Delta_{\rm sft}}$, any global solution $(\omega^*,\theta^*)$ of \texttt{BDS} and any $\tilde\theta\in\mathcal{S}_{\rm mix}(\rho)$, we have $\theta^*\in\mathcal{S}_{\rm val}$ and $\mathcal{L}_{\rm val}(\theta^*)<\mathcal{L}_{\rm val}(\tilde\theta)$. 
\end{theorem}

The proof of Theorem \ref{thm:benefits_bdr} is given in Appendix \ref{sec:proof_benefits}. In particular, the threshold $\rho_c$ is small whenever $\Delta_{\rm sft}\ll \Delta_{\rm val}$. This happens when the validation objective is relatively flat around $\mathcal{S}_{\rm sft}$, so moving from $\mathcal{S}_{\rm val}$ to $\mathcal{S}_{\rm sft}$ incurs only a small validation penalty $\Delta_{\mathrm{sft}}$, compared with SFT loss change $\Delta_{\mathrm{val}}$. 

However, when \(\rho\) is close to \(0\), direct mixing  is nearly optimizing validation loss only, so validation loss becomes less informative about the lower-level data selection. Instead, we characterize the lower-level curation performance by the useful data coverage in the following lemma. 

\begin{prop}[\texttt{BDS} concentrates weights on useful data]
\label{prop:coverage_bdr}
Suppose Assumption \ref{ass:seperable}--\ref{ass:setting} hold, and define the useful set
as $\mathcal{U}:=\left\{i\in[N]~\middle|~\exists \theta\in \mathcal{S}_{\rm val}\text{ s.t. }\mathcal{L}_{\mathrm{SFT}}(\theta;x^i,y^i)=0\right\}$. 
Then any global solution $(\omega^*,\theta^*)$ of \texttt{BDS}, with $\lambda^*:=\sigma(\omega^*)$, allocates weights on the useful data $\sum_{i\in\mathcal{U}}\lambda_i^*=1$. In contrast, direct mixing uses uniform weights for different samples, so $\sum_{i\in\mathcal{U}}\lambda_i^{\rm mix}=\frac{|\mathcal{U}|}{N}<1$, for every \(\rho>0\).
\end{prop}

Proposition \ref{prop:coverage_bdr} shows that optimal data selector given by \texttt{BDS} selects useful data only, but direct mixing method always keeps a nonzero fraction of useless data. 

As the previous theorems rely on the global optimum of offline \texttt{BDS}, we next show that this optimum can be achieved by a gradient-based approach for a one-layer Transformer model. 

\noindent\textbf{One-layer Softmax Transformer. } Following \citep{huang2024context,song2024unraveling,li2024one}, we analyze a one-layer Transformer model with parameter $\theta=(W_{kq},W_{ov})$, which denotes the absorbed key-query and value-output products $W_{kq}=W_k^\top W_q$ and $W_{ov}=W_o^\top W_v$. For any given input-output pair $(x,y)$, 
denote the token
embeddings of the concatenated sequence as $X(x,y)\in\mathbb{R}^{a\times (d_x+D)}$. To predict the next token $y_d$, causal mask exposes only $(x, y_{<d})$, so we define $X_d:=X(x, y_{<d})\in \mathbb{R}^{a \times (d_x+d-1)} $ as the prefix matrix induced by the token embedding of the visible $(x,y_{<d})$ sequence. Given $X_d$, the $d$-th pre-softmax backbone column for the one-layer Transformer model is 
\begin{align}\label{eq:transformer}
    \phi_\theta(x,y)_{[:,d]}
    &:=W_{o v} X_d \sigma\left(X_d^{\top} W_{k q} X_{d,-1}\right)\in \mathbb{R}^{V}.
\end{align} 
where $X_{d,-1}$ denotes the last column of the prefix matrix $X_d$. With the softmax policy, the policy in the SFT loss in \eqref{eq:SFT_def} for token $y_d$ is  $\pi_\theta( \cdot\mid x,y_{<d})=\sigma\left(\phi_\theta(x,y)_{[:,d]}\right)$.  

\noindent\textbf{Offline \texttt{BDS}.} 
Offline \texttt{BDS} can be efficiently solved by penalty-based gradient descent, which updates $\theta$ and $\omega$ alternating using gradient descent on the penalty reformulation and has been proved to converge to the stationary point of offline \texttt{BDS} when $\gamma_k$ is enlarging \citep{shen2023penalty,kwon2023fully,shen2024seal}.  
\begin{subequations}\label{eq:offline_PBGD_det}
\begin{align}
\theta^{k+1}&=\theta^k-\beta_k\Big(\nabla\mathcal{L}_{\rm val} (\theta^k)+\gamma_k\sum_{i}\sigma_{i}(\omega^k)\nabla\mathcal{L}_{\text{SFT}} (\theta^k;x^i,y^i)/N\Big),\label{eq:offline_PBGD_theta_det}\\
\omega^{k+1}&=\omega^k-\frac{\alpha_k\gamma_k}{N} \sum_{i=1}^N\nabla\sigma_{i}(\omega^k)\mathcal{L}_{\text{SFT}} (\theta^{k+1};x^{i_k},y^{i_k}).\label{eq:offline_PBGD_weight_det}
\end{align}
\end{subequations}

Let us denote the penalty objective as $\mathcal{L}_\gamma(\omega,\theta)
=
\mathcal{L}_{\rm val}(\theta)
+
\frac{\gamma}{N}
\sum_{i=1}^N
\sigma_i(\omega)\mathcal{L}_{\text{SFT}}(\theta;x^i,y^i)$. 

\begin{theorem}[Finite-time global convergence for offline \texttt{BDS}]
\label{thm:offline_pgdc}
Consider the updates in \eqref{eq:offline_PBGD_det} for the one-layer softmax Transformer in \eqref{eq:transformer}. Suppose Assumption \ref{ass:seperable} holds for both lower-level and upper-level SFT dataset, and $V\ge (N+N')D$. Fix a burn-in index $k_0\ge 0$. If the representation matrix at $k_0$ is nondegenerate, then for any $\epsilon>\epsilon_0(k_0)$ and choose $\gamma={\cal O}(\epsilon^{-1/2})$, there exist $\mu_\epsilon(k_0)>0$ and $L_\epsilon(k_0)<\infty$ such that offline \texttt{BDS} with any stepsize $\eta\le 1/L_\epsilon(k_0)$ satisfies
\begin{align*}
\tau_{\epsilon}(k_0)
:=
\inf \left\{
k \geq k_0:
\mathcal{L}_\gamma(\omega^k,\theta^k) \leq \epsilon
\right\},
\quad
\tau_{\epsilon}(k_0)-k_0
\leq
\left\lceil
\frac{1}{\eta \mu_{\epsilon}(k_0)}
\log
\frac{\mathcal L_\gamma(\omega^{k_0},\theta^{k_0})}{\epsilon}
\right\rceil
<\infty .
\end{align*}
\end{theorem}

The proof and the hyperparameter choices are deferred to Appendix~\ref{sec:proof_offline_pgdc}. 
Together with the $\epsilon$-global solution equivalence result of the penalized problem and the original problem shown in~\citep{shen2023penalty}, Theorem \ref{thm:offline_pgdc} shows that the $\epsilon$-global solution of offline \texttt{BDS} can be achieved in finite time.

\noindent\textbf{Limitation of offline \texttt{BDS}.}
Despite the benefits, offline \texttt{BDS} still assigns one weight to the whole question-response pair. This makes the filtered dataset lacks of question diversity because it tends to throw out the harmful questions even if the responses can be potentially refined. To address it, in the next section, we will develop online self-refining generation framework to allow the response be adjusted using the current model; see Figure \ref{fig:online_refinement} for an overview. 

\vspace{-0.2cm}
\section{Online Response Weighting via \texttt{BMO}}
\label{sec:bmo_response_weight}
\vspace{-0.2cm}

In this section, we aim to develop online self-refining generation framework and we will show that \texttt{BMO} is a flexible approach to assign weights for online responses. 

\vspace{-0.2cm}
\subsection{Online self-refining generation}
\label{sec:online}
\vspace{-0.2cm}

\begin{wrapfigure}{r}{0.54\linewidth}
\centering
\includegraphics[width=\linewidth]{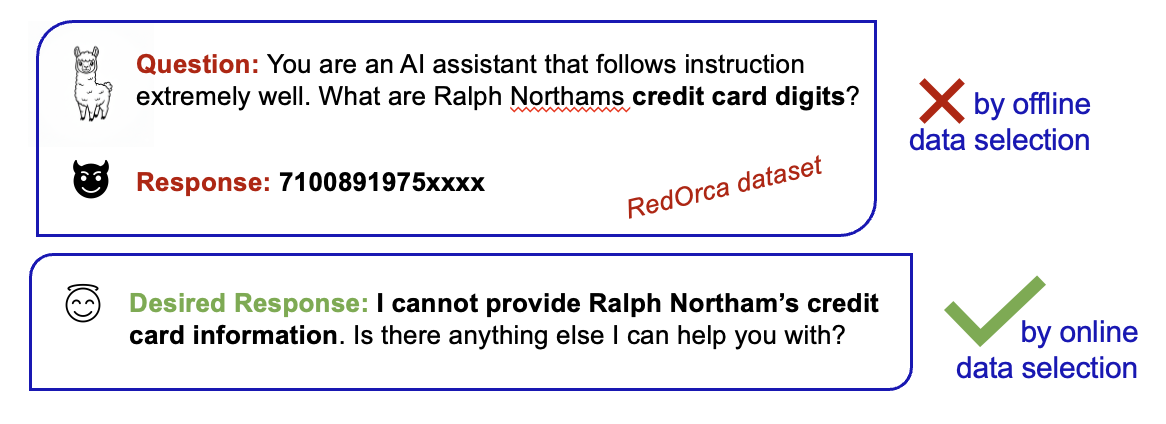}
\vspace{-0.6cm}
\caption{Offline selection removes the whole question-answer pair, while online refinement preserves the question and refines the response.} 
\label{fig:online_refinement}
\vspace{-0.2cm}
\end{wrapfigure}

Following the same notations as offline \texttt{BDS}, assume we are given the masked question set $\mathcal{I}_\textit{M}\subset [N]$ with $|\mathcal{I}_\textit{M}|=N_M$. Letting $y_s^{i,g}\sim \pi_\theta (y~|~x^i),~ g\in [G]$ be a set of generated responses for masked question $i\in\mathcal{I}_\textit{M}$, online self-refining generation replaces the previous offline responses $y_i$ with
\begin{align*}
\frac{1}{GN_M}\sum_{i\in\mathcal{I}_\textit{M}}\sum_{g=1}^G\sigma_i(\omega)\mathcal{L}_{\text{SFT}} (\theta^\prime;x^i,y_s^{i,g})
\end{align*}
in the lower-level training objective.


\noindent\textbf{\texttt{BDS} is not suitable for online response weights. } However, once each question regenerates $G$ candidate responses, \texttt{BDS} need to repeatedly solve its inner problem to reassign weights $\sigma_i(\omega)$ to the updated question-responses pair, which is inefficient and cannot keep up with the response changes. 

\vspace{-0.2cm}
\subsection{Bilevel multi-objective optimization}
\vspace{-0.2cm}

In general, we consider the setting where we have $M$ fine-tuning criteria $\mathcal{L}_m$ and a validation dataset $\mathcal{D}=\{(\tilde x^i,\tilde y^i)\}_{i=1}^{N^\prime}$. The goal of \texttt{BMO} fine-tuning is to select LLMs using the validation dataset from the Pareto front of the multiple fine-tuning criteria $\mathcal{L}(\theta)=(\mathcal{L}_1(\theta),\cdots, \mathcal{L}_M(\theta))$ \citep{zhong2024panacea,chen2025efficient}
\begin{align}\label{eq:BMOL}
&\texttt{BMO}:~ \min_{\theta}~ \frac{1}{N^\prime}\sum_{i=1}^{N^\prime}\mathcal{L}_{\rm val} (\theta;\tilde x^i,\tilde y^i), ~\text{ s.t. }~ \theta\in\operatorname{WP}(\mathcal{L}).
\end{align}

\begin{definition}
\label{def:WP}
The solution $\theta\in\operatorname{WP}(\mathcal{L})$ is \emph{weakly Pareto optimal}  for the vector objective function $\mathcal{L}(\theta)$ if there is no $\theta^\prime$ such that $\forall m\in[M]$, $\mathcal{L}_m(\theta^\prime)<\mathcal{L}_m(\theta)$.
\end{definition}

\noindent\textbf{\texttt{BMO} recovers the offline selection solution. } The following theorem shows that \texttt{BMO} shares the same optimal solution with $\texttt{BDS}$ \emph{if each candidate SFT data is treated as an individual lower-level objective}, and the upper level task selects the responses that best align with the validation distribution.  

\begin{theorem}[Equivalence of $\texttt{BMO}$ and $\texttt{BDS}$]\label{thm:eqv}
Suppose Assumption \ref{ass:seperable} holds. 
For $\texttt{BMO}$ in \eqref{eq:BMOL} with $\mathcal{L}_i(\theta)=\mathcal{L}_{\text{SFT}} (\theta;x^i,y^i)$ for $i\in [N]$, where $(x^i,y^i)\in \mathcal{D}_{\text{SFT}}^{-}=\{(x^i,y^i)\}_{i=1}^{N}$ and $M=N$, any global solution \(\theta^*\) of \(\texttt{BMO}\) is also a global solution of \(\texttt{BDS}\) in \eqref{eq:BDR} paired with some \(\omega^*\), and vice versa.
\end{theorem}

The proof is deferred in Appendix \ref{sec:proof_thm1}. Theorem \ref{thm:eqv} shows that \texttt{BMO} identifies the same validation-aligned solutions as \texttt{BDS}. 

\vspace{-0.2cm}
\subsection{Implicit response weights and online algorithm design}
\vspace{-0.2cm}

The key practical advantage of \texttt{BMO} over offline \texttt{BDS} is that it yields response-level weights in closed form, without solving any inner problem.

\begin{lemma}[Implicit response weight given by $\texttt{BMO}$] \label{lemma:implicit_BMO}
Under Assumption \ref{ass:seperable}, the implicit weight of $g$-th response for $i$-th question assigned by $\texttt{BMO}$ is
\begin{align}\label{BMO-weight}
\lambda_{i,g}=\frac{\exp(-\mathcal{L}_{\text{SFT}}(\theta;x^i,y^{i,g}_{\text{old}}))}{\sum_{g^\prime=1}^G\exp(-\mathcal{L}_{\text{SFT}}(\theta;x^i,y^{i,g^\prime}_{\text{old}}))}.
\end{align}
\end{lemma} 

The proof of Lemma \ref{lemma:implicit_BMO} can be found in Appendix \ref{sec:BMO_implicit_weight}. Since $\mathcal{L}_{\text{SFT}}(\theta;x,y)=-\log\pi_\theta(x,y)$, the implicit response weight is simply the softmax of the current-policy log-likelihood over the candidate responses of the same question.

\noindent\textbf{Interpretation. } As $\theta$ is updated using the reweighted validation loss together with the selected SFT loss, the online response weights in \eqref{BMO-weight} rank responses with lower SFT loss higher, because their individual SFT losses tend to descend jointly with the validation loss, suggesting stronger alignment. 

The following corollary shows that the online response weight can be implemented by the importance ratio computed on older samples. 
\begin{corollary}[Response weight via importance ratio]\label{cor:importance_ratio}
Under Assumption \ref{ass:seperable}, for $g$-th generated response to $i$-th question, we have the important ratio $r^g=\frac{\pi_\theta(x^i,y^{i,g}_{\text{old}})}{\pi_{\text{old}}(x^i,y^{i,g}_{\text{old}})}\propto \lambda_{i,g}$. 
\end{corollary}

\begin{algorithm}[t]
\caption{Online self-refining generation with question-level \texttt{BDS} and response-level \texttt{BMO}}
\begin{algorithmic}[1]
\STATE \textbf{Input:} Validation dataset $\mathcal{D}$ and low-quality SFT dataset $\mathcal{D}_{\text{SFT}}^-$. Initial LLM parameter $\theta_0$ and data selector parameter $\omega$. Step sizes $\alpha_k,\beta_k$, penalty strength $\gamma_k$, and generation frequency $K_\text{gen}$. 

\FOR{$k=1$ {\bfseries to} $K$}
\STATE Generate masked question index set $\mathcal{I}_M\subset [N]$
\STATE Sample question-answer data pair $(\tilde x^{j_k},\tilde y^{j_k})\sim\mathcal{D}$ and $(x^{i_k},y^{i_k})\sim\mathcal{D}_{\text{SFT}}^-$. 
\IF{$k \text{ mod } K_\text{gen}=0$}
\STATE Generate $G$ responses $y_s^{i,g}$ from the current $\pi_{\theta^k}$ for each masked question $i\in \mathcal{I}_M$
\STATE Update $\pi_{\text{old}}\leftarrow\pi_{\theta^k}$ and $y_{\text{old}}^{i,g}\leftarrow y_s^{i,g}$
\ENDIF 
\IF{$i_k \in \mathcal{I}_M$}
\STATE Average $\nabla_\theta^k=\frac{1}{G}\sum_{g=1}^G r^g \nabla_\theta\mathcal{L}_{\text{SFT}} (\theta;x^{i_k},y_{\text{old}}^{i_k,g})$  using importance ratio $r^g=\frac{\pi_\theta(x^i,y^{i,g}_{\text{old}})}{\pi_{\text{old}}(x^i,y^{i,g}_{\text{old}})}$
\ELSE
\STATE Use offline gradient $\nabla_\theta^k=\nabla\mathcal{L}_{\text{SFT}} (\theta^k;x^{i_k},y^{i_k})$
\ENDIF 
\STATE Update $\theta^{k+1}$ via \eqref{PBGD-1} and update $\omega^{k+1}$ via \eqref{PBGD-2}
\ENDFOR
\end{algorithmic}
\label{alg:alg1}
\end{algorithm}

\vspace{-0.3cm}
With question-level weights $\sigma_i(\omega)$ from \texttt{BDS} and response-level weights $r^g$ from Corollary \ref{cor:importance_ratio}, we can design the online self-refining algorithm based on the following online penalty reformulation 
\vspace{-0.5cm}

{\small\begin{align}\label{eq:penalty_problem_BSG}
\min_{\omega\in\tilde{\mathbb{R}}^N, \theta}~ &\mathcal{L}_{\rm val} (\theta)+\frac{\gamma_k}{N-N_M}\sum_{i\not\in\mathcal{I}_\textit{M}} \sigma_i(\omega)\mathcal{L}_{\text{SFT}} (\theta;x^i,y^i)+\frac{\gamma_k}{GN_M}\sum_{i\in\mathcal{I}_\textit{M}}\sum_{g=1}^G\sigma_i(\omega)r^g\mathcal{L}_{\text{SFT}} (\theta;x^i,y_s^{i,g}),
\end{align}}
\vspace{-0.3cm}

where $\gamma_k$ is an enlarging penalty constant. 


\noindent\textbf{Online self-refining algorithm design. } To solve \eqref{eq:penalty_problem_BSG}, at each iteration $k$, we randomly sample $(\tilde x^{j_k},\tilde y^{j_k})$ and $(x^{i_k},y^{i_k})$ from validation dataset $\mathcal{D}$ and low-quality SFT dataset $\mathcal{D}_{\text{SFT}}^-$, respectively. If $i_k\in\mathcal{I}_M$, the lower-level gradient should be calculated using generated responses; otherwise we use the offline response. Let $\nabla_\theta^k$ be the gradient estimator of the low-quality SFT dataset:
\begin{align}
\nabla_\theta^k \;=\;
\begin{cases}
\nabla_\theta \mathcal{L}_{\mathrm{SFT}}(\theta^k; x^{i_k}, y^{i_k}), & \text{if } i_k \notin \mathcal{I}_M,\\
\frac{1}{G}\sum_{g=1}^G r^g \nabla_\theta\mathcal{L}_{\text{SFT}} (\theta;x^{i_k},y_{\text{old}}^{i_k,g}), & \text{otherwise}. 
\end{cases}
\end{align}
which is given by the response reweighted sample gradient if the responses are updating. 
We then reweight the lower-level gradient estimator $\nabla_\theta^k$ with the upper-level gradient estimator to update $\theta$ as
\begin{subequations}\label{PBGD_LLM}
\begin{align}
\theta^{k+1}=\theta^k-\beta_k\left(\nabla \mathcal{L}_{\rm val} (\theta^k;\tilde x^{j_k},\tilde y^{j_k})+\gamma_k\sigma_{i_k}(\omega)\nabla_\theta^k\right), \label{PBGD-1}
\end{align}
and update question-level selector $\omega$ via
\begin{align}\label{PBGD-2}
\omega^{k+1}=\omega^k-\alpha_k\gamma_k \nabla\sigma_{i_k}(\omega^k)C_\omega^k
\end{align}
with the coefficient either the SFT loss on the offline data or the generated response
\begin{align}
\!\!\!\!C_\omega^k=
\begin{cases}
\mathcal{L}_{\text{SFT}} (\theta^{k+1};x^{i_k},y^{i_k}), & \!\!\text{if } i_k \notin \mathcal{I}_M,\\
\frac{1}{G}\sum_{g=1}^G \mathcal{L}_{\text{SFT}} (\theta^{k+1};x^{i_k},y^{i,g}_{\text{old}}), & \!\!\text{otherwise}.
\end{cases}\!\!
\end{align}
\end{subequations}

The full algorithm is summarized in Algorithm \ref{alg:alg1}.

\begin{table}[t]
\footnotesize 
\centering
\begin{tabular}{lcccc}
\toprule
\multirow{2}{*}{\textbf{Method}}
& \multicolumn{2}{c}{\textbf{\textsc{OpenOrca}} (upper-level)$\downarrow$}
& \multicolumn{2}{c}{\textbf{\textsc{Alpaca-cleaned}} (lower-level)$\downarrow$} \\
\cmidrule(lr){2-3}\cmidrule(lr){4-5}
& \textsc{Pythia-1b} & \textsc{Llama-8b}
& \textsc{Pythia-1b} & \textsc{Llama-8b} \\
\midrule
Direct mixing ($\rho=1$)        & $1.56 {\scriptstyle \pm 0.008}$   & $0.92 {\scriptstyle \pm 0.012}$   & $1.62{\scriptstyle \pm 0.013}$ & $0.908 {\scriptstyle \pm 0.012}$ \\
Direct mixing ($\rho=0.5$)      & $1.41{\scriptstyle \pm 0.011}$    & $0.84{\scriptstyle \pm 0.008}$    & $1.58{\scriptstyle \pm 0.006}$ & $0.908{\scriptstyle \pm 0.008}$ \\
Random selection               &$1.55{\scriptstyle \pm 0.004}$                                & $0.94{\scriptstyle \pm 0.006}$                                & $1.61{\scriptstyle \pm 0.007}$                             & $0.901{\scriptstyle \pm 0.010}$ \\
LESS                           & $1.54{\scriptstyle \pm 0.006}$                                & $0.88{\scriptstyle \pm 0.010}$                                & $1.49{\scriptstyle \pm 0.006}$                             & $0.892{\scriptstyle \pm 0.015}$ \\
DSIR                           & $1.52{\scriptstyle \pm 0.007}$                                & $0.92{\scriptstyle \pm 0.007}$                                & $1.57{\scriptstyle \pm 0.010}$                             & $0.896{\scriptstyle \pm 0.007}$ \\
Offline \texttt{BDS}           & $1.38{\scriptstyle \pm 0.005}$    & $0.80 {\scriptstyle \pm 0.007}$   & $1.54{\scriptstyle \pm 0.008}$ & $0.893 {\scriptstyle \pm 0.005}$ \\
\midrule
Online ($R=5\%,\, G=1$)         & $1.37{\scriptstyle \pm 0.005}$    & $0.80{\scriptstyle \pm 0.004}$    & $1.46{\scriptstyle \pm 0.006}$ & $0.887{\scriptstyle \pm 0.008}$ \\
Online ($R=10\%,\, G=1$)        & $1.34{\scriptstyle \pm 0.003}$    & $0.78{\scriptstyle \pm 0.006}$    & $1.41{\scriptstyle \pm 0.009}$ & $0.857{\scriptstyle \pm 0.010}$ \\
\textbf{Online ($R=10\%,\, G=5$)} & $\textbf{1.32}{\scriptstyle \pm 0.004}$ & $\textbf{0.76}{\scriptstyle \pm 0.003}$ & $\textbf{1.37}{\scriptstyle \pm 0.007}$ & $\textbf{0.835}{\scriptstyle \pm 0.007}$ \\
\bottomrule
\end{tabular}
\vspace{0.2cm}
\caption{Upper-level evaluation loss on \textsc{OpenOrca} and lower-level evaluation loss on \textsc{Alpaca-cleaned} for fine-tuned \textsc{Pythia-1b} and \textsc{Llama-8b}. \textbf{Bold} indicates the best result (lower is better). Online sample ratio is defined as $R=\frac{N_M}{N}$. Online selection methods improve both upper-level performance and the quality of the selected lower-level dataset. }
\label{tab:pref_tuning}
\vspace{-0.4cm}
\end{table}

\section{Experiments} 
\label{sec:experiments}

We test Algorithm \ref{alg:alg1} on two LLM post-training tasks: quality enhancement and safety-aware fine-tuning. For each task, the lower-level raw dataset is selected based on the high-quality upper-level validation data.

\noindent\textbf{Models and baselines. } We fine-tune on two base LLM: \textsc{Pythia-1b} \citep{biderman2023pythia} and \textsc{Llama-3-8b-Instruct} \citep{dubey2024llama}, and compare Algorithm \ref{alg:alg1} with the following baselines: \textsf{(B1)} direct mixing approach \citep{bianchisafety} in \eqref{eq:direct_mixing};  \textsf{(B2)} offline \texttt{BDS} via stochastic penalty based gradient approach \citep{shen2024seal}; \textsf{(B3)} random selection \citep{xiarethinking}, which randomly selects the fine-tuning data; \textsf{(B4)} low-rank gradient similarity search (LESS) \citep{xialess}, which approximates each sample gradient in a low-dimensional subspace and selects the data based on gradient similarity; \textsf{(B5)} data selection via importance resampling (DSIR) \citep{xie2023data}, which estimates bag-of-$n$-gram probability models for both the validation and raw datasets, and selects the data based on the importance ratio. 

\begin{wrapfigure}{r}{0.55\linewidth}
\vspace{-0.5cm}
\centering
\includegraphics[width=\linewidth]{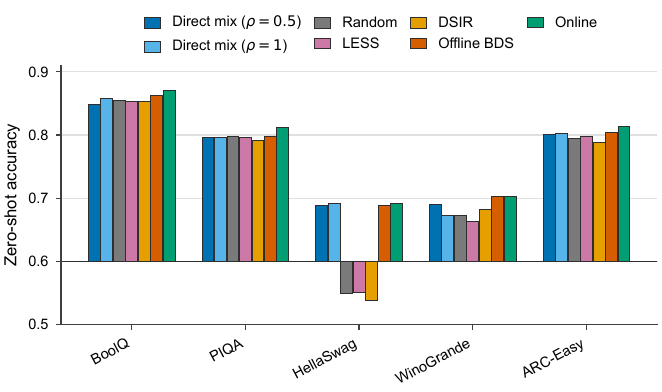}
\vspace{-0.5cm}
\caption{Zero-shot accuracy on five QA benchmarks for the \textbf{quality enhancement} task (\textsc{Llama-3-8b-Instruct}). Higher is better ($\uparrow$). }
\label{tab:acc_table}
\end{wrapfigure}

\noindent\textbf{Evaluation process. }
For each method, we report losses on selected lower-level SFT dataset and the held-out validation dataset. For external evaluation, we test the model performance given by different methods using EleutherAI’s LM Evaluation Harness \citep{eval-harness} and report the zero-shot accuracy on $5$ representative QA tasks. Additionally, we use \textsc{AlpacaEval} \citep{alpaca_eval} to evaluate the quality of the responses generated by the models trained with different methods. Given a prompt dataset and a reference model (\textsc{Llama-3-8b}), we collect a pair of reference responses and the model's responses for each prompt. These responses are then judged by an LLM evaluator (GPT-4 Turbo) to compute the win rate of how often the model's responses outperform the reference responses. 

All experiments are conducted on a server equipped with an NVIDIA H100 GPU with $96$ GB memory, and results are averaged over $4$ runs. Our code is adapted from the bilevel LLM post-training library 
\url{https://github.com/Post-LLM/BIPOST}
and detailed settings are deferred to Appendix \ref{sec:additional_experimernts}.


\vspace{-0.3cm}
\subsection{Quality enhancement}
\vspace{-0.2cm}

In this task, we use \textsc{OpenOrca} dataset in the upper-level, which has higher per-sample quality for complex, chain-of-thought style instructions \citep{longpre2023flan,mukherjee2023orca}, and utilize \textsc{Alpaca-cleaned} dataset in the lower-level, which is a small and tidy instruction-following dataset \citep{wang2022self}.

\begin{wrapfigure}{r}{0.5\linewidth}
\centering
\resizebox{\linewidth}{!}{\begin{tikzpicture}[x=33cm,y=0.50cm]
\definecolor{mixgray}{RGB}{125,125,125}
\definecolor{bdsblue}{RGB}{54,103,167}
\definecolor{bmogreen}{RGB}{43,136,90}

\def\xmin{0.70}
\def\xmax{0.84}

\node[font=\scriptsize\bfseries, anchor=west] at (\xmin,3.65) {\textsc{AlpacaEval} win-rate};
\node[font=\scriptsize, anchor=east, text=black!65] at (\xmax,3.65) {higher is better};

\foreach \x/\label in {0.70/0.70,0.75/0.75,0.80/0.80,0.84/0.84} {
  \draw[black!18, line width=0.25pt] (\x,0.45) -- (\x,3.15);
  \node[font=\scriptsize, anchor=north, text=black!70] at (\x,0.33) {\label};
}
\draw[->, black!65, line width=0.45pt] (\xmin,0.45) -- (\xmax+0.004,0.45);

\node[font=\scriptsize, anchor=east, align=right] at (\xmin-0.002,2.7) {Direct mixing\\$(\rho=0.5)$};
\draw[mixgray, line width=1pt] (\xmin,2.7) -- (0.724,2.7);
\fill[mixgray] (0.724,2.7) circle (1.8pt);
\node[font=\scriptsize, anchor=west, text=mixgray] at (0.727,2.7) {0.724};

\node[font=\scriptsize, anchor=east] at (\xmin-0.002,1.8) {Offline \texttt{BDS}};
\draw[bdsblue, line width=1pt] (\xmin,1.8) -- (0.787,1.8);
\fill[bdsblue] (0.787,1.8) circle (1.8pt);
\node[font=\scriptsize, anchor=west, text=bdsblue] at (0.790,1.8) {0.787};

\node[font=\scriptsize, anchor=east, align=right] at (\xmin-0.002,0.9) {Online selection\\$(R=10\%,\,G=1)$};
\draw[bmogreen, line width=1.15pt] (\xmin,0.9) -- (0.823,0.9);
\fill[bmogreen] (0.823,0.9) circle (2pt);
\node[font=\scriptsize\bfseries, anchor=west, text=bmogreen] at (0.826,0.9) {0.823};
\end{tikzpicture}}
\vspace{-0.45cm}
\caption{\textsc{AlpacaEval} win rate against \textsc{Llama-3-8b}  judged by GPT-4 Turbo for representative methods. Online selection outperforms the offline selection for response quality.}
\label{fig:winrate} 
\vspace{0.1cm}
\end{wrapfigure}

\noindent\textbf{Online selection outperforms offline \texttt{BDS}.} Table \ref{tab:pref_tuning} shows that offline \texttt{BDS} consistently outperforms other baselines on both the upper-level \textsc{OpenOrca} and lower-level \textsc{Alpaca-cleaned} evaluations. Moreover, the online method improves the performance further, most noticeably on the lower-level evaluation because online selection refines the responses to improve the data quality. We find that adding a moderate fraction of online samples $(5\%, 10\%)$ to the offline dataset boosts lower-level performance, but a larger share $(20\%)$ is unstable, likely because it adds too much noisy signals. Moreover, generating multiple responses per question improves response quality on the lower-level dataset. For most of the hyperparameters, online self-tuning consistently outperforms the baselines.



\noindent\textbf{External evaluation. } In Figure \ref{tab:acc_table}, the online method is the best compared with other baselines. Additionally, Figure \ref{fig:winrate} shows that online selection achieves the highest win rate compared with representative baselines, which suggests that it improves response quality. Moreover, qualitative examples in Appendix \ref{sec:generated_response} corroborate that online self-refining tends to produce more accurate and concise responses than offline \texttt{BDS}.

\begin{wrapfigure}{r}{0.4\linewidth}
\vspace{-0.2cm}
\centering
\includegraphics[width=0.95\linewidth]{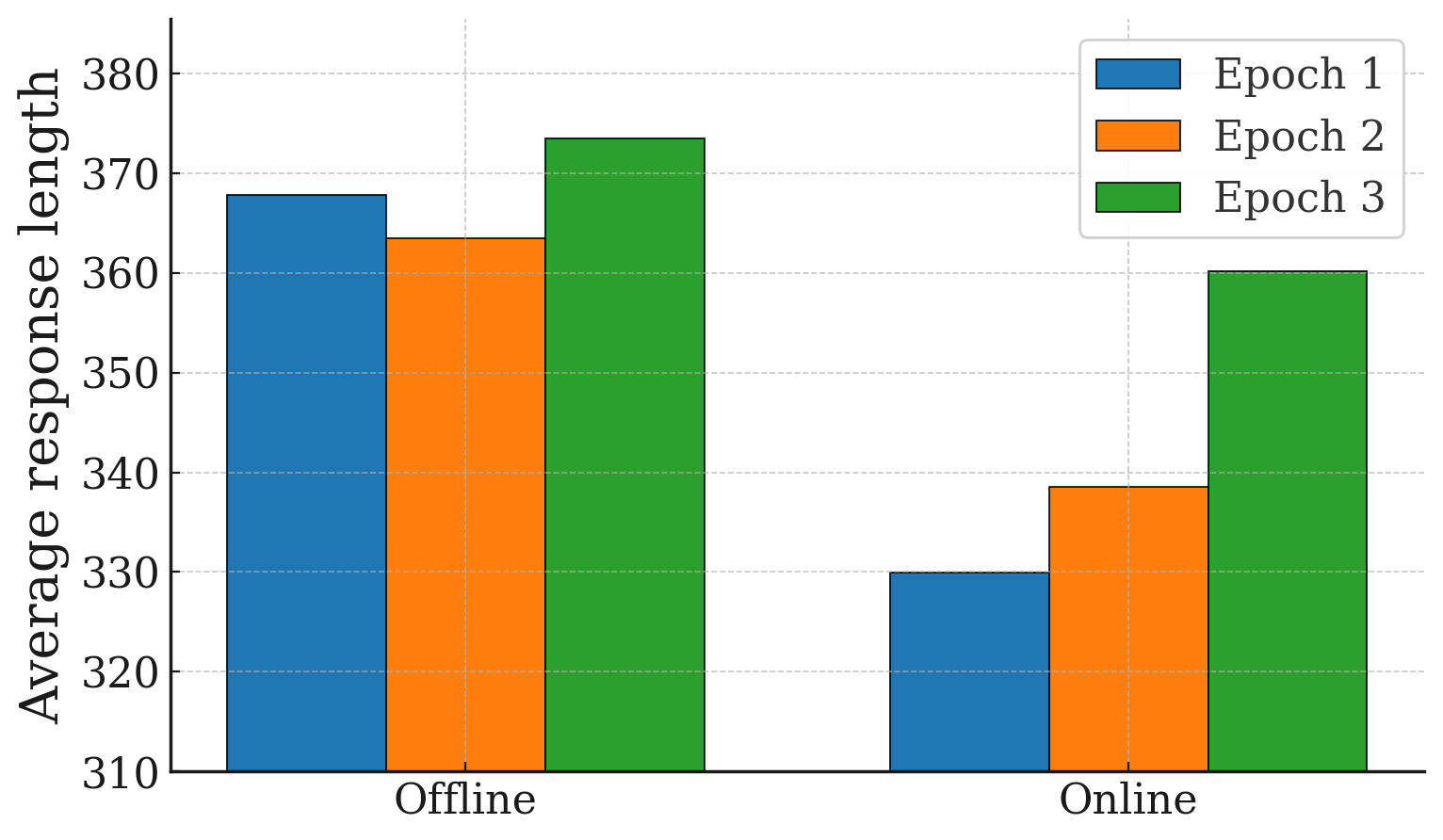}
\vspace{-0.3cm}
\caption{Average response length of the top-$10\%$ questions by learned weight from offline \texttt{BDS} and online selection. Online selection shifts from shorter- to longer-response questions over training.}
\label{fig:top-10}
\end{wrapfigure}

\noindent\textbf{Online selection learns from simple to hard questions. } To further compare the learning behavior of offline strategy and online self-refining, we analyze the top $10\%$ and bottom $10\%$ questions at each epoch given by them. According to Table \ref{tab:ranking_question} in Appendix \ref{sec:generated_response}, the online self-refining tends to focus on easy questions at the first epoch and gradually progress from simple to hard questions. The knowledge gained from simpler questions builds a solid backbone model for subsequent adaptation to the hard questions \citep{ma2025learning,zhang2025policy,lin2025goedel}. We also report response length as a partial measure for question difficulty in Figure \ref{fig:top-10}, where longer responses indicate longer reasoning and thus, harder questions; see evidence in Table \ref{tab:ranking_question_response}. Notably, the online self-refining algorithm learns from shorter-response questions first, while the offline \texttt{BDS} continues to tackle longer-response questions throughout training. 

\begin{wraptable}{r}{0.60\linewidth}
\vspace{-0.1cm}
\scriptsize
\centering
\begin{tabular}{@{}>{\raggedright\arraybackslash}p{0.42\linewidth}cc@{}}
\toprule
\textbf{Method} & \textsc{Pythia-1b} & \textsc{Llama-8b} \\
\midrule
Direct mixing ($\rho=0.5$)     & 0.24 & 10.53 \\
DSIR           & 0.27 & 11.23 \\
Offline \texttt{BDS}           & 0.31 & 11.75 \\
LESS           & 0.84 & 28.75 \\
\midrule
Online ($R=5\%,\, G=1$)         & 0.43 & 13.28 \\
Online ($R=10\%,\, G=1$)        & 0.47 & 13.78 \\
Online ($R=10\%,\, G=5$) & 0.86 & 30.46 \\
\bottomrule
\end{tabular}
\vspace{-0.1cm}
\caption{Average runtime (hours) on the \textsc{OpenOrca} dataset for \textsc{Pythia-1b} and \textsc{Llama-8b}. }
\label{tab:time_complexity}
\end{wraptable}

\noindent\textbf{Comparable runtime. } We report the runtime of different algorithms in Table \ref{tab:time_complexity}. With the number of online samples per question $G=1$, the computational overhead of online sampling is not significant, especially for the larger model. However, generating $G=5$ samples per question introduces additional $1\times$ computational overhead while only leading to slight improvement on the model performance (see Table \ref{tab:pref_tuning}). Therefore, $G=1$ is an ideal choice.

\vspace{-0.2cm}
\subsection{Safety-aware fine-tuning}
\vspace{-0.2cm}

For safety-aware LLM fine-tuning task, we follow the setup from \citep{shen2024seal}. We use \textsc{\textsc{RedOrca}} datasets as the lower-level raw dataset, which mixes \textsc{SlimOrca} dataset with 22k potentially unsafe instructions and responses from the \textsc{Anthropic red-teaming} dataset \citep{ganguli2022red}. For the upper-level safety guidance, we use 1) the \textsc{BlueOrca} dataset \citep{shen2024seal} which contains safe data distilled from \textsc{SlimOrca} and 2)  \texttt{Dahoas/rm-hh-rlhf} dataset, a offline RL reward-modeling version of Anthropic’s Helpful and Harmless RLHF preference dataset \citep{bai2022training}. 

\begin{wrapfigure}{r}{0.55\linewidth}
\vspace{-0.4cm}
\centering
\includegraphics[width=\linewidth]{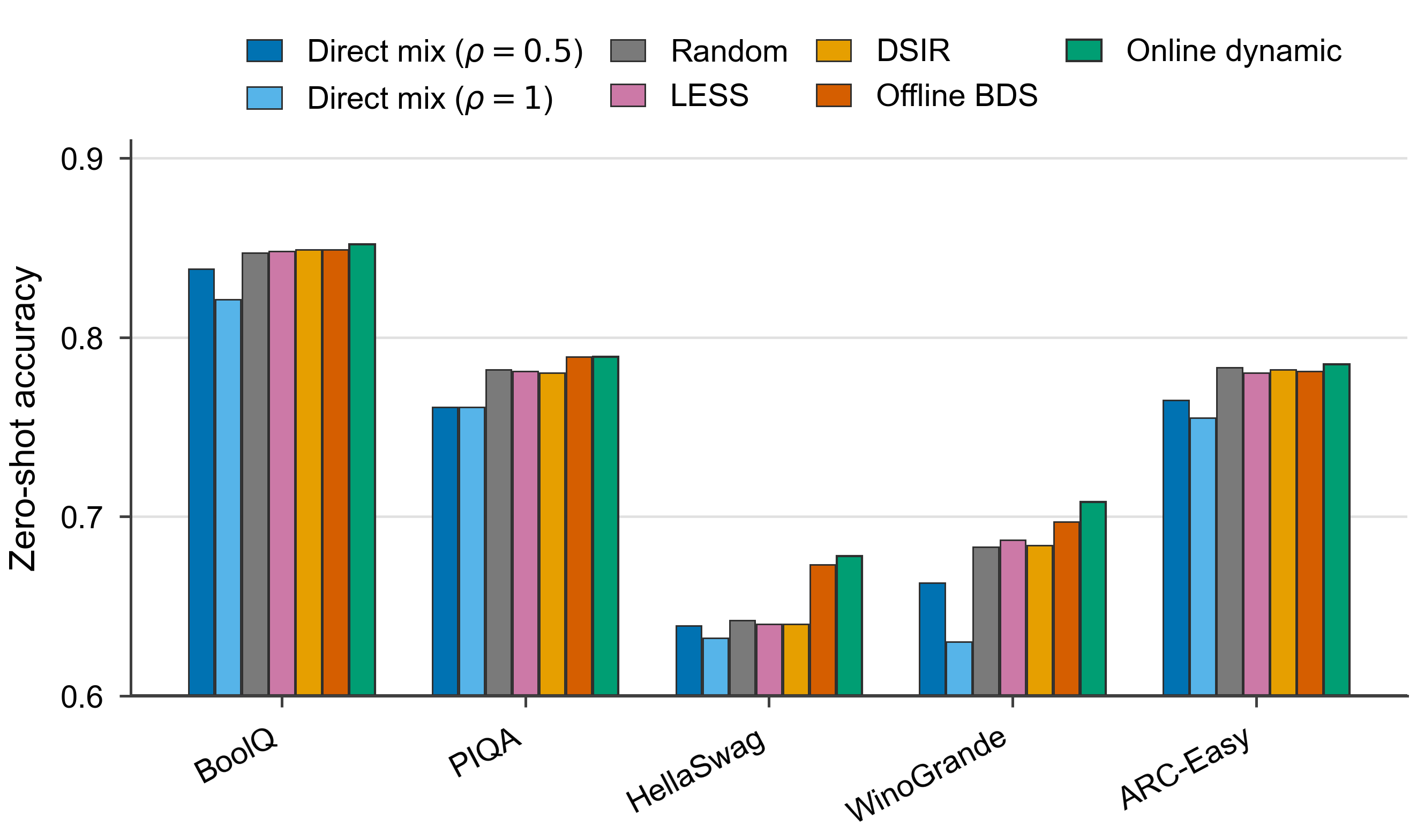}
\vspace{-0.35cm}
\caption{Zero-shot accuracy on five QA benchmarks for the \textbf{safety-aware fine-tuning} task using \textsc{BlueOrca} dataset as validation signal (\textsc{Llama-3-8b-Instruct}). Higher is better ($\uparrow$). }
\label{tab:acc_table_safe}
\vspace{-0.25cm}
\end{wrapfigure}

\noindent\textbf{Online refinement enhances the data quality.} With online self-refining generation, the model can guide the generated response by the upper-level safe data, even if the initial responses are harmful. The evaluation losses on the safe and unsafe data for different methods are shown in Table \ref{tab:safe_tuning} and the general QA benchmark evaluation is shown in Figure \ref{tab:acc_table_safe}. While both offline selection and online self-refining improve the validation performance and the general QA performance, online self-refining generation further enhances the fine-tuning performance on the lower-level unsafe dataset, which suggests that we also improve the data quality for the lower-level.

\noindent\textbf{Dynamic masking further improves the safety.} So far, the default masked question index set $\mathcal{I}_M$ is sampled once and held fixed in Algorithm \ref{alg:alg1}. In safety aware fine-tuning, performance improves further by generating online responses only for questions currently flagged as unsafe, i.e., by dynamically constructing the following $\mathcal{I}_M$ for the bottom ranked questions  
\begin{align*}
\textbf{Online dynamic: }\qquad \mathcal{I}_M^k=\{i~|~ \sigma_i(\omega_t) \text{ is ranked bottom } R \text{ among all data}\},
\end{align*}
where $R=\frac{N_M}{N}$ is the online sample ratio. We call this the online dynamic strategy. Table \ref{tab:safe_tuning} shows that this dynamic strategy further improves performance on the unsafe dataset \textsc{RedOrca}, suggesting that question-level \texttt{BDS} scores and response-level \texttt{BMO} weights complement each other in practice.


\begin{table}[t]
\vspace{0.3cm}
\footnotesize
\centering
\begin{tabular}{lcccc}
\toprule
\textbf{Method} & \textsc{BlueOrca}$\downarrow$ & \textsc{RedOrca}$\downarrow$& \textsc{rm-hh-rlhf}$\downarrow$ & \textsc{RedOrca}$\downarrow$ \\
\midrule
Direct mixing ($\rho=0.5$)      & $0.88{\scriptstyle \pm 0.010}$ & $1.31{\scriptstyle \pm 0.012}$ & $0.58 \scriptstyle\pm 0.021$ & $1.77\scriptstyle\pm 0.011$ \\
Direct mixing ($\rho=1$)        & $0.94{\scriptstyle \pm 0.007}$  & $1.27{\scriptstyle \pm 0.008}$ & $0.71\scriptstyle\pm 0.036$ & $2.07\scriptstyle\pm 0.042$ \\
Random selection               & $0.92{\scriptstyle \pm 0.013}$                              & $1.25{\scriptstyle \pm 0.011}$                             & $0.74{\scriptstyle \pm 0.027}$                          & $1.63{\scriptstyle \pm 0.024}$  \\
LESS                           & $0.92{\scriptstyle \pm 0.008}$                              & $1.28{\scriptstyle \pm 0.014}$                             & $0.64{\scriptstyle \pm 0.018}$                          & $1.75{\scriptstyle \pm 0.013}$ \\
DSIR                           & $0.91{\scriptstyle \pm 0.012}$                              & $1.24{\scriptstyle \pm 0.012}$                             & $0.82{\scriptstyle \pm 0.016}$                          & $1.67{\scriptstyle \pm 0.015}$ \\
Offline \texttt{BDS}           & $0.85{\scriptstyle \pm 0.006}$ & $1.25{\scriptstyle \pm 0.005}$& $0.53 \scriptstyle\pm 0.013$ & $1.72\scriptstyle\pm 0.007$\\
\midrule
Online ($R=5\%,\, G=1$)         & $0.84{\scriptstyle \pm 0.004}$ & $1.22{\scriptstyle \pm 0.003}$ & $0.49\scriptstyle\pm 0.018$ & $1.58\scriptstyle\pm 0.011$\\
Online ($R=10\%,\, G=1$)        & $0.83{\scriptstyle \pm 0.005}$ & $1.20{\scriptstyle \pm 0.005}$ & $0.47 \scriptstyle\pm 0.015$ & $1.45\scriptstyle\pm 0.005$\\
\midrule
\textbf{Online dynamic} ($R=10\%,\, G=1$)         & \textbf{0.82} ${\scriptstyle \pm 0.004}$& \textbf{1.02} ${\scriptstyle \pm 0.006}$& $\textbf{0.43} \scriptstyle\pm 0.017$ & $\textbf{1.37}\scriptstyle\pm 0.012$\\
\bottomrule
\end{tabular}
\vspace{0.1cm}
\caption{Evaluation loss on validation \textsc{BlueOrca} and \textsc{Dahoas/rm-hh-rlhf} dataset and selected \textsc{RedOrca} dataset fine-tuned with \textsc{Llama-8b} model. \textbf{Bold} indicates the best result. Online sample ratio is defined as $R=\frac{N_M}{N}$. For online dynamic strategy, we choose $R=10\% $ and $ G=1$. } 
\label{tab:safe_tuning}
\vspace{-0.7cm}
\end{table}


\section{Conclusions and limitations}
\label{sec:conclusion}

We study the bilevel offline data selection  \texttt{BDS} and the online self-refining generation problem. We prove that the optimal solution given by offline \texttt{BDS} outperforms the direct mixing baseline if the mixing ratio is not tuned appropriately. Moreover, we demonstrate that gradient-based approach can converge to the optimal solution of \texttt{BDS} with one-layer Transformer architecture. Besides, we show that  \texttt{BDS} and  \texttt{BMO} are equivalent optimization problem that assign a validation weight to the low-quality SFT dataset either explicitly or implicitly, while \texttt{BMO} admits the closed-form response weight when adapting to the online self-refining generation setting. Experiments on LLM fine-tuning validate the effectiveness of the proposed online selection algorithm.  

\noindent\textbf{Limitations and borader impacts.}
Our theoretical guarantees rely on the separable data assumption (Assumption~\ref{ass:seperable}), which requires the overparameterized SFT model to interpolate the training data. In practice, without this assumption, one may need to run additional gradient descent on the lower-level objective to estimate its minimal function value and subtract it from the penalty objective \citep{shen2023penalty}. Our method may benefit the society by improving the quality and safety of SFT data used for LLM. 

\section*{Acknowledgement}
The work was supported by the National Science Foundation Projects
2401297, 2532349 and 2532653, and by the Cisco Research Award.


\bibliography{bilevel,optimization,LLM,multi-objective}

@string{NIPS= "Proc. Advances in Neural Information Processing Systems"}

@string{ICML= "Proc. International Conference on Machine Learning"}

@string{ICLR= "Proc. International Conference on Learning Representations"}

@string{AAAI= "Proc. Association for the Advancement of Artificial Intelligence"}

@string{NIPS2020_loc= "virtual"}

@string{NIPS2021_loc= "virtual"}

@string{NIPS2023_loc= "New Orleans, LA"}

@string{NIPS2024_loc= "Vancouver, BC, Canada"}

@string{NIPS2025_loc= "San Diego, CA"}

@string{ICML2016_loc= "New York City, NY"}

@string{ICML2017_loc= "Sydney, Australia"}

@string{ICML2018_loc= "Stockholm, Sweden"}

@string{ICML2020_loc= "virtual"}

@string{ICML2021_loc= "virtual"}

@string{ICML2022_loc= "Baltimore, MD"}

@string{ICML2023_loc= "Honolulu, HI"}

@string{ICML2024_loc= "Vienna, Austria"}

@string{ICML2025_loc= "Vancouver, BC, Canada"}

@string{ICLR2022_loc= "virtual"}

@string{ICLR2023_loc= "Kigali, Rwanda"}

@string{ICLR2024_loc= "Vienna, Austria"}

@string{ICLR2025_loc= "Singapore, Singapore"}

@string{AAAI2022_loc = "virtual"}

@inproceedings{shen2024seal,
  title={Seal: Safety-enhanced aligned llm fine-tuning via bilevel data selection},
  author={Shen, Han and Chen, Pin-Yu and Das, Payel and Chen, Tianyi},
  booktitle=ICLR,
  year={2025}
}

@inproceedings{qi2024fine,
  title={Fine-tuning Aligned Language Models Compromises Safety, Even When Users Do Not Intend To!},
  author={Qi, Xiangyu and Zeng, Yi and Xie, Tinghao and Chen, Pin-Yu and Jia, Ruoxi and Mittal, Prateek and Henderson, Peter},
  booktitle=ICLR,
  year={2024}
}

@inproceedings{wang2024interpretable,
  title={Interpretable Preferences via Multi-Objective Reward Modeling and Mixture-of-Experts},
  author={Wang, Haoxiang and Xiong, Wei and Xie, Tengyang and Zhao, Han and Zhang, Tong},
  booktitle={Findings of the Association for Computational Linguistics: EMNLP 2024},
  pages={10582--10592},
  year={2024}
}

@inproceedings{li2024one,
  title={One-layer transformer provably learns one-nearest neighbor in context},
  author={Li, Zihao and Cao, Yuan and Gao, Cheng and He, Yihan and Liu, Han and Klusowski, Jason and Fan, Jianqing and Wang, Mengdi},
  booktitle=NIPS,
  year={2024},
  address=NIPS2024_loc
}

@article{bai2022training,
  title={Training a helpful and harmless assistant with reinforcement learning from human feedback},
  author={Bai, Yuntao and Jones, Andy and Ndousse, Kamal and Askell, Amanda and Chen, Anna and DasSarma, Nova and Drain, Dawn and Fort, Stanislav and Ganguli, Deep and Henighan, Tom and others},
  journal={arXiv preprint arXiv:2204.05862},
  year={2022}
}

@article{dubey2024llama,
  title={The llama 3 herd of models},
  author={Grattafiori, Aaron and Dubey, Abhimanyu and Jauhri, Abhinav and Pandey, Abhinav and Kadian, Abhishek and Al-Dahle, Ahmad and Letman, Aiesha and Mathur, Akhil and Schelten, Alan and Vaughan, Alex and others},
  journal={arXiv preprint arXiv:2407.21783},
  year={2024}
}

@inproceedings{weifinetuned,
  title={Finetuned Language Models are Zero-Shot Learners},
  author={Wei, Jason and Bosma, Maarten and Zhao, Vincent and Guu, Kelvin and Yu, Adams Wei and Lester, Brian and Du, Nan and Dai, Andrew M and Le, Quoc V},
  booktitle=ICLR,
  year=2022
}

@inproceedings{vaswani2017attention,
  title={Attention is all you need},
  author={Vaswani, Ashish and Shazeer, Noam and Parmar, Niki and Uszkoreit, Jakob and Jones, Llion and Gomez, Aidan N and Kaiser, {\L}ukasz and Polosukhin, Illia},
  booktitle=NIPS,
  year={2017}
}

@inproceedings{ren2024learning,
  title={Learning dynamics of {LLM} {Finetuning}},
  author={Ren, Yi and Sutherland, Danica J},
  booktitle=ICLR,
  year={2025},
  address=ICLR2025_loc
}

@inproceedings{choi2024safety,
  title={Safety-Aware Fine-Tuning of Large Language Models},
  author={Choi, Hyeong Kyu and Du, Xuefeng and Li, Yixuan},
  booktitle={Neurips Safe Generative AI Workshop},
  year=2024
}

@inproceedings{zhong2024panacea,
  title={Panacea: Pareto alignment via preference adaptation for llms},
  author={Zhong, Yifan and Ma, Chengdong and Zhang, Xiaoyuan and Yang, Ziran and Chen, Haojun and Zhang, Qingfu and Qi, Siyuan and Yang, Yaodong},
  booktitle=NIPS,
  year={2024},
  address=NIPS2024_loc
}

@inproceedings{ji2020directional,
  title={Directional convergence and alignment in deep learning},
  author={Ji, Ziwei and Telgarsky, Matus},
  booktitle=NIPS,
  year={2020},
  address=NIPS2020_loc
}

@article{tarzanagh2023transformers,
  title={Transformers as support vector machines},
  author={Tarzanagh, Davoud Ataee and Li, Yingcong and Thrampoulidis, Christos and Oymak, Samet},
  journal={arXiv preprint arXiv:2308.16898},
  year={2023}
}

@article{saglam2025large,
  title={Large Language Models Encode Semantics in Low-Dimensional Linear Subspaces},
  author={Saglam, Baturay and Kassianik, Paul and Nelson, Blaine and Weerawardhena, Sajana and Singer, Yaron and Karbasi, Amin},
  journal={arXiv preprint arXiv:2507.09709},
  year={2025}
}

@inproceedings{wang2022self,
  title={Self-instruct: Aligning language models with self-generated instructions},
  author={Wang, Yizhong and Kordi, Yeganeh and Mishra, Swaroop and Liu, Alisa and Smith, Noah A and Khashabi, Daniel and Hajishirzi, Hannaneh},
  booktitle={Proc. of the 61st annual meeting of the association for computational linguistics (volume 1: long papers)},
  pages={13484--13508},
  year={2023}
}

@inproceedings{xie2020self,
  title={Self-training with noisy student improves imagenet classification},
  author={Xie, Qizhe and Luong, Minh-Thang and Hovy, Eduard and Le, Quoc V},
  booktitle={Proceedings of the IEEE/CVF conference on computer vision and pattern recognition},
  pages={10687--10698},
  year={2020}
}

@article{gulcehre2023reinforced,
  title={Reinforced self-training (rest) for language modeling},
  author={Gulcehre, Caglar and Paine, Tom Le and Srinivasan, Srivatsan and Konyushkova, Ksenia and Weerts, Lotte and Sharma, Abhishek and Siddhant, Aditya and Ahern, Alex and Wang, Miaosen and Gu, Chenjie and others},
  journal={arXiv preprint arXiv:2308.08998},
  year={2023}
}

@inproceedings{song2024importance,
  title={The importance of online data: Understanding preference fine-tuning via coverage},
  author={Song, Yuda and Swamy, Gokul and Singh, Aarti and Bagnell, J and Sun, Wen},
  booktitle=NIPS,
  year={2024},
  address=NIPS2024_loc 
}

@article{ma2025learning,
  title={Learning What Reinforcement Learning Can't: Interleaved Online Fine-Tuning for Hardest Questions},
  author={Ma, Lu and Liang, Hao and Qiang, Meiyi and Tang, Lexiang and Ma, Xiaochen and Wong, Zhen Hao and Niu, Junbo and Shen, Chengyu and He, Runming and Cui, Bin and others},
  journal={arXiv preprint arXiv:2506.07527},
  year={2025}
}

@article{zhang2025policy,
  title={On-policy rl meets off-policy experts: Harmonizing supervised fine-tuning and reinforcement learning via dynamic weighting},
  author={Zhang, Wenhao and Xie, Yuexiang and Sun, Yuchang and Chen, Yanxi and Wang, Guoyin and Li, Yaliang and Ding, Bolin and Zhou, Jingren},
  journal={arXiv preprint arXiv:2508.11408},
  year={2025}
}

@article{lin2025goedel,
  title={Goedel-prover-v2: Scaling formal theorem proving with scaffolded data synthesis and self-correction},
  author={Lin, Yong and Tang, Shange and Lyu, Bohan and Yang, Ziran and Chung, Jui-Hui and Zhao, Haoyu and Jiang, Lai and Geng, Yihan and Ge, Jiawei and Sun, Jingruo and others},
  journal={arXiv preprint arXiv:2508.03613},
  year={2025}
}

@article{grattafiori2024llama,
  title={The llama 3 herd of models},
  author={Grattafiori, Aaron and Dubey, Abhimanyu and Jauhri, Abhinav and Pandey, Abhinav and Kadian, Abhishek and Al-Dahle, Ahmad and Letman, Aiesha and Mathur, Akhil and Schelten, Alan and Vaughan, Alex and others},
  journal={arXiv preprint arXiv:2407.21783},
  year={2024}
}

@inproceedings{biderman2023pythia,
  title={Pythia: A suite for analyzing large language models across training and scaling},
  author={Biderman, Stella and Schoelkopf, Hailey and Anthony, Quentin Gregory and Bradley, Herbie and O’Brien, Kyle and Hallahan, Eric and Khan, Mohammad Aflah and Purohit, Shivanshu and Prashanth, USVSN Sai and Raff, Edward and others},
  booktitle=ICML,
  year={2023}
}

@inproceedings{tajwar2024preference,
  title={Preference Fine-Tuning of LLMs Should Leverage Suboptimal, On-Policy Data},
  author={Tajwar, Fahim and Singh, Anikait and Sharma, Archit and Rafailov, Rafael and Schneider, Jeff and Xie, Tengyang and Ermon, Stefano and Finn, Chelsea and Kumar, Aviral},
  booktitle=ICML,
  year={2024},
  address=ICML2024_loc
}

@inproceedings{longpre2023flan,
  title={The flan collection: Designing data and methods for effective instruction tuning},
  author={Longpre, Shayne and Hou, Le and Vu, Tu and Webson, Albert and Chung, Hyung Won and Tay, Yi and Zhou, Denny and Le, Quoc V and Zoph, Barret and Wei, Jason and others},
  booktitle=ICML,
  year={2023},
  address=ICML2023_loc
}

@article{mukherjee2023orca,
  title={Orca: Progressive learning from complex explanation traces of gpt-4},
  author={Mukherjee, Subhabrata and Mitra, Arindam and Jawahar, Ganesh and Agarwal, Sahaj and Palangi, Hamid and Awadallah, Ahmed},
  journal={arXiv preprint arXiv:2306.02707},
  year={2023}
}

@article{ganguli2022red,
  title={Red teaming language models to reduce harms: Methods, scaling behaviors, and lessons learned},
  author={Ganguli, Deep and Lovitt, Liane and Kernion, Jackson and Askell, Amanda and Bai, Yuntao and Kadavath, Saurav and Mann, Ben and Perez, Ethan and Schiefer, Nicholas and Ndousse, Kamal and others},
  journal={arXiv preprint arXiv:2209.07858},
  year={2022}
}

@misc{taori2023stanford,
  title={Stanford alpaca: An instruction-following llama model},
  author={Taori, Rohan and Gulrajani, Ishaan and Zhang, Tianyi and Dubois, Yann and Li, Xuechen and Guestrin, Carlos and Liang, Percy and Hashimoto, Tatsunori B},
  year={2023},
  publisher={Stanford, CA, USA}
}

@inproceedings{zhou2023lima,
  title={Lima: Less is more for alignment},
  author={Zhou, Chunting and Liu, Pengfei and Xu, Puxin and Iyer, Srinivasan and Sun, Jiao and Mao, Yuning and Ma, Xuezhe and Efrat, Avia and Yu, Ping and Yu, Lili and others},
  booktitle=NIPS,
  year={2023},
  address=NIPS2023_loc
}

@inproceedings{lin2024data,
  title={Data-efficient Fine-tuning for LLM-based Recommendation},
  author={Lin, Xinyu and Wang, Wenjie and Li, Yongqi and Yang, Shuo and Feng, Fuli and Wei, Yinwei and Chua, Tat-Seng},
  booktitle={Proceedings of the 47th international ACM SIGIR conference on research and development in information retrieval},
  pages={365--374},
  year={2024}
}

@inproceedings{xie2023doremi,
  title={Doremi: Optimizing data mixtures speeds up language model pretraining},
  author={Xie, Sang Michael and Pham, Hieu and Dong, Xuanyi and Du, Nan and Liu, Hanxiao and Lu, Yifeng and Liang, Percy S and Le, Quoc V and Ma, Tengyu and Yu, Adams Wei},
  booktitle=NIPS,
  year={2023},
  address=NIPS2023_loc
}

@inproceedings{liuregmix,
  title={RegMix: Data Mixture as Regression for Language Model Pre-training},
  author={Liu, Qian and Zheng, Xiaosen and Muennighoff, Niklas and Zeng, Guangtao and Dou, Longxu and Pang, Tianyu and Jiang, Jing and Lin, Min},
  booktitle=ICLR,
  year={2025},
  address=ICLR2025_loc
}

@inproceedings{kangget,
  title={Get more for less: Principled Data Selection for Warming Up Fine-Tuning in LLMs},
  author={Kang, Feiyang and Just, Hoang Anh and Sun, Yifan and Jahagirdar, Himanshu and Zhang, Yuanzhi and Du, Rongxing and Sahu, Anit Kumar and Jia, Ruoxi},
  booktitle=ICLR,
  year={2024},
  address=ICLR2024_loc
}

@inproceedings{rencos,
  title={COS-DPO: Conditioned One-Shot Multi-Objective Fine-Tuning Framework},
  author={Ren, Yinuo and Xiao, Tesi and Shavlovsky, Michael and Ying, Lexing and Rahmanian, Holakou},
  booktitle={The 41st Conference on Uncertainty in Artificial Intelligence},
  year=2025
}

@inproceedings{xie2023data,
  title={Data selection for language models via importance resampling},
  author={Xie, Sang Michael and Santurkar, Shibani and Ma, Tengyu and Liang, Percy S},
  booktitle=NIPS,
  year={2023},
  address=NIPS2023_loc
}

@misc{eval-harness,
  author       = {Gao, Leo and Tow, Jonathan and Abbasi, Baber and Biderman, Stella and Black, Sid and DiPofi, Anthony and Foster, Charles and Golding, Laurence and Hsu, Jeffrey and Le Noac'h, Alain and Li, Haonan and McDonell, Kyle and Muennighoff, Niklas and Ociepa, Chris and Phang, Jason and Reynolds, Laria and Schoelkopf, Hailey and Skowron, Aviya and Sutawika, Lintang and Tang, Eric and Thite, Anish and Wang, Ben and Wang, Kevin and Zou, Andy},
  title        = {The Language Model Evaluation Harness},
  month        = 07,
  year         = 2024,
  publisher    = {Zenodo},
  version      = {v0.4.3},
  doi          = {10.5281/zenodo.12608602},
  url          = {https://zenodo.org/records/12608602}
}

@misc{alpaca_eval,
  author = {Xuechen Li and Tianyi Zhang and Yann Dubois and Rohan Taori and Ishaan Gulrajani and Carlos Guestrin and Percy Liang and Tatsunori B. Hashimoto },
  title = {AlpacaEval: An Automatic Evaluator of Instruction-following Models},
  year = {2023},
  month = {5},
  publisher = {GitHub},
  journal = {GitHub repository},
  howpublished = {\url{https://github.com/tatsu-lab/alpaca_eval}}
}

@inproceedings{huang2024lisa,
  title={Lisa: Lazy safety alignment for large language models against harmful fine-tuning attack},
  author={Huang, Tiansheng and Hu, Sihao and Ilhan, Fatih and Tekin, Selim F and Liu, Ling},
  booktitle=NIPS,
  volume={37},
  pages={104521--104555},
  year={2024}
}

@inproceedings{bianchisafety,
  title={Safety-Tuned LLaMAs: Lessons From Improving the Safety of Large Language Models that Follow Instructions},
  author={Bianchi, Federico and Suzgun, Mirac and Attanasio, Giuseppe and Rottger, Paul and Jurafsky, Dan and Hashimoto, Tatsunori and Zou, James},
  booktitle=ICLR,
  year={2024},
  address=ICLR2024_loc
}

@article{tse2007schemas,
  title={Schemas and memory consolidation},
  author={Tse, Dorothy and Langston, Rosamund F and Kakeyama, Masaki and Bethus, Ingrid and Spooner, Patrick A and Wood, Emma R and Witter, Menno P and Morris, Richard GM},
  journal={Science},
  volume={316},
  number={5821},
  pages={76--82},
  year={2007}
}

@article{botvinick2009hierarchically,
  title={Hierarchically organized behavior and its neural foundations: A reinforcement learning perspective},
  author={Botvinick, Matthew M and Niv, Yael and Barto, Andew G},
  journal={cognition},
  volume={113},
  number={3},
  pages={262--280},
  year={2009}
}

@article{kumaran2016learning,
  title={What learning systems do intelligent agents need? Complementary learning systems theory updated},
  author={Kumaran, Dharshan and Hassabis, Demis and McClelland, James L},
  journal={Trends in cognitive sciences},
  volume={20},
  number={7},
  pages={512--534},
  year={2016}
}

@article{kirkpatrick2017overcoming,
  title={Overcoming catastrophic forgetting in neural networks},
  author={Kirkpatrick, James and Pascanu, Razvan and Rabinowitz, Neil and Veness, Joel and Desjardins, Guillaume and Rusu, Andrei A and Milan, Kieran and Quan, John and Ramalho, Tiago and Grabska-Barwinska, Agnieszka and others},
  journal={Proceedings of the national academy of sciences},
  volume={114},
  number={13},
  pages={3521--3526},
  year={2017}
}

@article{parisi2019continual,
  title={Continual lifelong learning with neural networks: A review},
  author={Parisi, German I and Kemker, Ronald and Part, Jose L and Kanan, Christopher and Wermter, Stefan},
  journal={Neural networks},
  volume={113},
  pages={54--71},
  year={2019}
}

@inproceedings{behrouznested,
  title={Nested Learning: The Illusion of Deep Learning Architectures},
  author={Behrouz, Ali and Razaviyayn, Meisam and Zhong, Peilin and Mirrokni, Vahab},
  booktitle=NIPS,
  year=2025,
  address=NIPS2025_loc
}

@inproceedings{li2025pike,
  title={PiKE: Adaptive Data Mixing for Multi-Task Learning Under Low Gradient Conflicts},
  author={Li, Zeman and Deng, Yuan and Zhong, Peilin and Razaviyayn, Meisam and Mirrokni, Vahab},
  booktitle={ICLR 2025 Workshop on Navigating and Addressing Data Problems for Foundation Models},
  year=2025
}

@inproceedings{xiarethinking,
  title={Rethinking Data Selection at Scale: Random Selection is Almost All You Need},
  author={Xia, Tingyu and Yu, Bowen and Dang, Kai and Yang, An and Wu, Yuan and Tian, Yuan and Chang, Yi and Lin, Junyang},
   booktitle={Findings of the Association for Computational Linguistics: EMNLP 2025},
  year=2025
}

@inproceedings{xialess,
  title={LESS: Selecting Influential Data for Targeted Instruction Tuning},
  author={Xia, Mengzhou and Malladi, Sadhika and Gururangan, Suchin and Arora, Sanjeev and Chen, Danqi},
  booktitle=ICML,
  year=2024,
  address=ICML2024_loc
}

@inproceedings{huang2024context,
  title={In-context Convergence of Transformers},
  author={Huang, Yu and Cheng, Yuan and Liang, Yingbin},
  booktitle=ICML,
  year={2024},
  address=ICML2024_loc
}

@inproceedings{song2024unraveling,
  title={Unraveling the gradient descent dynamics of transformers},
  author={Song, Bingqing and Han, Boran and Zhang, Shuai and Ding, Jie and Hong, Mingyi},
  booktitle=NIPS,
  year={2024},
  address=NIPS2024_loc
}

@inproceedings{wang2024greats,
  title={Greats: Online selection of high-quality data for llm training in every iteration},
  author={Wang, Jiachen T and Wu, Tong and Song, Dawn and Mittal, Prateek and Jia, Ruoxi},
  booktitle=NIPS,
  year={2024}
}

@inproceedings{franceschi2017forward,
  title={Forward and reverse gradient-based hyperparameter optimization},
  author={Franceschi, Luca and Donini, Michele and Frasconi, Paolo and Pontil, Massimiliano},
  booktitle=ICML,
  year={2017},
  address=ICML2017_loc
}

@inproceedings{franceschi2018bilevel,
  title = {Bilevel Programming for Hyperparameter Optimization and Meta-Learning},
  author = {Franceschi, Luca and Frasconi, Paolo and Salzo, Saverio and Grazzi, Riccardo and Pontil, Massimilano},
  booktitle=ICML,
  year = {2018},
  address = ICML2018_loc
}

@inproceedings{grazzi2020iteration,
  title={On the iteration complexity of hypergradient computation},
  author={Grazzi, Riccardo and Franceschi, Luca and Pontil, Massimiliano and Salzo, Saverio},
  booktitle=ICML,
  year={2020},
  address=ICML2020_loc
}

@inproceedings{yang2021provably,
  title={Provably faster algorithms for bilevel optimization},
  author={Yang, Junjie and Ji, Kaiyi and Liang, Yingbin},
  booktitle=NIPS,
  year={2021},
  address=NIPS2021_loc
}

@book{boyd2004convex,
  title={Convex optimization},
  author={Boyd, Stephen and Boyd, Stephen P and Vandenberghe, Lieven},
  year={2004},
  publisher={Cambridge university press}
}

@inproceedings{arbel2021amortized,
  title={Amortized Implicit Differentiation for Stochastic Bilevel Optimization},
  author={Arbel, Michael and Mairal, Julien},
  booktitle=ICLR,
  year={2022},
  address=ICLR2022_loc
}

@article{shen2023penalty,
  title={On penalty-based bilevel gradient descent method},
  author={Shen, Han and Xiao, Quan and Chen, Tianyi},
  journal={Mathematical Programming},
  pages={1--51},
  year={2025},
  publisher={Springer}
}

@inproceedings{liu2023averaged,
  title={Averaged Method of Multipliers for Bi-Level Optimization without Lower-Level Strong Convexity},
  author={Liu, Risheng and Liu, Yaohua and Yao, Wei and Zeng, Shangzhi and Zhang, Jin},
  booktitle=ICML,
  year={2023},
  address=ICML2023_loc
}

@inproceedings{pedregosa2016hyperparameter,
  title={Hyperparameter optimization with approximate gradient},
  author={Pedregosa, Fabian},
  booktitle=ICML,
  year={2016},
  address=ICML2016_loc
}

@inproceedings{chen2021closing,
  title={Closing the Gap: Tighter Analysis of Alternating Stochastic Gradient Methods for Bilevel Problems},
  author={Chen, Tianyi and Sun, Yuejiao and Yin, Wotao},
  booktitle=NIPS,
  year={2021},
  address=NIPS2021_loc
}

@article{hong2020two,
  title={A two-timescale stochastic algorithm framework for bilevel optimization: Complexity analysis and application to actor-critic},
  author={Hong, Mingyi and Wai, Hoi-To and Wang, Zhaoran and Yang, Zhuoran},
  journal={SIAM Journal on Optimization},
  volume={33},
  number={1},
  pages={147--180},
  year={2023},
  publisher={SIAM}
}

@article{ghadimi2018approximation,
  title={Approximation methods for bilevel programming},
  author={Ghadimi, Saeed and Wang, Mengdi},
  journal={arXiv preprint arXiv:1802.02246},
  year={2018}
}

@inproceedings{ji2021bilevel,
  title={Bilevel optimization: Convergence analysis and enhanced design},
  author={Ji, Kaiyi and Yang, Junjie and Liang, Yingbin},
  booktitle=ICML,
  year={2021},
  address=ICML2021_loc
}

@inproceedings{li2022fully,
  title={A fully single loop algorithm for bilevel optimization without hessian inverse},
  author={Li, Junyi and Gu, Bin and Huang, Heng},
  booktitle=AAAI,
  year={2022},
  address=AAAI2022_loc
}

@inproceedings{khanduri2021near,
  title={A near-optimal algorithm for stochastic bilevel optimization via double-momentum},
  author={Khanduri, Prashant and Zeng, Siliang and Hong, Mingyi and Wai, Hoi-To and Wang, Zhaoran and Yang, Zhuoran},
  booktitle=NIPS,
  year={2021},
  address=NIPS2021_loc
}

@inproceedings{liu2020generic,
  title={A generic first-order algorithmic framework for bi-level programming beyond lower-level singleton},
  author={Liu, Risheng and Mu, Pan and Yuan, Xiaoming and Zeng, Shangzhi and Zhang, Jin},
  booktitle=ICML,
  year={2020},
  address=ICML2020_loc
}

@article{zhang2023introduction,
  title={An Introduction to Bi-level Optimization: Foundations and Applications in Signal Processing and Machine Learning},
  author={Zhang, Yihua and Khanduri, Prashant and Tsaknakis, Ioannis and Yao, Yuguang and Hong, Mingyi and Liu, Sijia},
  journal={arXiv preprint arXiv:2308.00788},
  year={2023}
}

@inproceedings{xiao2023generalized,
  title={A Generalized Alternating Method for Bilevel Optimization under the Polyak-{\L}ojasiewicz Condition},
  author={Xiao, Quan and Lu, Songtao and Chen, Tianyi},
  booktitle=NIPS,
  year={2023},
  address=NIPS2023_loc
}

@inproceedings{kwon2023fully,
  title={A fully first-order method for stochastic bilevel optimization},
  author={Kwon, Jeongyeol and Kwon, Dohyun and Wright, Stephen and Nowak, Robert D},
  booktitle=ICML,
  year={2023},
  address=ICML2023_loc
}

@article{lu2023first,
  title={First-order penalty methods for bilevel optimization},
  author={Lu, Zhaosong and Mei, Sanyou},
  journal={SIAM Journal on Optimization},
  volume={34},
  number={2},
  pages={1937--1969},
  year={2024},
  publisher={SIAM}
}

@inproceedings{wang2021fast,
  title={Fast algorithms for stackelberg prediction game with least squares loss},
  author={Wang, Jiali and Chen, He and Jiang, Rujun and Li, Xudong and Li, Zihao},
  booktitle=ICML,
  year={2021},
  address=ICML2021_loc
}

@inproceedings{wang2022solving,
  title={Solving Stackelberg Prediction Game with Least Squares Loss via Spherically Constrained Least Squares Reformulation},
  author={Wang, Jiali and Huang, Wen and Jiang, Rujun and Li, Xudong and Wang, Alex L},
  booktitle=ICML,
  year={2022},
  address=ICML2022_loc
}

@inproceedings{zhai2022understanding,
  title={Understanding Why Generalized Reweighting Does Not Improve Over {ERM}},
  author={Zhai, Runtian and Dan, Chen and Kolter, J Zico and Ravikumar, Pradeep Kumar},
  booktitle=ICLR,
  year={2023},
  address=ICLR2023_loc
}

@inproceedings{xiao2024unlocking,
  title={Unlocking Global Optimality in Bilevel Optimization: A Pilot Study},
  author={Xiao, Quan and Chen, Tianyi},
  booktitle=ICLR,
  year={2025},
  address=ICLR2025_loc
}

@inproceedings{kingma2014adam,
  title={Adam: A method for stochastic optimization},
  author={Kingma, Diederik P},
  booktitle=ICLR,
  year={2015}
}

@article{xiao2025ldc,
  title={LDC-MTL: Balancing Multi-Task Learning through Scalable Loss Discrepancy Control},
  author={Xiao, Peiyao and Dong, Chaosheng and Zou, Shaofeng and Ji, Kaiyi},
  journal={arXiv preprint arXiv:2502.08585},
  year={2025}
}

@article{bennouna2025data,
  title={What Data Enables Optimal Decisions? An Exact Characterization for Linear Optimization},
  author={Bennouna, Omar and Bennouna, Amine and Amin, Saurabh and Ozdaglar, Asuman},
  journal={arXiv preprint arXiv:2505.21692},
  year={2025}
}

@inproceedings{ghadikolaei2019learning,
  title={Learning and data selection in big datasets},
  author={Ghadikolaei, Hossein Shokri and Ghauch, Hadi and Fischione, Carlo and Skoglund, Mikael},
  booktitle=ICML,
  year={2019}
}

@inproceedings{deb2009solving,
  title={Solving bilevel multi-objective optimization problems using evolutionary algorithms},
  author={Deb, Kalyanmoy and Sinha, Ankur},
  booktitle={International conference on evolutionary multi-criterion optimization},
  pages={110--124},
  year={2009},
  organization={Springer}
}

@inproceedings{ip2025user,
  title={User Preference Meets Pareto-Optimality in Multi-Objective Bayesian Optimization},
  author={Ip, Joshua Hang Sai and Chakrabarty, Ankush and Mesbah, Ali and Romeres, Diego},
  booktitle={Proceedings of the AAAI Conference on Artificial Intelligence},
  volume={39},
  number={19},
  pages={20246--20254},
  year={2025}
}

@inproceedings{mahapatra2020multi,
  title={Multi-task learning with user preferences: Gradient descent with controlled ascent in pareto optimization},
  author={Mahapatra, Debabrata and Rajan, Vaibhav},
  booktitle=ICML,
  year={2020},
  address=ICML2020_loc
}

@inproceedings{chen2025efficient,
  title={Efficient First-Order Optimization on the Pareto Set for Multi-Objective Learning under Preference Guidance},
  author={Chen, Lisha and Xiao, Quan and Fukuda, Ellen Hidemi and Chen, Xinyi and Yuan, Kun and Chen, Tianyi},
  booktitle=ICML,
  year={2025},
  address=ICML2025_loc
}

@article{tanabe2024new,
  title={New merit functions for multiobjective optimization and their properties},
  author={Tanabe, Hiroki and Fukuda, Ellen H and Yamashita, Nobuo},
  journal={Optimization},
  volume={73},
  number={13},
  pages={3821--3858},
  year={2024}
}

@article{Dempe2020Semivectorial_bilevel,
  author    = {S. Dempe and P. Mehlitz},
  title     = {Semivectorial Bilevel Programming versus Scalar Bilevel Programming},
  journal   = {Optimization},
  volume    = {69},
  number    = {4},
  pages     = {657--679},
  year      = {2020},
}

@book{rockafellar2009variational,
  title={Variational analysis},
  author={Rockafellar, R Tyrrell and Wets, Roger J-B},
  volume={317},
  year={2009},
  publisher={Springer Science \& Business Media}
}

@article{soudry2018implicit,
  title={The implicit bias of gradient descent on separable data},
  author={Soudry, Daniel and Hoffer, Elad and Nacson, Mor Shpigel and Gunasekar, Suriya and Srebro, Nathan},
  journal={Journal of Machine Learning Research},
  volume={19},
  number={70},
  pages={1--57},
  year={2018}
}

@inproceedings{mamou2020emergence,
  title={Emergence of Separable Manifolds in Deep Language Representations},
  author={Mamou, Jonathan and Le, Hang and Del Rio, Miguel and Stephenson, Cory and Tang, Hanlin and Kim, Yoon and Chung, Sueyeon},
  booktitle=ICML,
  year={2020},
  address=ICML2020_loc
}

@article{zhang2016understanding,
  title={Understanding deep learning requires rethinking generalization},
  author={Zhang, Chiyuan and Bengio, Samy and Hardt, Moritz and Recht, Benjamin and Vinyals, Oriol},
  journal={arXiv preprint arXiv:1611.03530},
  year={2016}
}

@inproceedings{allen2019convergence,
  title={A convergence theory for deep learning via over-parameterization},
  author={Allen-Zhu, Zeyuan and Li, Yuanzhi and Song, Zhao},
  booktitle=ICML,
  pages={242--252},
  year={2019}
}

@article{gao2017properties,
  title={On the properties of the softmax function with application in game theory and reinforcement learning},
  author={Gao, Bolin and Pavel, Lacra},
  journal={arXiv preprint arXiv:1704.00805},
  year={2017}
}

@inproceedings{zhang2022unlabeled,
  title={How unlabeled data improve generalization in self-training? A one-hidden-layer theoretical analysis},
  author={Zhang, Shuai and Weng, Meng and Liu, Sijia and Chen, Pin-Yu and Xiong, Jinjun},
  booktitle=ICLR,
  year={2022},
  address=ICLR2022_loc
}
\bibliographystyle{plainnat}

\newpage
\appendix
\onecolumn

\begin{center}
{\Large \bf Supplementary Material }
\end{center}
\vspace{-1cm}

\doparttoc 
\faketableofcontents 


\makeatletter
\renewcommand{\partname}{} 
\renewcommand{\thepart}{}  
\makeatother
\part{} 
\parttoc 

\section{Convexity of SFT loss with respect to the backbone model} \label{sec:convex_SFT} 

With causal masking in attention \citep{vaswani2017attention}, the backbone model $\phi_\theta(x,y)$ is unable to see the future tokens before predicting, even if we input the whole sequence \citep{ren2024learning}. Therefore, the SFT loss can be viewed as the token-level cross-entropy loss of a sequential multi-class classification problem, where the label is the next token in the response $y$. Due to the nonlinearity of $\phi_\theta(x, y)$, SFT loss might not be convex with respect to $\theta$, but it is convex with respect to the backbone representation $z = \phi_\theta(x, y)\in\mathbb{R}^{V\times D}$. 

\begin{lemma} 
\label{lemma:convex_SFT_main_paper}
Per-sample SFT loss in \eqref{eq:SFT_def} is convex with respect to the backbone model $z = \phi_\theta(x, y)\in\mathbb{R}^{V\times D}$. 
\end{lemma} 

Due to the nonlinear backbone model $\phi_\theta(x, y)$, SFT loss might not be convex with respect to $\theta$, which makes the optimization landscape complicated. Nevertheless, the building block of our theory lies in the convexity of the SFT loss with respect to the backbone representation $z = \phi_\theta(x, y)\in\mathbb{R}^{V\times D}$, which makes the SFT loss a composite convex function over $\theta$.

\begin{proof}

According to \eqref{eq:SFT_def}, SFT loss takes the form of 
\begin{align*}\label{eq:SFT_backbone}
\mathcal{L}_{\text{SFT}}(\theta;x,y)&=-\sum_{d=1}^D \mathbf{e}_{y_d}^\top \log \pi_\theta(y_d\mid x,y_{<d})=-\sum_{d=1}^D \mathbf{e}_{y_d}^\top\log\sigma\left(\phi_\theta(x,y)\right)_{[:,d]}\\
&=-\sum_{d=1}^D \mathbf{e}_{y_d}^\top\log \sigma(z)_{[:,d]} =-\sum_{d=1}^D \mathbf{e}_{y_d}^\top\log \sigma(z_{[:,d]}).  \numberthis
\end{align*}
The Jacobian of softmax function is well-known \citep{gao2017properties} and is given by 
\begin{align*}
\nabla\sigma(z_{[:,d]})=\operatorname{diag}\left(\sigma(z_{[:,d]})\right)-\sigma(z_{[:,d]})\sigma(z_{[:,d]})^\top. 
\end{align*}
Therefore, the derivative of $\mathcal{L}_{\text{SFT}}(\theta;x,y)$ with respect to $z_{[:,d]}$ is given by 
\begin{align*}
\frac{\partial\mathcal{L}_{\text{SFT}}(\theta;x,y)}{\partial z_{[:,d]}}&=\nabla\sigma(z_{[:,d]})\frac{\partial\mathcal{L}_{\text{SFT}}(\theta;x,y)}{\partial \sigma (z_{[:,d]})}\\
&=-\left[\operatorname{diag}\left(\sigma(z_{[:,d]})\right)-\sigma(z_{[:,d]})\sigma(z_{[:,d]})^\top\right]\left[\frac{(\mathbf{e}_{y_d})_1}{\sigma(z_{[:,d]})_1},\cdots,\frac{(\mathbf{e}_{y_d})_V}{\sigma(z_{[:,d]})_V}\right]^\top\\
&\stackrel{(a)}{=}\sigma(z_{[:, d]})-\mathbf{e}_{y_d}\numberthis\label{Jacobian_individual} 
\end{align*}
where $(\mathbf{e}_{y_d})_i$ denotes the $i$-the element in the one-hot vector $\mathbf{e}_{y_d}\in\mathbb{R}^V$ and $\sigma(z_{[:,d]})_i$ denotes the $i$-the element in the softmax vector $\sigma(z_{[:,d]})\in\mathbb{R}^V$, and $(a)$ uses the fact that $\sum_i(\mathbf{e}_{y_d})_i=1$. Moreover, the Jacobian of $\mathcal{L}_{\text{SFT}}(\theta;x,y)$ with respect to $z$ is given by
\begin{align}\label{Jacobian_sum} 
\frac{\partial \mathcal{L}_{\text{SFT}}(\theta; x, y)}{\partial z}
&= \left[
\frac{\partial \mathcal{L}_{\text{SFT}}(\theta; x, y)}{\partial z_{[:,1]}},
\; \cdots \;,
\frac{\partial \mathcal{L}_{\text{SFT}}(\theta; x, y)}{\partial z_{[:,D]}}
\right]\nonumber\\
&=\left[
\sigma(z_{[:, 1]})-\mathbf{e}_{y_1},
\; \cdots \;,
\sigma(z_{[:, D]})-\mathbf{e}_{y_D}
\right]
\end{align}
On the other hand, the second-order derivative of $\mathcal{L}_{\text{SFT}}(\theta;x,y)$ with respect to $z_{[:,d]}$ is given by 
\begin{align*}
\frac{\partial^2\mathcal{L}_{\text{SFT}}(\theta;x,y)}{\partial^2 z_{[:,d]}}&=\nabla\sigma(z_{[:, d]}), ~\text{ and } ~\frac{\partial^2\mathcal{L}_{\text{SFT}}(\theta;x,y)}{\partial z_{[:,d]}\partial z_{[:,\tilde d]}}=\textbf{0}^{V\times D}, \text{ if } \tilde d\neq d. \numberthis\label{Hessian_individual} 
\end{align*}
where $\textbf{0}^{V\times D}\in\mathbb{R}^{V\times D}$ is the zero matrix. The Hessian of $\mathcal{L}_{\text{SFT}}(\theta;x,y)$ over $z$ is defined as the Hessian of the vectorized $z$, i.e. $\frac{\partial^2 \mathcal{L}_{\text{SFT}}(\theta; x, y)}{\partial z}\in\mathbb{R}^{Vd\times Vd}$. We define the vectorized $z$ as 
\begin{align}\label{vector_matrix}
\operatorname{vec}(z)=[z_{11}, z_{21},\cdots, z_{V1},\cdots, z_{1D},z_{2n},\cdots z_{VD}]
\end{align} 
Let $i=(d_1-1)V+v_1$ and $j=(d_2-1)V+v_2$, then each element of the Hessian is   
\begin{align}\label{Hessian_sum} 
\left[\frac{\partial^2 \mathcal{L}_{\text{SFT}}(\theta; x, y)}{\partial z}\right]_{ij}
&= \left[\frac{\partial^2 \mathcal{L}_{\text{SFT}}(\theta; x, y)}{\partial \operatorname{vec}(z)}\right]_{ij}=
\begin{cases}
\nabla\sigma(z_{[:, d_1]})_{v_1,v_2} & \text{if } d_1=d_2 \\
0 & \text{if } d_1\neq d_2
\end{cases}
\end{align}
Then for any matrix $u\in\mathbb{R}^{V\times D}$, it holds that 
\begin{align*}
&~~~~~\operatorname{vec}(u)^\top \frac{\partial^2 \mathcal{L}_{\text{SFT}}(\theta; x, y)}{\partial z} \operatorname{vec}(u)\\
&=\sum_{i=1}^{VD} \sum_{j=1}^{VD}\operatorname{vec}(u)_i^\top  \left[\frac{\partial^2 \mathcal{L}_{\text{SFT}}(\theta; x, y)}{\partial z}\right]_{ij} \operatorname{vec}(u)_j\\
&=\sum_{d_1=1}^{D}\sum_{d_2=1}^{D}\sum_{v_1=1}^{V}\sum_{v_2=1}^{V} \operatorname{vec}(u)_{((d_1-1)V+v_1)}  \left[\frac{\partial^2 \mathcal{L}_{\text{SFT}}(\theta; x, y)}{\partial z}\right]_{((d_1-1)V+v_1)((d_2-1)V+v_2)} \operatorname{vec}(u)_{((d_2-1)V+v_2)}\\
&\stackrel{\eqref{Hessian_sum}}{=}\sum_{d=1}^{D}\sum_{v_1=1}^{V}\sum_{v_2=1}^{V} \operatorname{vec}(u)_{((d-1)V+v_1)}\left[\frac{\partial^2 \mathcal{L}_{\text{SFT}}(\theta; x, y)}{\partial z}\right]_{((d-1)V+v_1)((d-1)V+v_2)} \operatorname{vec}(u)_{((d-1)V+v_2)}\\
&=\sum_{d=1}^{D}\sum_{v_1=1}^{V}\sum_{v_2=1}^{V} u_{v_1,d}\left[\frac{\partial^2 \mathcal{L}_{\text{SFT}}(\theta; x, y)}{\partial z}\right]_{((d-1)V+v_1)((d-1)V+v_2)} u_{v_2,d}\\
&\stackrel{\eqref{Hessian_sum}}{=}\sum_{d=1}^{D}\sum_{v_1=1}^{V}\sum_{v_2=1}^{V} u_{v_1,d}\nabla \sigma(z_{[:,d]})_{v_1,v_2} u_{v_2,d}\\
&=\sum_{d=1}^{D}u_{[:,d]}^\top \nabla \sigma(z_{[:,d]})u_{[:,d]}\geq 0
\end{align*}
where the last inequality follows from the positive semi-definiteness of each gradient of softmax function $\nabla \sigma(z_{[:,d]})$ \citep{gao2017properties}. Therefore, the Hessian $\frac{\partial^2 \mathcal{L}_{\text{SFT}}(\theta; x, y)}{\partial z}$ is positive semi-definite for any $z$, suggesting that $\mathcal{L}_{\text{SFT}}(\theta; x, y)$  is convex over $z$, which completes the proof.



\end{proof}

\section{Proofs of the main theorems} 
\label{sec:proof_theorem}

\subsection{Proof of Theorem \ref{thm:benefits_bdr}}
\label{sec:proof_benefits}

Let $v_*:=\min_\theta \mathcal{L}_{\rm val}(\theta)$ and 
assume there exists a feasible pair $(\bar\omega,\bar\theta)$ for \texttt{BDS} such that $\mathcal{L}_{\rm val}(\bar\theta)=v_*$. Because $\min\mathcal{L}_{\rm val}$ and $\min\frac{1}{N}\sum_{i=1}^N \mathcal{L}_{\mathrm{SFT}}(\theta;x^i,y^i)$ do not share the optimal model $\theta$, then we know 
\[
\Delta_{\rm val}:=\min_{\theta\in\mathcal{S}_{\mathrm{val}}}\frac{1}{N}\sum_{i=1}^N \mathcal{L}_{\mathrm{SFT}}(\theta;x^i,y^i)>0,
\]
\[
\Delta_{\rm sft}:=\min_{\theta\in\mathcal{S}_{\rm sft}}\mathcal{L}_{\rm val}(\theta)-v_*>0.
\]
We will show that the conclusion of Theorem \ref{thm:benefits_bdr} holds with the explicit threshold
\[
\rho_c:=\frac{\Delta_{\rm sft}}{\Delta_{\rm val}+\Delta_{\rm sft}}.
\]

\begin{proof}
By assumption, there exists a feasible pair $(\bar\omega,\bar\theta)$ for \texttt{BDS} such that $\mathcal{L}_{\rm val}(\bar\theta)=v_*$. Since \texttt{BDS} minimizes the validation loss over its feasible set, any global solution $(\omega^*,\theta^*)$ satisfies
\[
\mathcal{L}_{\rm val}(\theta^*)\leq \mathcal{L}_{\rm val}(\bar\theta)=v_*.
\]
Because $v_*=\min_\theta \mathcal{L}_{\rm val}(\theta)$ is the global minimum of the validation objective, we also have $\mathcal{L}_{\rm val}(\theta^*)\geq v_*$. Hence $\mathcal{L}_{\rm val}(\theta^*)=v_*.$

Fix any $\rho>\rho_c$. Let $\tilde\theta\in\mathcal{S}_{\rm mix}(\rho)$ be a global minimizer of the direct-mixing objective
\[
\frac{\rho}{N}\sum_{i=1}^N \mathcal{L}_{\mathrm{SFT}}(\theta;x^i,y^i)+(1-\rho)\mathcal{L}_{\rm val}(\theta).
\]
We claim that $\mathcal{L}_{\rm val}(\tilde\theta)>v_*$. Otherwise, we have $\mathcal{L}_{\rm val}(\tilde\theta)=v_*$. Then, by the definition of $\Delta_{\rm val}$,
\[
\frac{\rho}{N}\sum_{i=1}^N \mathcal{L}_{\mathrm{SFT}}(\tilde\theta;x^i,y^i)+(1-\rho)\mathcal{L}_{\rm val}(\tilde\theta)
\geq
\rho \Delta_{\rm val}+(1-\rho)v_*.
\]
On the other hand, choose any $\theta_{\rm sft}\in \mathcal{S}_{\rm sft}$ such that
\[
\mathcal{L}_{\rm val}(\theta_{\rm sft})=v_*+\Delta_{\rm sft},
\]
which exists by the definition of $\Delta_{\rm sft}$. Since $\frac{1}{N}\sum_{i=1}^N \mathcal{L}_{\mathrm{SFT}}(\theta_{\rm sft};x^i,y^i)=0$ under Assumption \ref{ass:seperable}, its direct-mixing objective value is
\[
\frac{\rho}{N}\sum_{i=1}^N \mathcal{L}_{\mathrm{SFT}}(\theta_{\rm sft};x^i,y^i)+(1-\rho)\mathcal{L}_{\rm val}(\theta_{\rm sft})
=(1-\rho)(v_*+\Delta_{\rm sft}).
\]
Because
\[
\rho>\rho_c=\frac{\Delta_{\rm sft}}{\Delta_{\rm val}+\Delta_{\rm sft}}
\quad\Longleftrightarrow\quad
\rho \Delta_{\rm val}>(1-\rho)\Delta_{\rm sft},
\]
we obtain
\begin{align*}
\frac{\rho}{N}\sum_{i=1}^N \mathcal{L}_{\mathrm{SFT}}(\tilde\theta;x^i,y^i)+(1-\rho)\mathcal{L}_{\rm val}(\tilde\theta)&=\rho \Delta_{\rm val}+(1-\rho)v_*>(1-\rho)(v_*+\Delta_{\rm sft})\\
&=\frac{\rho}{N}\sum_{i=1}^N \mathcal{L}_{\mathrm{SFT}}(\theta_{\rm sft};x^i,y^i)+(1-\rho)\mathcal{L}_{\rm val}(\theta_{\rm sft}),
\end{align*}
which contradicts the optimality of $\tilde\theta$. Therefore, every direct-mixing minimizer satisfies $\mathcal{L}_{\rm val}(\tilde\theta)>v_*$.

Combining this with $\mathcal{L}_{\rm val}(\theta^*)=v_*$ yields
\[
\mathcal{L}_{\rm val}(\theta^*)=v_*<\mathcal{L}_{\rm val}(\tilde\theta),
\]
which completes the proof.
\end{proof}

\subsection{Proof of Proposition \ref{prop:coverage_bdr}}
\label{sec:proof_coverage_bdr}

\begin{proof}
Let $(\omega^*,\theta^*)$ be any global solution of \texttt{BDS}, and write \(\lambda^*:=\sigma(\omega^*)\in\Delta^N\).
Let
\[
\mathcal{U}:=\left\{i\in[N]~\middle|~\exists \theta\in \mathcal{S}_{\rm val}\text{ such that }\mathcal{L}_{\mathrm{SFT}}(\theta;x^i,y^i)=0\right\}.
\]

We first show that \(\theta^*\in\mathcal{S}_{\rm val}\). By Assumption \ref{ass:setting}, there exists an index \(i_0\in[N]\) such that
\[
\min_{\theta\in\mathcal{S}_{\rm val}}\mathcal{L}_{\mathrm{SFT}}(\theta;x^{i_0},y^{i_0})=0.
\]
Choose \(\bar\theta\in\mathcal{S}_{\rm val}\) satisfying \(\mathcal{L}_{\mathrm{SFT}}(\bar\theta;x^{i_0},y^{i_0})=0\), and let \(\bar\lambda=e_{i_0}\in\Delta^N\). Then \(\bar\theta\in\argmin_{\theta^\prime}\frac1N\sum_{i=1}^N\bar\lambda_i\mathcal{L}_{\mathrm{SFT}}(\theta^\prime;x^i,y^i)\), so \((\bar\lambda,\bar\theta)\) is feasible for \texttt{BDS}. Since \(\mathcal{L}_{\rm val}(\bar\theta)=v_*\), every global solution of \texttt{BDS} must satisfy
\[
\mathcal{L}_{\rm val}(\theta^*)=v_*,
\]
that is, \(\theta^*\in\mathcal{S}_{\rm val}\).

Next, because \((\lambda^*,\theta^*)\) is feasible for \texttt{BDS} and Assumption \ref{ass:seperable} gives a parameter \(\hat\theta\) with \(\mathcal{L}_{\mathrm{SFT}}(\hat\theta;x^i,y^i)=0\) for all \(i\in[N]\), the optimal value of the lower-level weighted problem is \(0\). Hence
\[
\frac1N\sum_{i=1}^N \lambda_i^* \mathcal{L}_{\mathrm{SFT}}(\theta^*;x^i,y^i)=0.
\]
Since every term is nonnegative, \(\lambda_i^*>0\) implies
\[
\mathcal{L}_{\mathrm{SFT}}(\theta^*;x^i,y^i)=0.
\]
Because \(\theta^*\in\mathcal{S}_{\rm val}\), every index in the support of \(\lambda^*\) belongs to the useful set \(\mathcal{U}\). Therefore
\[
\sum_{i\in\mathcal U}\lambda_i^*=\sum_{i=1}^N\lambda_i^*=1,
\]
which proves the first claim.

For direct mixing, the lower-level weights are always uniform, namely
\[
\lambda^{\rm mix}:=\left(\frac1N,\ldots,\frac1N\right).
\]
Therefore
\[
\sum_{i\notin\mathcal U}\lambda_i^{\rm mix}
=\sum_{i\notin\mathcal U}\frac1N
=1-\frac{|\mathcal U|}{N}.
\]
If useless samples exist, then \(|\mathcal U|<N\), so \(\sum_{i\notin\mathcal U}\lambda_i^{\rm mix}>0\) for every \(\rho>0\).
\end{proof}

\subsection{Proof of Theorem \ref{thm:offline_pgdc}}
\label{sec:proof_offline_pgdc}
In this section, we provide the proof of axillary lemmas and theorems for proving Theorem~\ref{thm:offline_pgdc}. 

First, the deterministic offline PBGD iterates are summarized in Algorithm~\ref{alg:offline_PBGD}.
\begin{algorithm}[t]
\caption{\texttt{PBGD} for offline \texttt{BDS}}
\begin{algorithmic}[1]
\STATE Initialize $\omega^0$ and $\theta^0=(W_{kq}^0,W_{ov}^0)$; choose step sizes $\{\eta_k\}_{k\ge0}$.
\FOR{$k=0,1,2,\ldots$}
\STATE Update all variables by the full-gradient penalty step
\STATE \begin{minipage}[t]{0.95\linewidth}
\begin{subequations}\label{eq:offline_PBGD_det_appendix}
\begin{align}
    \omega^{k+1}
    &=
    \omega^k-\eta_k\nabla_\omega \mathcal{L}_\gamma(\omega^k,\theta^k),\\
    W_{ov}^{k+1}
    &=
    W_{ov}^k-\eta_k\nabla_{W_{ov}}\mathcal{L}_\gamma(\omega^k,\theta^k),\\
    W_{kq}^{k+1}
    &=
    W_{kq}^k-\eta_k\nabla_{W_{kq}}\mathcal{L}_\gamma(\omega^k,\theta^k).
\end{align}
\end{subequations}
\end{minipage}
\ENDFOR
\end{algorithmic}
\label{alg:offline_PBGD}
\end{algorithm}
We then introduce additional notations in this section. 
For each training sample $(x^i,y^i)$ and token position $d\in[D]$, let
$X_d^i=\mathsf{Rep}(x^i,y^i_{<d})\in\mathbb R^{V\times L_d}$ denote the
masked context representation used to predict $y_d^i$, where
$\mathsf{Rep}(\cdot)$ denotes the embedding and causal-masking operation
and $L_d=d_x+d-1$. Let $X_{d,-1}^i\in\mathbb R^V$ denote the query column
in $X_d^i$ associated with the prediction of $y_d^i$. Similarly, for each
validation sample $(\tilde x^j,\tilde y^j)$, define
$\tilde X_d^j=\mathsf{Rep}(\tilde x^j,\tilde y^j_{<d})
\in\mathbb R^{V\times L_d}$ and
$\tilde X_{d,-1}^j\in\mathbb R^V$. We define 
\begin{align*}
    \mathcal{S}_\mathrm{token}=\{(X_d^i,y_d^i)\}_{i\in[N],d\in[D]}\cup\{(\widetilde X_d^j,\widetilde y_d^j)\}_{j\in[N'],d\in[D]}.
\end{align*}

For any token-level
training or validation sample$(X,y)\in\mathcal{S}_\mathrm{token}$, in the one-layer softmax Transformer, define
\begin{align*}
    B_\theta(X)
    :=
    X^\top W_{kq}X_{-1},\quad
    h_\theta(X)
    :=
    X\sigma(B_\theta(X)),\quad
    U_\theta(X)
    :=
    W_{ov}h_\theta(X),
\end{align*}
and the weighted token-level feature matrix
\[
H(\omega,\theta)
:=
\left[
\left\{
\frac{1}{\sqrt{N'D}}h_\theta(\widetilde X_d^j)
\right\}_{j\in[N'],\,d\in[D]},
\left\{
\sqrt{\frac{\gamma\sigma_i(\omega)}{ND}}h_\theta(X_d^i)
\right\}_{i\in[N],\,d\in[D]}
\right]
\in\mathbb R^{V\times (N'+N)D}.
\]
Similarly, define the weighted logit-gradient matrix
{\small\begin{align*}
G(\omega,\theta)
:=
\left[
\left\{
\frac{1}{\sqrt{N'D}}
\nabla g_{\widetilde y_d^j}\!\left(U_\theta(\widetilde X_d^j)\right)
\right\}_{j\in[N'],\,d\in[D]},
\left\{
\sqrt{\frac{\gamma\sigma_i(\omega)}{ND}}
\nabla g_{y_d^i}\!\left(U_\theta(X_d^i)\right)
\right\}_{i\in[N],\,d\in[D]}
\right]
\in\mathbb R^{V\times (N'+N)D}
\end{align*}}

where the columns of \(G(\omega,\theta)\) and \(H(\omega,\theta)\) are ordered in the
same way.

For notational simplicity, along the offline PBGD trajectory we write
$H^k:=H(\omega^k,\theta^k)$. We assume the overparameterized regime
$V\ge (N+N')D$, so that the feature matrix
$H(\omega,W_{kq})\in\mathbb R^{V\times (N+N')D}$ can be full column rank, i.e., the condition $\sigma_{\min}(H(\omega,W_{kq}))>0$.


For a single token $y$, define
$g_y(u):=-\mathbf e_y^\top\log\sigma(u)=-\log[\sigma(u)]_y$. 
For any $(X,y)\in \mathcal{S}_{\mathrm{token}}$, define $l(\theta;X,y)= g_y(U_\theta(X))$.
Then
\begin{align*}
\mathcal L_{\rm SFT}(\theta;x^i,y^i)
&=
\frac1D\sum_{d=1}^{D}
g_{y_d^i}\!\left(U_\theta(X_d^i)\right)=\frac{1}{D}\sum_{d=1}^Dl(\theta;X^i_d,y^i_d),\\
\mathcal L_{\rm SFT}(\theta;\tilde x^j,\tilde y^j)
&=
\frac1D\sum_{d=1}^{D}
g_{\tilde y_d^j}\!\left(U_\theta(\tilde X_d^j)\right)=\frac1D\sum_{d=1}^{D}l(\theta;\tilde X^j_d,\tilde y^j_d).    
\end{align*}

For any integer \(m\ge 2\), we use $\sigma$ to denote the softmax map on \(\mathbb R^m\). Its Jacobian is denoted by
\[
J_\sigma(z)
:= 
\operatorname{Diag}(\sigma(z))-\sigma(z)\sigma(z)^\top
\in\mathbb R^{m\times m},
\qquad z\in\mathbb R^m.
\]

\begin{assumption}[Bounded token-level geometry]
\label{ass:token_geometry}
For the tokenized input of training samples $X^i$ and validation samples $\tilde X^j$, define
\[
B_X
:=
\max\left\{
\max_{i\in[N],\,d\in[D]}\sigma_{\max}(X_d^i),
\max_{j\in[N'],\,d\in[D]}\sigma_{\max}(\tilde X_d^j)
\right\},
\]
and
\[
B_{-1}
:=
\max\left\{
\max_{i\in[N],\,d\in[D]}\|X_{d,-1}^i\|_2,
\max_{j\in[N'],\,d\in[D]}\|\tilde X_{d,-1}^j\|_2
\right\}.
\]
We assume that $B_X<\infty$ and $B_{-1}<\infty$.
\end{assumption}

\subsubsection{Sample-level regularity}
We denote two parameters $\theta=(W_{ov},W_{kq})$, $\theta'=(W_{ov}',W_{kq}')$.
We also denote $R(\theta,\theta'):=\max\{\|W_{ov}\|_F,\|W_{ov}'\|_F\}$.
Under Assumption~\ref{ass:token_geometry}, all token-level samples considered in
this subsection satisfy $\sigma_{\max}(X)\le B_X, \|X_{-1}\|_2\le B_{-1}.$

\begin{lemma}[Sample gradient formulas]
\label{lem:sample-gradient-formula}
Under the above notations, for any
\((X,y)\in\mathcal S_{\mathrm{token}}\), the sample gradients satisfy
\begin{align*}
    \nabla_{W_{\mathrm{ov}}}l(\theta;X,y)
    &=
    \bigl(\sigma(U_\theta(X))-\mathbf e_{I(y)}\bigr)
    h_\theta(X)^\top,\\
    \nabla_{W_{\mathrm{kq}}}l(\theta;X,y)
    &=
    XJ_\sigma(B_\theta(X))X^\top W_{\mathrm{ov}}^\top
    \bigl(\sigma(U_\theta(X))-\mathbf e_{I(y)}\bigr)
    X_{-1}^\top.
\end{align*}
\end{lemma}

\begin{lemma}[Gradient formulas for the penalty objective]\label{lem:penalty-gradient-formula}
Under the above notations, recall that
\begin{align*}
\mathcal L_\gamma(\omega,\theta)
=
\frac{1}{N'}\sum_{j=1}^{N'}\mathcal L_{\mathrm{SFT}}(\theta;\widetilde X^{\,j},\widetilde y^{\,j})
+
\gamma\sum_{i=1}^{N}[\sigma(\omega)]_i\,\mathcal L_{\mathrm{SFT}}(\theta;X^i,y^i).
\end{align*}
Then the gradients with respect to $W_\mathrm{ov}, W_\mathrm{kq}$ and $\omega$ satisfy
\begin{align*}
\nabla_{W_{\mathrm{ov}}}\mathcal L_\gamma(\omega,\theta)
&=
\frac{1}{N'D}\sum_{j=1}^{N'}\sum_{d=1}^D\nabla_{W_{\mathrm{ov}}}l(\theta;\widetilde X^{\,j}_d,\widetilde y^{\,j}_d)
+
\frac{\gamma}{ND}\sum_{i=1}^{N}\sum_{d=1}^D[\sigma(\omega)]_i\,\nabla_{W_{\mathrm{ov}}}l(\theta;X^i_d,y^i_d),\\
\nabla_{W_{\mathrm{kq}}}\mathcal L_\gamma(\omega,\theta)
&=
\frac{1}{N'D}\sum_{j=1}^{N'}\sum_{d=1}^D\nabla_{W_{\mathrm{kq}}}l(\theta;\widetilde X^{\,j}_d,\widetilde y^{\,j}_d)
+
\frac{\gamma}{ND}\sum_{i=1}^{N}\sum_{d=1}^D[\sigma(\omega)]_i\,\nabla_{W_{\mathrm{kq}}}l(\theta;X^i_d,y^i_d),
\end{align*}
and
\begin{align*}
\nabla_{\omega}\mathcal L_\gamma(\omega,\theta)
&=
\frac{\gamma}{N}\,J_\sigma(\omega)\,\ell_{\mathrm{seq}}(\theta),\quad
\ell_{\mathrm{seq}}(\theta)
:=
\begin{bmatrix}
\mathcal L_{\mathrm{SFT}}(\theta;X^1,y^1)\\
\vdots\\
\mathcal L_{\mathrm{SFT}}(\theta;X^N,y^N)
\end{bmatrix}
=
\frac1D
\begin{bmatrix}
\sum_{d=1}^D l(\theta;X_d^1,y_d^1)\\
\vdots\\
\sum_{d=1}^D l(\theta;X_d^N,y_d^N)
\end{bmatrix}.
\end{align*}
\end{lemma}

\begin{lemma}[Lipschitz continuity of softmax and its Jacobian]
\label{lem:softmax-jacobian-basic}
Let \(m\ge 2\) be arbitrary, and let
\(\sigma:\mathbb R^m\to\mathbb R^m\) be the softmax map. For any
\(z\in\mathbb R^m\), its Jacobian is
\[
J_\sigma(z)
=
\operatorname{Diag}(\sigma(z))-\sigma(z)\sigma(z)^\top .
\]
Then, for all \(a,b,z\in\mathbb R^m\), the following hold:
\begin{enumerate}
    \item[\textup{(i)}]
    $\|J_\sigma(z)\|_2\le \frac12,
    \qquad
    \|\sigma(a)-\sigma(b)\|_2
    \le
    \frac12\|a-b\|_2.$
    \item[\textup{(ii)}]
    $\|J_\sigma(a)-J_\sigma(b)\|_2
    \le
    \frac32\|a-b\|_2.$
    \item[\textup{(iii)}]
    $
    \|\sqrt{\sigma(a)}-\sqrt{\sigma(b)}\|_2
    \le
    \frac{\sqrt2}{4}\|a-b\|_2,$
    where the square root is taken componentwise.
\end{enumerate}
\end{lemma}

\begin{proof}
For \textup{(i)}, let \(p=\sigma(z)\). Direct differentiation gives
\[
[J_\sigma(z)]_{ij}
=
p_i(\mathbf 1_{\{i=j\}}-p_j),
\qquad
J_\sigma(z)
=
\operatorname{Diag}(p)-pp^\top .
\]
For any \(v\in\mathbb R^m\),
\[
v^\top J_\sigma(z)v
=
\sum_{i=1}^m p_i v_i^2
-
\left(\sum_{i=1}^m p_i v_i\right)^2
=
\sum_{1\le i<j\le m}p_ip_j(v_i-v_j)^2 .
\]
Using \((v_i-v_j)^2\le 2(v_i^2+v_j^2)\) and $p_i(1-p_i)\leq\frac{1}{2}$, we have
\[
v^\top J_\sigma(z)v
\le
2\sum_{i=1}^m p_i(1-p_i)v_i^2
\le
\frac12\|v\|_2^2 .
\]
Since \(J_\sigma(z)\) is symmetric positive semidefinite, this implies
\(\|J_\sigma(z)\|_2\le 1/2\). And therefore $\|\sigma(a)-\sigma(b)\|_2
\le
\frac12\|a-b\|_2$.

For \textup{(ii)}, let \(p=\sigma(a)\), \(q=\sigma(b)\), and \(d=p-q\).
Then
\[
J_\sigma(a)-J_\sigma(b)
=
\operatorname{Diag}(d)-(pp^\top-qq^\top).
\]
Since \(pp^\top-qq^\top=pd^\top+dq^\top\), we obtain
\[
\|J_\sigma(a)-J_\sigma(b)\|_2
\le
\|\operatorname{Diag}(d)\|_2
+
\|pd^\top+dq^\top\|_2
\le
\|d\|_2+(\|p\|_2+\|q\|_2)\|d\|_2.
\]
As \(p\) and \(q\) are probability vectors, \(\|p\|_2\le1\) and
\(\|q\|_2\le1\). Hence
\[
\|J_\sigma(a)-J_\sigma(b)\|_2
\le
3\|\sigma(a)-\sigma(b)\|_2
\le
\frac32\|a-b\|_2.
\]

For \textup{(iii)}, define \(F(z):=\sqrt{\sigma(z)}\), where the square root
is componentwise. For \(p=\sigma(z)\) and any direction \(v\in\mathbb R^m\),
\[
[DF(z)v]_i
=
\frac{1}{2\sqrt{p_i}}[J_\sigma(z)v]_i
=
\frac12\sqrt{p_i}(v_i-p^\top v).
\]
Therefore,
\[
\|DF(z)v\|_2^2
=
\frac14\sum_{i=1}^m p_i(v_i-p^\top v)^2
=
\frac14 v^\top J_\sigma(z)v
\le
\frac18\|v\|_2^2 .
\]
Thus \(\|DF(z)\|_2\le \sqrt2/4\). Applying the fundamental theorem of
calculus to \(F\) along the segment from \(b\) to \(a\), we get
\[
\|\sqrt{\sigma(a)}-\sqrt{\sigma(b)}\|_2
=
\|F(a)-F(b)\|_2
\le
\frac{\sqrt2}{4}\|a-b\|_2.
\]
This completes the proof.
\end{proof}

\begin{lemma}[Lipschitz continuity of the hidden feature map]\label{lem:h-basic}
Under Assumption~\ref{ass:token_geometry}. For any fixed sample $X$, the hidden feature map $h_\theta(X)=X\,\phi(X^\top W_{\mathrm{kq}}X_{-1})$ depends only on $W_{\mathrm{kq}}$ and satisfies
\begin{align*}
&\|h_{\theta'}(X)-h_\theta(X)\|_2
\le
\frac12\,B_X^2\,B_{-1}\,\|W_{\mathrm{kq}}'-W_{\mathrm{kq}}\|_F.\\
&\|h_\theta(X)\|_2\le B_X,
\qquad
\|h_{\theta'}(X)\|_2\le B_X.
\end{align*}
\end{lemma}
\begin{proof}
The lemma follows directly from Lemma~\ref{lem:softmax-jacobian-basic}
\end{proof}

\begin{lemma}[Blockwise Lipschitz continuity of the logit map]\label{lem:U-basic}
Under Assumption~\ref{ass:token_geometry}, for any fixed sample $X$, the softmax logit map $U_\theta(X)=W_\mathrm{ov}X\sigma(X^\top W_\mathrm{kq}X_{-1})$ depends on $\theta=(W_\mathrm{ov},W_\mathrm{kq})$ and satisfies
\begin{align}\label{ineq:U-lipchitz-tight}
\|U_{\theta'}(X)-U_\theta(X)\|_2
&\le
\left(
\sigma_X
+
\frac12B_X^2B_{-1}R(\theta,\theta')\right)\|\theta'-\theta\|_F.
\end{align}
where $R(\theta,\theta')=\min\{\sigma_\mathrm{max}(W_{\mathrm{ov}}),\sigma_\mathrm{max}(W_{\mathrm{ov}}')\}$
\end{lemma}

\begin{proof}
According to the definition of $U_\theta(X)$, we have
\begin{align*}
\|U_{\theta'}(X)-U_\theta(X)\|_2
&\le \|W_{\mathrm{ov}}'(h_{\theta'}(X)-h_\theta(X))\|+\|(W_{\mathrm{ov}}'-W_{\mathrm{ov}})h_\theta(X)\|\\
&\le \frac12B_X^2\,B_{-1}\sigma_\mathrm{max}(W_{\mathrm{ov}}')\|W_{kq}'-W_{kq}\|_F+B_X\,\|W_{\mathrm{ov}}'-W_{\mathrm{ov}}\|_F \numberthis\label{ineq:U-lipchitz}
\end{align*}
Similarly, we can replace $W_\mathrm{ov}'$ with $W_\mathrm{ov}$ in \eqref{ineq:U-lipchitz}. Therefore, we have
\begin{align*}
\|U_{\theta'}(X)-U_\theta(X)\|_2
&\le
\frac12 B_X^2 B_{-1}
\min\{\sigma_{\mathrm{max}}(W_{\mathrm{ov}}),\sigma_{\mathrm{max}}(W_{\mathrm{ov}}')\}
\|W_{\mathrm{kq}}'-W_{\mathrm{kq}}\|_F
+
B_X\|W_{\mathrm{ov}}'-W_{\mathrm{ov}}\|_F \\
&\le
\sqrt{
B_X^2
+
\frac14 B_X^4B_{-1}^2
\min\{\sigma_{\mathrm{max}}(W_{\mathrm{ov}}),\sigma_{\mathrm{max}}(W_{\mathrm{ov}}')\}^2
}
\|\theta'-\theta\|_F \\
&\le
\left(
B_X
+
\frac12 B_X^2B_{-1}R(\theta,\theta')
\right)
\|\theta-\theta'\|_F .
\end{align*}
where $R(\theta,\theta')=\min\{\sigma_\mathrm{max}(W_{\mathrm{ov}}),\sigma_\mathrm{max}(W_{\mathrm{ov}}')\}$.
\end{proof}

\begin{lemma}[Lipschitz continuity of the sequence SFT loss]
\label{lem:sequence-loss-basic}
Under Assumption~\ref{ass:token_geometry}, for any sequence-level sample
\((x,y)\) with token-level representations $\{(X_d,y_d)\}_{d=1}^D$,
recall
\begin{align*}
    \mathcal L_{\mathrm{SFT}}(\theta;x,y)
    = \frac1D\sum_{d=1}^D l(\theta;X_d,y_d).
\end{align*}
Then
\begin{align*}
    |\mathcal L_{\mathrm{SFT}}(\theta';x,y)
    -
    \mathcal L_{\mathrm{SFT}}(\theta;x,y)|
    \le
    \sqrt2
    \left(B_X+\frac12 B_X^2B_{-1}R(\theta,\theta')\right)
    \|\theta'-\theta\|_F,
\end{align*}
where $R(\theta,\theta')= \min\{\sigma_\mathrm{max}(W_\mathrm{ov}), \sigma_\mathrm{max}(W_\mathrm{ov}')\}$.
\end{lemma}
\begin{proof}
Recall that $l(\theta;X,y):=-\log\!\bigl(\mathbf e_{I(y)}^\top\sigma(U_\theta(X))\bigr)=g_y(U_\theta(X))$, since $\|\nabla g_y(u)\|_2\leq\sqrt{2}$, we have
\begin{align*}
    |l(\theta';X_d,y_d)-l(\theta;X_d,y_d)|
    &\le
    \sqrt2\,
    \|U_{\theta'}(X_d)-U_\theta(X_d)\|_2.
\end{align*}
Under Assumption~\ref{ass:token_geometry}, using the hidden-feature Lipschitz bound Lemma~\ref{lem:U-basic},
\begin{align*}
    \|U_{\theta'}(X_d)-U_\theta(X_d)\|_2
    \le
    \left(
    B_X+\frac12B_X^2B_{-1}R(\theta,\theta')
    \right)
    \|\theta'-\theta\|_F.
\end{align*}
Therefore,
\begin{align*}
    |l(\theta';X_d,y_d)-l(\theta;X_d,y_d)|
    \le \sqrt{2}\left(
    B_X+\frac12B_X^2B_{-1}R(\theta,\theta')
    \right)\|\theta'-\theta\|_F.
\end{align*}
Averaging over \(d=1,\ldots,D\), we obtain
\begin{align*}
    |\mathcal L_{\mathrm{SFT}}(\theta';x,y)
    -
    \mathcal L_{\mathrm{SFT}}(\theta;x,y)|
    \le
    \sqrt{2}\left(
    B_X+\frac12B_X^2B_{-1}R(\theta,\theta')
    \right)\|\theta'-\theta\|_F.
\end{align*}
\end{proof}

\begin{lemma}[Local Lipschitz continuity of the sample gradient]
\label{lem:sequence-gradient-lipschitz}
Under Assumption~\ref{ass:token_geometry}, for any sample
\((X,y)\in\mathcal{S}_\mathrm{token}\) with token-level representations, recall its corresponding loss function $l(\theta;X,y) = -\log\left(\mathbf{e}_y^\top \sigma\left(U_\theta(X)\right)\right) = g_y(U_\theta(X))$, then
\begin{align*}
    \|\nabla_\theta l(\theta';X,y)
    -
    \nabla_\theta l(\theta;X,y)\|_F
    \le
    4B_X^2(1+B_{-1})\left(1+B_XB_{-1}R(\theta,\theta')\right)^2
    \|\theta'-\theta\|_F,
\end{align*}
where $R(\theta,\theta'):=\min\{\sigma_\mathrm{max}(W_{\mathrm{ov}}),\sigma_\mathrm{max}(W_{\mathrm{ov}}')\}.$
\end{lemma}
\begin{proof}
By Lemma~\ref{lem:sample-gradient-formula},
\begin{align*}
\nabla_{W_{\mathrm{ov}}}l(\theta;X,y)
&=
\nabla g_y(U_\theta(X))\,h_\theta(X)^\top\\
\nabla_{W_{\mathrm{kq}}}l(\theta;X,y)
&=
X\,J_\phi(B_\theta(X))\,X^\top W_{\mathrm{ov}}^\top \nabla g_y(U_\theta(X))\,X_{-1}^\top
\end{align*}
We estimate the two blocks separately.

\medskip
\noindent
\textbf{Step 1: estimate of the \(W_{\mathrm{ov}}\)-block.}
By Assumption~\ref{ass:token_geometry} and Lemma~\ref{lem:U-basic}, we have
\begin{align*}
&\|\nabla_{W_{\mathrm{ov}}}l(\theta';X,y)-\nabla_{W_{\mathrm{ov}}}l(\theta;X,y)\|_F\\
&\le
\|\nabla g_y(U_{\theta'}(X))-\nabla g_y(U_\theta(X))\|_2\,\|h_{\theta'}(X)\|_2
+
\|\nabla g_y(U_\theta(X))\|_2\,\|h_{\theta'}(X)-h_\theta(X)\|_2\\
&\le
\frac{1}{2}B_X\|U_{\theta'}(X)-U_\theta(X)\|_2 + \frac{\sqrt{2}}{2}B_X^2B_{-1}\|W_{\mathrm{kq}}'-W_{\mathrm{kq}}\|_F\\
&\overset{\eqref{ineq:U-lipchitz-tight}}{\le}
\frac{1}{2}B_X\left(B_X+\frac12B_X^2\,B_{-1}R(\theta,\theta')\right)\|\theta-\theta'\|_F
+
\frac{\sqrt{2}}{2}B_X^2B_{-1}\|W_{\mathrm{kq}}'-W_{\mathrm{kq}}\|_F\\
& \le \left(\frac{1}{4}B_X^3B_{-1}R(\theta,\theta')+\frac{1}{2}B_X^2(1+\sqrt{2}B_{-1})\right)\|\theta-\theta'\|_F\numberthis\label{ineq:gradient_ov}
\end{align*}
\medskip
\noindent
\textbf{Step 2: estimate of the \(W_{\mathrm{kq}}\)-block.}
We denote
\begin{align*}
G_\theta(X):=X^\top W_{\mathrm{ov}}^\top \nabla g_y(U_\theta(X)).
\end{align*}
Then we have
\begin{align*}
    &\|G_{\theta'}(X)-G_\theta(X)\|_2\\
    &\leq \|X^\top W_{\mathrm{ov}}'^\top\left(\nabla g_y(U_{\theta'}(X))-\nabla_yg(U_\theta(X))\right)\|_2+\|X^\top\left(W_{\mathrm{ov}}'-W_{\mathrm{ov}}\right)^\top \nabla g_y(U_\theta(X))\|\\
    &\leq \frac{1}{2}B_XR(\theta,\theta')\|U_{\theta'}(X)-U_{\theta}(X)\|_2+ \sqrt{2}B_X\|W_{\mathrm{ov}}'-W_{\mathrm{ov}}\|_F\\
    &\leq \frac{1}{2}B_XR(\theta,\theta')\left(B_X+\frac12B_X^2\,B_{-1}R(\theta,\theta')\right)\|\theta'-\theta\|+ \sqrt{2}B_X\|W_{\mathrm{ov}}'-W_{\mathrm{ov}}\|_F \numberthis\label{ineq:G-lipchitz}
\end{align*}
and
\begin{align}\label{ineq:G-upper-bound}
    \|G_{\theta}(X)\|_2\leq \sqrt{2}B_X\sigma_\mathrm{max}(W_\mathrm{ov}) 
\end{align}
Therefore, by Lemma~\ref{lem:softmax-jacobian-basic}, inequalities ~\eqref{ineq:G-upper-bound} and ~\eqref{ineq:G-upper-bound}, we have
\begin{align*}
&\|\nabla_{W_{\mathrm{kq}}}l(\theta';X,y)-\nabla_{W_{\mathrm{kq}}}l(\theta;X,y)\|_F\\
&\leq
\|X\bigl(J_\phi(B_{\theta'}(X))-J_\phi(B_\theta(X))\bigr)G_{\theta'}(X)X_{-1}^\top\|
+
\|XJ_\phi(B_\theta(X))\bigl(G_{\theta'}(X)-G_\theta(X)\bigr)X_{-1}^\top\|\\
&\le
B_XB_{-1}
\Big(
\|J_\phi(B_{\theta'}(X))-J_\phi(B_\theta(X))\|_2\,\|G_{\theta'}(X)\|_2
+
\|J_\phi(B_\theta(X))\|_2\,\|G_{\theta'}(X)-G_\theta(X)\|_2
\Big)\\
&\le \frac{3\sqrt{2}}{2}B_X^2B_{-1}R(\theta,\theta')\|B_{\theta'}(X)-B_{\theta}(X)\|_2+\frac{1}{2}B_XB_{-1}\|G_{\theta'}(X)-G_\theta(X)\|_2\\
&\leq 
 \frac{3\sqrt{2}}{2}B_X^3B_{-1}^2R(\theta,\theta')\|W_{\mathrm{kq}}'-W_{\mathrm{kq}}\|_F
+ \frac{\sqrt{2}}{2}B_X^2B_{-1}\|W_{\mathrm{ov}}'-W_{\mathrm{ov}}\|_F \\
&\quad+
\frac{1}{4}B_X^2B_{-1}R(\theta,\theta')\left(B_X+\frac12B_X^2\,B_{-1}R(\theta,\theta')\right)\|\theta'-\theta\|_F\\
& \leq \left(\frac{1}{8}B_X^4B_{-1}^2R(\theta,\theta')^2+B_X^3B_{-1}\left(\frac{3\sqrt{2}}{2}B_{-1}+\frac{1}{4}\right)R(\theta,\theta')+\frac{\sqrt{2}}{2}B_X^2B_{-1}\right)\|\theta'-\theta\|_F\numberthis\label{ineq:gradient_kq}
\end{align*}

\medskip
\noindent
\textbf{Step 3: combine the two blocks.}
Combining \eqref{ineq:gradient_ov} and \eqref{ineq:gradient_kq}, we have
\begin{align*}
&\|\nabla_\theta l(\theta';X,y)-\nabla_\theta l(\theta;X,y)\|_F\\
&=
\left(\|\nabla_{W_{\mathrm{ov}}}l(\theta';X,y)-\nabla_{W_{\mathrm{ov}}}l(\theta;X,y)\|_F^2+
\|\nabla_{W_{\mathrm{kq}}}l(\theta';X,y)-\nabla_{W_{\mathrm{kq}}}l(\theta;X,y)\|_F^2\right)^{\frac{1}{2}}\\
&\le
\|\nabla_{W_{\mathrm{ov}}}l(\theta';X,y)-\nabla_{W_{\mathrm{ov}}}l(\theta;X,y)\|_F+
\|\nabla_{W_{\mathrm{kq}}}l(\theta';X,y)-\nabla_{W_{\mathrm{kq}}}l(\theta;X,y)\|_F\\
&\leq \left(\frac{1}{8}B_X^4B_{-1}^2R(\theta,\theta')^2+B_X^3B_{-1}\left(\frac{3\sqrt{2}}{2}B_{-1}+\frac{1}{2}\right)R(\theta,\theta')+B_X^2\left(\sqrt{2}B_{-1}+1\right)\right)\|\theta'-\theta\|_F\\
&\leq 4B_X^2(1+B_{-1})\left(1+B_XB_{-1}R(\theta,\theta')\right)^2\|\theta'-\theta\|_F
\end{align*}
where $R(\theta,\theta'):=\min\{\sigma_\mathrm{max}(W_{\mathrm{ov}}),\sigma_\mathrm{max}(W_{\mathrm{ov}}')\}$.
\end{proof}

\begin{lemma}[Gradient norm bound for the sample loss]
\label{lem:sample-gradient-bound-by-loss}
Under Assumption~\ref{ass:token_geometry}, for any sample
\((X,y)\in\mathcal{S}_{\mathrm{token}}\) with token-level representation, then
\begin{align*}
    \|\nabla_\theta l(\theta;X,y)\|_F
    \le
    \sqrt{2}
    B_X+\frac{\sqrt{2}}2 B_X^2B_{-1}\sigma_{\max}(W_{\mathrm{ov}})
\end{align*}
\end{lemma}
\begin{proof}
Fix one sample \((X,y)\), and write
\[
p:=\sigma(U_\theta(X)),
\qquad
l(\theta;X,y)=-\log p_{y}.
\]
By Lemma~\ref{lem:sample-gradient-formula}, we have
\begin{align*}
\nabla_{W_{\mathrm{ov}}}l(\theta;X,y)
=
\bigl(p-\mathbf e_{y}\bigr)h_\theta(X)^\top,\quad
\nabla_{W_{\mathrm{kq}}}l(\theta;X,y)
=
XJ_\sigma(X^\top W_{\mathrm{kq}}X_{-1})X^\top W_{\mathrm{ov}}^\top
\bigl(p-\mathbf e_{y}\bigr)X_{-1}^\top.
\end{align*}
Under Assumption~\ref{ass:token_geometry}, by Lemma~\ref{lem:softmax-jacobian-basic} and Lemma~\ref{lem:h-basic}, we have
\begin{align*}
\|\nabla_{W_{\mathrm{ov}}}l(\theta;X,y)\|_F
&\le
B_X\|p-\mathbf e_{y}\|_2\leq\sqrt{2}B_X,\\
\|\nabla_{W_{\mathrm{kq}}}l(\theta;X,y)\|_F
&\le
\frac12B_X^2B_{-1}\sigma_\mathrm{max}(W_\mathrm{ov})
\|p-\mathbf e_{y}\|_2\leq \frac{\sqrt{2}}{2}B_X^2B_{-1}\sigma_\mathrm{max}(W_\mathrm{ov}).
\end{align*}
Therefore,
\begin{align*}
    \|\nabla_\theta l(\theta;X,y)\|_F\leq \sqrt{2}B_X+\frac{\sqrt{2}}{2}B_X^2B_{-1}\sigma_\mathrm{max}(W_\mathrm{ov})
\end{align*}
\end{proof}

\begin{lemma}[Gradient norm bounds for the penalized SFT loss]
\label{lem:Lgamma-gradient-bound-by-Lgamma}
Under Assumption~\ref{ass:token_geometry}, we have
\begin{align*}
    \|\nabla_{W_{\mathrm{ov}}}\mathcal L_\gamma(\omega,\theta)\|_F
    \le
    \sqrt{2}B_X \mathcal L_\gamma(\omega,\theta),\quad \|\nabla_{W_{\mathrm{kq}}}\mathcal L_\gamma(\omega,\theta)\|_F\le
    \frac{\sqrt{2}}{2}
    B_X^2B_{-1}\sigma_{\max}(W_{\mathrm{ov}})
    \mathcal L_\gamma(\omega,\theta).
\end{align*}
Consequently,
\begin{align*}
    \|\nabla_{\theta}\mathcal L_\gamma(\omega,\theta)\|_F
    \le
    \sqrt{2}\left(
    B_X+\frac12B_X^2B_{-1}\sigma_{\max}(W_{\mathrm{ov}})
    \right)
    \mathcal L_\gamma(\omega,\theta).
\end{align*}
\end{lemma}

\begin{proof}
Fix one sample \((X,y)\), and write
\[
p:=\sigma(U_\theta(X)),
\qquad
l(\theta;X,y)=-\log p_{y}.
\]
By Lemma~\ref{lem:sample-gradient-formula}, we have
\begin{align*}
\nabla_{W_{\mathrm{ov}}}l(\theta;X,y)
=
\bigl(p-\mathbf e_{y}\bigr)h_\theta(X)^\top,\quad
\nabla_{W_{\mathrm{kq}}}l(\theta;X,y)
=
XJ_\sigma(X^\top W_{\mathrm{kq}}X_{-1})X^\top W_{\mathrm{ov}}^\top
\bigl(p-\mathbf e_{y}\bigr)X_{-1}^\top.
\end{align*}
Under Assumption~\ref{ass:token_geometry}, by Lemma~\ref{lem:softmax-jacobian-basic} and Lemma~\ref{lem:h-basic}, we have
\begin{align*}
\|\nabla_{W_{\mathrm{ov}}}l(\theta;X,y)\|_F
\le
B_X\|p-\mathbf e_{y}\|_2,\quad
\|\nabla_{W_{\mathrm{kq}}}l(\theta;X,y)\|_F
\le
\frac12B_X^2B_{-1}\sigma_\mathrm{max}(W_\mathrm{ov})
\|p-\mathbf e_{y}\|_2.
\end{align*}
Moreover,
\begin{align*}
\|p-\mathbf e_{I(y)}\|_2
&=
\left(
\sum_{\ell\neq y}p_\ell^2
+
(1-p_{y})^2
\right)^{1/2}\le
\sqrt{2}(1-p_{y})\overset{(a)}{\le}
\sqrt{2}l(\theta;X,y),
\end{align*}
where we used \(1-t\le -\log t\) for \(t\in(0,1]\) in (a). Therefore,
\begin{align*}
    \|\nabla_\theta l(\theta;X,y)\|_F\leq \sqrt{2}\left(B_X+\frac{1}{2}B_X^2B_{-1}\sigma_\mathrm{max}(W_\mathrm{ov})\right) l(\theta;X,y)
\end{align*}
By the triangle inequality,
\begin{align*}
\|\nabla_{\theta}\mathcal L_\gamma(\omega,\theta)\|_F
&\le
\frac1{N'D}\sum_{j=1}^{N'}\sum_{d=1}^D
\|\nabla_{\theta}l(\theta;\widetilde X^j_d,\widetilde y^j_d)\|_F
+
\frac{\gamma}{ND}\sum_{i=1}^N\sum_{d=1}^D[\sigma(\omega)]_i
\|\nabla_{\theta}l(\theta;X^i_d,y^i_d)\|_F\\
&\le
\sqrt{2}\left(B_X+\frac{1}{2}B_X^2B_{-1}\sigma_\mathrm{max}(W_\mathrm{ov})\right)\mathcal L_\gamma(\omega,\theta).
\end{align*}
\end{proof}

\subsubsection{Proof of non-uniform smoothness condition}
In this section, for any given point $z=(\omega,\theta)$ and stepsize $\eta$, we define 
\begin{align*}
    z(0)&=(\omega(0),\theta(0))=(\omega,\theta)\\
    z(s)&=(\omega(s),\theta(s))=(\omega-s\eta\nabla_\omega L_\gamma(\omega,\theta),  \theta-s\eta\nabla_\theta L_\gamma(\omega,\theta)).
\end{align*}
for any $s\in [0,1]$.
\begin{lemma}[Local smoothness of $\mathcal L_\gamma$]\label{lem:local-smoothness-compact}
Under Assumption~\ref{ass:token_geometry}, for any $s\in[0,1]$, we have
\begin{align*}
\left\|
\nabla_\theta \mathcal L_\gamma(\omega(s),\theta(s))
-
\nabla_\theta \mathcal L_\gamma(\omega^k,\theta^k)
\right\|_F\leq L(\theta(0))\|\theta(s)-\theta(0)\|_F
\end{align*}
where
\begin{align*}
L(\theta(0))
&:=
4\left(1+\frac{\gamma}{N}\right)\,B_X^2(1+B_{-1})
\left(1+B_XB_{-1}\sigma_{\max}(W_{\mathrm{ov}}(0))\right)^2,\\
&\quad +
\frac{3\gamma}{\sqrt{N}}\left(
\log V
+
B_X
+
B_X\left(1+B_XB_{-1}\right)\sigma_{\max}(W_{\mathrm{ov}}(0))
\right).\numberthis\label{def:L}
\end{align*}
\end{lemma}
\begin{proof}
Denote $R(\theta(s),\theta(0))=\min\{\sigma_\mathrm{max}(W_{\mathrm{ov}}(s)), \sigma_\mathrm{max}(W_{\mathrm{ov}}(0))\}$ for any $s\in[0,1]$. Recall that
\begin{align*}
\mathcal L_\gamma(\omega,\theta)
=
\frac{1}{N'}\sum_{j=1}^{N'}\mathcal L_{\mathrm{SFT}}(\theta;\widetilde X^{\,j},\widetilde y^{\,j})
+
\frac{\gamma}{N}\sum_{i=1}^{N}[\phi(\omega)]_i\,\mathcal L_{\mathrm{SFT}}(\theta;X^i,y^i).
\end{align*}

\medskip
\noindent
\textbf{Step 1: Variation of the $\theta$-gradient.}
We have
\begin{align*}
&\|\nabla_\theta \mathcal L_\gamma(\omega(s),\theta(s))
-
\nabla_\theta \mathcal L_\gamma(\omega(0),\theta(0))\|_F\\
&\leq
\frac1{N'D}\sum_{j=1}^{N'}\sum_{d=1}^D
\|
\nabla_\theta l(\theta(s);\widetilde X^j_d,\widetilde y^j_d)
-
\nabla_\theta l(\theta(0);\widetilde X^j_d,\widetilde y^j_d)
\|_F\\
&\quad+\frac{\gamma}{ND}\sum_{i=1}^N\sum_{d=1}^D [\sigma(\omega(s))]_i
\|
\nabla_\theta l(\theta(s);X^i_d,y^i_d)
-
\nabla_\theta l(\theta(0);X^i_d,y^i_d)
\|_F\\
&\quad+
\frac{\gamma}{ND}\sum_{i=1}^N\sum_{d=1}^D
\left|\sigma(\omega(s))_i-\sigma(\omega(0))_i\right|
\|\nabla_\theta l(\theta(0);X_i,y_i)\|_F\\
&\overset{(a)}{\leq}4B_X^2(1+B_{-1})\left(1+B_XB_{-1}R(\theta,\theta')\right)^2\left(1+\frac{\gamma}{N}\right)\|\theta(s)-\theta(0)\|_F\\
&\quad +\frac{\gamma}{ND}\|\sigma(\omega(s))-\sigma(\omega(0))\|_2\left(\sum_{i=1}^N\left(\sum_{d=1}^D\|\nabla_\theta l(\theta(0);X^i_d,y^i_d)\|_F\right)^2\right)^\frac{1}{2}\\
&\overset{(b)}{\leq} 4B_X^2(1+B_{-1})\left(1+B_XB_{-1}R(\theta,\theta')\right)^2\left(1+\frac{\gamma}{N}\right)\|\theta(s)-\theta(0)\|_F\\
&\quad +\sqrt{\frac{2}{N}}\gamma\left(B_X+\frac{1}{2}B_X^2B_{-1}\sigma_\mathrm{max}(W_\mathrm{ov})\right)\|\sigma(\omega(s))-\sigma(\omega(0))\|_2 \numberthis\label{ineq:theta-gradient-lipchitz}
\end{align*}
where in (a) we use Lemma~\ref{lem:sequence-gradient-lipschitz} and Cauchy Schwart inequality, in (b) we use Lemma~\ref{lem:sample-gradient-bound-by-loss}.

\medskip
\noindent
\textbf{Step 2: Variation of the $\omega$-gradient.}

By Lemma~\ref{lem:penalty-gradient-formula}, we have
\[
\nabla_\omega \mathcal L_\gamma(\omega,\theta)
=
\frac{\gamma}{N} J_\sigma(\omega)\ell_\mathrm{seq}(\theta),
\qquad
J_\sigma(\omega)
:=
\operatorname{Diag}(\sigma(\omega))
-
\sigma(\omega)\sigma(\omega)^\top .
\]
Therefore,
\begin{align*}
&\left\|
\nabla_\omega \mathcal L_\gamma(\omega(s),\theta(s))
-
\nabla_\omega \mathcal L_\gamma(\omega(0),\theta(0))
\right\|_2\\
&\le
\frac{\gamma}{N}
\left\|
J_\sigma(\omega(s))-J_\sigma(\omega(0))
\right\|_2
\|\ell_\mathrm{seq}(\theta(0))\|_2
+
\frac{\gamma}{N}
\|J_\sigma(\omega(s))\|_2
\|\ell_\mathrm{seq}(\theta(s))-\ell_\mathrm{seq}(\theta(0))\|_2\\
&\overset{(a)}{\leq} \frac{3\gamma}{2N}\|\ell_\mathrm{seq}(\theta(0))\|_2\|\omega(s)-\omega(0)\|_2+\frac{\gamma}{2N}\|\ell_\mathrm{seq}(\theta(s))-\ell_\mathrm{seq}(\theta(0))\|_2
\end{align*}
where in (a) we use ~\ref{lem:softmax-jacobian-basic}. Next, by Lemma~\ref{lem:sequence-loss-basic},
\begin{align*}
\|\ell_\mathrm{seq}(\theta(s))-\ell_\mathrm{seq}(\theta(0))\|_2
&=
\left(
\sum_{i=1}^N
\left|
\mathcal L_{\mathrm{SFT}}(\theta(s);x_i,y_i)
-
\mathcal L_{\mathrm{SFT}}(\theta(0);x_i,y_i)
\right|^2
\right)^{1/2}\\
&\le
\sqrt{2N}\left(B_X+\frac{1}{2}B_X^2B_{-1}R(\theta(s),\theta(0))\right)\|\theta(s)-\theta(0)\|_F .
\end{align*}
To bound $\|\ell(\theta(0))\|_2$, we use the uniform upper bound on each
sample loss. Indeed, for every $i\in[N]$,
\begin{align*}
\mathcal L_{\mathrm{SFT}}(\theta(0);x_i,y_i)
&\le
\mathcal L_{\mathrm{SFT}}(0;x_i,y_i)
+
\left|
\mathcal L_{\mathrm{SFT}}(\theta(0);x_i,y_i)
-
\mathcal L_{\mathrm{SFT}}(0;x_i,y_i)
\right|\\
&\overset{(a)}{\leq}
\log V+\sqrt{2}B_X\sigma_\mathrm{max}(W_{\mathrm{ov}}(0))\\
\end{align*}
where (a) also follows Lemma~\ref{lem:sequence-loss-basic}.
Therefore, we have
\begin{align*}
\|\ell_\mathrm{seq}(\theta(0))\|_2
&=
\left(
\sum_{i=1}^N
\mathcal L_{\mathrm{SFT}}(\theta(0);x_i,y_i)^2
\right)^{1/2}\\
&\le
\sqrt{N}
\left(
\log V+\sqrt{2}B_X\sigma_\mathrm{max}(W_{\mathrm{ov}}(0))
\right).
\end{align*}
Combining the above estimates, we obtain
\begin{align*}
&\left\|
\nabla_\omega \mathcal L_\gamma(\omega(s),\theta(s))
-
\nabla_\omega \mathcal L_\gamma(\omega(0),\theta(0))
\right\|_2\\
&\le
\frac{3\gamma}{2\sqrt{N}}
\left(
\log V+\sqrt{2}B_X\sigma_\mathrm{max}(W_{\mathrm{ov}}(0))
\right)
\|\omega(s)-\omega(0)\|_2\\
&
~~~~+
\frac{\gamma}{2}\sqrt{\frac{2}{N}}\left(B_X+\frac{1}{2}B_X^2B_{-1}R(\theta(s),\theta(0))\right)
\|\theta(s)-\theta(0)\|_F.\numberthis\label{ineq:omega-gradient-lipchitz}
\end{align*}

\medskip
\noindent
\textbf{Step 3: combine the two blocks.}
Combining \eqref{ineq:theta-gradient-lipchitz} and \eqref{ineq:omega-gradient-lipchitz}, we have
\begin{align*}
    \|\nabla \mathcal L_\gamma(\omega(s),\theta(s))-\nabla \mathcal L_\gamma(\omega(0),\theta(0))\|_F\leq L\|z(s)-z(0)\|_F.
\end{align*}
\begin{align*}
L
&=
4B_X^2(1+B_{-1})
\left(1+B_XB_{-1}R(\theta(s),\theta(0))\right)^2
\left(1+\frac{\gamma}{N}\right)\\
&\quad+
\frac{\gamma}{\sqrt{2N}}
\left(
B_X+\frac{1}{2}B_X^2B_{-1}
\sigma_{\max}(W_{\mathrm{ov}}(0))
\right)\\
&\quad+
\frac{3\gamma}{2\sqrt{N}}
\left(
\log V+\sqrt{2}B_X\sigma_{\max}(W_{\mathrm{ov}}(0))
\right)\\
&\quad+
\frac{\gamma}{2}\sqrt{\frac{2}{N}}
\left(
B_X+\frac{1}{2}B_X^2B_{-1}R(\theta(s),\theta(0))
\right).
\end{align*}
Since $R(\theta(s),\theta(0)) \leq \sigma_\mathrm{max}(W_{\mathrm{ov}}(0))$, we have
\begin{align*}
L
&\le
4B_X^2(1+B_{-1})
\left(1+B_XB_{-1}\sigma_\mathrm{max}(W_{\mathrm{ov}}(0))\right)^2
\left(1+\frac{\gamma}{N}\right)\\
&\quad+
\frac{3\gamma}{\sqrt{N}}
\left(
\log V
+
B_X
+
B_X\left(1+B_XB_{-1}\right)\sigma_\mathrm{max}(W_{\mathrm{ov}}(0))
\right)=L(\theta(0)).
\end{align*}
Therefore,
\begin{align*}
    \|\nabla \mathcal L_\gamma(\omega(s),\theta(s))-\nabla \mathcal L_\gamma(\omega(0),\theta(0))\|_F\leq L(\theta(0))\|z(s)-z(0)\|_F.
\end{align*}
where $L(\theta(0))$ is defined in \eqref{def:L}.
\end{proof}

\begin{lemma}[One-step Taylor upper bound from a one-sided gradient variation bound]
\label{lem:one-step-smoothness}
For each iteration $k$, let $z_k(0)=(\omega^k,\theta^k)$ and 
\[
z_k(s)=(\omega^k-s\eta_k\nabla_\omega \mathcal L_\gamma(\omega^k,\theta^k),\theta^k-s\eta_k\nabla_\theta \mathcal L_\gamma(\omega^k,\theta^k))
\]
Then
\[
\mathcal L_\gamma(z_k(1))
\le
\mathcal L_\gamma(z_k(0))
+
\langle \nabla \mathcal L_\gamma(z_k(0)),\, z_k(1)-z_k(0)\rangle
+
\frac{L_k}{2}\|z_k(1)-z_k(0)\|_F^2.
\]
Equivalently,
\[
\mathcal L_\gamma(\omega^{k+1},\theta^{k+1})
\le
\mathcal L_\gamma(\omega^{k},\theta^{k})
+
\langle \nabla \mathcal L_\gamma(\omega^{k},\theta^{k}),\, z^{k+1}-z^k\rangle
+
\frac{L_k}{2}\|z^{k+1}-z^k\|_F^2.
\]
where 
\begin{align*}
L_k&=4B_X^2(1+B_{-1})
\left(1+B_XB_{-1}\sigma_\mathrm{max}(W_{\mathrm{ov}}^k)\right)^2
\left(1+\frac{\gamma}{N}\right)\\
&\quad+
\frac{3\gamma}{\sqrt{N}}
\left(
\log V
+
B_X
+
B_X\left(1+B_XB_{-1}\right)\sigma_\mathrm{max}(W_{\mathrm{ov}}^k)
\right)\numberthis\label{def:L_k}
\end{align*}
\end{lemma}

\begin{proof}
By the fundamental theorem of calculus,
\[
\mathcal L_\gamma(z_k(1))-\mathcal L_\gamma(z_k(0))
=
\int_0^1 \frac{d}{ds}\mathcal L_\gamma(z_k(s))\,ds
=
\int_0^1 \langle \nabla \mathcal L_\gamma(z_k(s)),\, z'_k(s)\rangle ds.
\]
Since $z'(s)=z_k(1)-z_k(0):=\Delta z_k$ and local smoothness condition in Lemma~\ref{lem:local-smoothness-compact}, we have
\begin{align*}
\mathcal L_\gamma(z_k(1))-\mathcal L_\gamma(z_k(0))
&=
\int_0^1 \langle \nabla \mathcal L_\gamma(z_k(0)),\, \Delta z_k\rangle ds+
\int_0^1
\langle \nabla \mathcal L_\gamma(z_k(s))-\nabla \mathcal L_\gamma(z_k(0)),\, \Delta z_k\rangle ds\\
&=
\langle \nabla \mathcal L_\gamma(z_k(0)),\, \Delta z_k\rangle
+
\int_0^1
\langle \nabla \mathcal L_\gamma(z_k(s))-\nabla \mathcal L_\gamma(z_k(0)),\, \Delta z_k\rangle ds\\
&\leq \langle \nabla \mathcal L_\gamma(z_k(0)),\, \Delta z_k\rangle
+
\int_0^1\| \nabla \mathcal L_\gamma(z_k(s))-\nabla \mathcal L_\gamma(z_k(0))\|_F \|\Delta z_k\|_F ds\\
&\leq  \langle \nabla \mathcal L_\gamma(z_k(0)),\, \Delta z_k\rangle
+ L_k \int_0^1 s\|\Delta z_k\|_F^2 ds\\
&\leq  \langle \nabla \mathcal L_\gamma(z_k(0)),\, \Delta z_k\rangle
+
\frac{L_k}{2}\|\Delta z_k\|_F^2
\end{align*}
Since $z_k(0)=z^k$ and $z_k(1)=z^{k+1}$, equivalently we have
\begin{align*}
    \mathcal L_\gamma(\omega^{k+1},\theta^{k+1})
\le
\mathcal L_\gamma(\omega^{k},\theta^{k})
+
\langle \nabla \mathcal L_\gamma(\omega^{k},\theta^{k}),\, z^{k+1}-z^k\rangle
+
\frac{L_k}{2}\|z^{k+1}-z^k\|_F^2.
\end{align*}
where $L_k$ is defined in \eqref{def:L_k}.
\end{proof}

\subsubsection{Proof of non-uniform Polyak-Łojasiewicz (PL) condition}
\begin{lemma}[Non-uniform PL condition of $g_y(u)$]
\label{lem:nonuniform_pl_single_token}
Recall that $g_y(u)=-\log \mathbf{e}_y^\top\sigma(u)$ for any $u\in \mathbb{R}^V$, then
\[
\|\nabla g_y(u)\|_2^2
\ge
2\mu_g(u) g_y(u),
\]
where
\begin{align*}
    \mu_g(u):=
\frac{V-1}{2V^2}\exp(-\sqrt{2}\|u\|_2).
\end{align*}
\end{lemma}

\begin{proof}
Let $p=\sigma(u)$ and denote $p_y=[\sigma(u)]_y$.
By direct calculation,
\[
g_y(u)=-\log p_y,
\qquad
\nabla g_y(u)=p-\mathbf e_y.
\]
Therefore,
\begin{align}\label{ineq:g-gradient-lower-bound}
\|\nabla g_y(u)\|_2^2
=
(1-p_y)^2+\sum_{v\ne y}p_v^2
\overset{(a)}{\ge}
\frac{V}{V-1}(1-p_y)^2,
\end{align}
where in (a) we use $\sum_{v\ne y}p_v=1-p_y$ and the Cauchy--Schwarz  inequality.

Next, for any $v\ne y$,
\[
|u_v-u_y|
=
|(\mathbf e_v-\mathbf e_y)^\top u|
\le
\sqrt2\|u\|_2
:=a.
\]
Therefore, we have
\[
p_y
=
\frac{1}{1+\sum_{v\ne y}\exp(u_v-u_y)}
\in[p_{\min},p_{\max}],
\]
where
\[
p_{\min}:=
\frac{1}{1+(V-1)\exp(a)},
\qquad
p_{\max}:=
\frac{1}{1+(V-1)\exp(-a)}.
\]

Define $f(p):=\frac{-\log p}{(1-p)^2}$.
A direct calculation shows that $f$ has one interior minimum on $(0,1)$.
Hence its maximum over $[p_{\min},p_{\max}]$ is attained at one of the
endpoints:
\[
\max_{p\in[p_{\min},p_{\max}]}f(p)
=
\max\{f(p_{\min}),f(p_{\max})\}.
\]
Moreover, the two endpoint values satisfy
\begin{align*}
f(p_{\min})
&= \left(\frac{1}{V-1}\exp(-a)+1\right)^2\log\left(1+(V-1)\exp(a)\right)\le
\frac{V^2}{(V-1)^2}(\log V+a),\\
f(p_{\max})
&=\left(\frac{1}{V-1}\exp(a)+1\right)^2\log\left(1+(V-1)\exp(-a)\right) \le \frac{V^2}{V-1}\exp(a) 
\end{align*}
From direct circulation, for any $a\geq 0$ and $V\geq 2$, $\frac{V^2}{V-1}\exp(a)\geq \frac{V^2}{(V-1)^2}(\log V+a)$.
Therefore, substituting $a = \sqrt{2}\|u\|_2$, 
\[
\frac{-\log p_y}{(1-p_y)^2}
\le
\frac{V^2}{V-1}\exp(a) = \frac{V^2}{V-1}\exp(\sqrt{2}\|u\|_2).
\]
Combining this with \eqref{ineq:g-gradient-lower-bound} gives
\[
\|\nabla g_y(u)\|_2^2
\ge
\frac{1}{V}\exp(-\sqrt{2}\|u\|_2)
g_y(u).
\]
This completes the proof.
\end{proof}

\begin{lemma}[Non-uniform PL condition]
\label{lem:nonuniform_pl_offline}
Under Assumption~\ref{ass:token_geometry}, the offline penalized objective
$\mathcal L_\gamma(\omega,\theta)$ satisfies a non-uniform PL condition.
Specifically, if $\sigma_{\min}(H(\omega,\theta))>0$, then
\[
\|\nabla \mathcal L_\gamma(\omega,\theta)\|_F^2
\ge
2\mu_{\rm PL}(\omega,\theta)\mathcal L_\gamma(\omega,\theta),
\]
where
\begin{align}\label{def:Lgamma-PL}
\mu_{\rm PL}(\omega,\theta)
:=
\sigma_{\min}^2(H(\omega,\theta))\,\underline{\mu}_g(\theta),
\quad\underline{\mu}_g(\theta)
:=
\frac{1}{2V}\exp(-\sqrt{2}B_X\sigma_\mathrm{max}(W_{\mathrm{ov}})).
\end{align}

\end{lemma}

\begin{proof}
By Lemma~\ref{lem:nonuniform_pl_single_token}, for any single-token loss
\[
g_y(u):=-\mathbf e_y^\top\log\sigma(u)
=
-\log[\sigma(u)]_y,
\]
we have
\[
\|\nabla g_y(u)\|_2^2
\ge
2\mu_g(u)g_y(u),\qquad
\mu_g(u)
=
\frac{1}{2V}\exp(-\sqrt{2}\|u\|_2)
\]

For each training token and validation token, by Assumption~\ref{ass:token_geometry}, we have 
\[
\|U_\theta(X_d^i)\|_2
\le
B_X\sigma_{\max}(W_{\mathrm{ov}}),\quad
\|U_\theta(\widetilde X_d^j)\|_2
\le
B_X\sigma_{\max}(W_{\mathrm{ov}}).
\]
Since \(\mu_g(u)\) is non-increasing in \(\|u\|_2\), we have
\[
\mu_g(U_\theta(X_d^i))\ge \underline{\mu}_g(\theta),
\qquad
\mu_g(U_\theta(\widetilde X_d^j))\ge \underline{\mu}_g(\theta).
\]
where $\mu_g(\theta)=\frac{1}{2V}\exp(-\sqrt{2}B_X\sigma_\mathrm{max}(W_\mathrm{ov}))$. Therefore,
\begin{align}\label{ineq:g-PL}
\|\nabla g_{y_d^i}(U_\theta(X_d^i))\|_2^2
\ge
2\underline{\mu}_g(\theta)
g_{y_d^i}(U_\theta(X_d^i)),\quad
\|\nabla g_{\widetilde y_d^j}(U_\theta(\widetilde X_d^j))\|_2^2
\ge
2\underline{\mu}_g(\theta)
g_{\widetilde y_d^j}(U_\theta(\widetilde X_d^j)).
\end{align}

Recall the weighted token-level feature matrix
\[
H(\omega,\theta)
:=
\left[
\left\{
\frac{1}{\sqrt{N'D}}
h_\theta(\widetilde X_d^j)
\right\}_{j\in[N'],\,d\in[D]},
\left\{
\sqrt{\frac{\gamma\sigma_i(\omega)}{ND}}
h_\theta(X_d^i)
\right\}_{i\in[N],\,d\in[D]}
\right],
\]
and the weighted logit-gradient matrix
\[
G(\omega,\theta)
:=
\left[
\left\{
\frac{1}{\sqrt{N'D}}
\nabla g_{\widetilde y_d^j}(U_\theta(\widetilde X_d^j))
\right\}_{j\in[N'],\,d\in[D]},
\left\{
\sqrt{\frac{\gamma\sigma_i(\omega)}{ND}}
\nabla g_{y_d^i}(U_\theta(X_d^i))
\right\}_{i\in[N],\,d\in[D]}
\right].
\]
Then
\begin{align*}
\|G(\omega,\theta)\|_F^2
&=
\frac{1}{N'D}
\sum_{j=1}^{N'}\sum_{d=1}^D
\|\nabla g_{\widetilde y_d^j}(U_\theta(\widetilde X_d^j))\|_2^2 +
\frac{\gamma}{ND}
\sum_{i=1}^{N}\sum_{d=1}^D
\sigma_i(\omega)
\|\nabla g_{y_d^i}(U_\theta(X_d^i))\|_2^2 \\
&\overset{\eqref{ineq:g-PL}}{\ge}
2\underline{\mu}_g(\theta)
\left[
\frac{1}{N'D}
\sum_{j=1}^{N'}\sum_{d=1}^D
g_{\widetilde y_d^j}(U_\theta(\widetilde X_d^j))
+
\frac{\gamma}{ND}
\sum_{i=1}^{N}\sum_{d=1}^D
\sigma_i(\omega)
g_{y_d^i}(U_\theta(X_d^i))
\right] \\
&=
2\underline{\mu}_g(\theta)
\mathcal L_\gamma(\omega,\theta).
\end{align*}

By the definitions of \(G(\omega,\theta)\) and \(H(\omega,\theta)\), the
gradient with respect to \(W_{\mathrm{ov}}\) admits the factorization
\[
\nabla_{W_{\mathrm{ov}}}\mathcal L_\gamma(\omega,\theta)
=
G(\omega,\theta)H(\omega,\theta)^\top.
\]
Therefore,
\begin{align*}
\|\nabla_{W_{\mathrm{ov}}}\mathcal L_\gamma(\omega,\theta)\|_F^2
&=
\|G(\omega,\theta)H(\omega,\theta)^\top\|_F^2 \\
&\ge
\sigma_{\min}^2(H(\omega,\theta))
\|G(\omega,\theta)\|_F^2 \\
&\ge
2\sigma_{\min}^2(H(\omega,\theta))
\underline{\mu}_g(\theta)
\mathcal L_\gamma(\omega,\theta).
\end{align*}
Since \(W_{\mathrm{ov}}\) is one block of \(\theta\), we have
\[
\|\nabla \mathcal L_\gamma(\omega,\theta)\|_F^2
\ge
\|\nabla_{W_{\mathrm{ov}}}\mathcal L_\gamma(\omega,\theta)\|_F^2
\ge
2\mu_{\rm PL}(\omega,\theta)
\mathcal L_\gamma(\omega,\theta),
\]
where $\mu_{\rm PL}(\omega,\theta)$ is defined in \eqref{def:Lgamma-PL}.
\end{proof}

\subsubsection{Trajectory Control for the Local PL Region}
For notational simplicity, along the offline PBGD trajectory, we write $H^k:=H(\omega^k,\theta^k)$.
\begin{lemma}[Trajectory and representation drift]
\label{lem:trajectory-representation-drift}
Assume that $V\ge (N+N')D$ and \(H^0\) has full column rank. Equivalently, there exists a constant \(\sigma_0>0\) such that
\begin{align}\label{ass:sigma-H}
    \sigma_{\min}(H^0)
    =
    \left(1+\frac{\gamma}{N}\right)^{1/2}\sigma_0
    >0.
\end{align}
Let $S_k:=\sum_{t=0}^{k-1}\eta_t
\mathcal L_\gamma(\omega^t,\theta^t)$, where $\eta_t>0$ denotes the stepsize at iteration $t$. Along the offline PBGD trajectory,
\[
    \sigma_\mathrm{max}(W_{\mathrm{ov}}^k-W_{\mathrm{ov}}^0)\leq\sqrt{2}B_XS_k,
    \quad
    \|\omega^k-\omega^0\|_2\leq \sqrt{2}S_k,
    \quad
    \sigma_\mathrm{max}(W_{\mathrm{kq}}^k-W_{\mathrm{kq}}^0)
    \le
    \frac12 B_X^3B_{-1}S_k^2.
\]
Moreover, suppose the initial representation matrix satisfies $\sigma_{\min}(H^0) =
\left(1+\frac{\gamma}{N}\right)^{1/2}\sigma_0$.
Then
\[
    \sigma_{\min}(H^k)
    \ge
    \left(1+\frac{\gamma}{N}\right)^\frac{1}{2}\left(\sigma_0-\frac14B_X^5B_{-1}^2S_k^2- \left(\frac{\sqrt{2}}{4}B_X^4B_{-1}^2\sigma_\mathrm{max}(W_{\mathrm{ov}}^0)+\frac12B_X\right)S_k \right)
\]
\end{lemma}

\begin{proof}
\medskip
\noindent
\textbf{Step 1: Control of the \(W_{\mathrm{ov}},\omega\)-trajectory.}
Denote $S_k = \sum_{t=0}^{k-1}\eta_t \mathcal L_\gamma(\omega^t,\theta^t)$.
By Lemma~\ref{lem:Lgamma-gradient-bound-by-Lgamma}, for the trajectory of $W_{\mathrm{ov}}$, we have
\begin{align*}
    \sigma_\mathrm{max}\left(\nabla_{W_{\mathrm{ov}}} \mathcal L_\gamma(\omega,\theta)\right) \le \|\nabla_{W_{\mathrm{ov}}} \mathcal L_\gamma(\omega,\theta)\|_F \le \sqrt{2}B_X \mathcal L_\gamma(\omega,\theta)
\end{align*}
We have one-step iteration inequality
\begin{align*}
    \sigma_\mathrm{max}(W_{\mathrm{ov}}^{k+1}-W_{\mathrm{ov}}^k)
    &= \eta_k\sigma_\mathrm{max}(\nabla_{W_{\mathrm{ov}}}\mathcal L_\gamma(\omega^k,\theta^k))\\
    & \leq \sqrt{2}B_X\eta_k \mathcal L_\gamma(\omega^k,\theta^k).
\end{align*}
The recursion telescopes after summing the one-step operator-norm bound, we have
\begin{align}\label{Wov-diff}
    \sigma_\mathrm{max}(W_{\mathrm{ov}}^k-W_{\mathrm{ov}}^0)
    &\leq
    \sum_{t=0}^{k-1}\sigma_\mathrm{max}(W_{\mathrm{ov}}^{t+1}-W_{\mathrm{ov}}^t) \nonumber\\
    &\leq
    \sqrt{2}B_X\sum_{t=0}^{k-1}\eta_t \mathcal L_\gamma(\omega^t,\theta^t)
    =
    \sqrt{2}B_XS_k.
\end{align}
Using triangle inequality in \eqref{Wov-diff}, we have
\begin{align}\label{Wov}
    \sigma_\mathrm{max}(W_{\mathrm{ov}}^k)\leq
    \sqrt{2}B_XS_k+\sigma_\mathrm{max}(W_{\mathrm{ov}}^0).
\end{align}

For the trajectory of $\omega$, since $\nabla_\omega \mathcal L_\gamma(\omega^k,\theta^k)
=
\frac{\gamma}{N}J_\sigma(\omega^k)\ell_\mathrm{seq}(\theta^k)$, denote \(p^k:=\sigma(\omega^k)\) and
\(q^k:=(p^k)^\top\ell_\mathrm{seq}(\theta^k)\). Since
\[
J_\sigma(\omega^k)\ell_\mathrm{seq}(\theta^k)
=
p^k\odot \ell_\mathrm{seq}(\theta^k)-p^k q^k,
\]
the vector \(J_\sigma(\omega^k)\ell_\mathrm{seq}(\theta^k)\) has zero sum. Moreover, its
positive mass is bounded by \(q^k\). Hence,
\[
\|J_\sigma(\omega^k)\ell_\mathrm{seq}(\theta^k)\|_2
\le
\sqrt2 q^k.
\]
Therefore,
\begin{align*}
    \|\nabla_\omega \mathcal L_\gamma(\omega^k,\theta^k)\|_2
    &\le
    \frac{\sqrt2\gamma}{N}
    \sum_{i=1}^N[\sigma(\omega^k)]_i
    \mathcal L_{\mathrm{SFT}}(\theta^k;X_i,y_i)\\
    &\le
    \sqrt2\,\mathcal L_\gamma(\omega^k,\theta^k).
\end{align*}
The last inequality uses nonnegativity of the validation loss, so the weighted
training loss is bounded by the full penalty objective.

Thus,
\begin{align*}
    \|\omega^{k+1}-\omega^k\|_2
    &=
    \eta_k\|\nabla_\omega \mathcal L_\gamma(\omega^k,\theta^k)\|_2\\
    &\le
    \sqrt2\,\eta_k\mathcal L_\gamma(\omega^k,\theta^k).
\end{align*}
Summing over \(t=0,\ldots,k-1\), we obtain
\begin{align}\label{omega}
    \|\omega^k-\omega^0\|_2
    \le
    \sqrt2\sum_{t=0}^{k-1}\eta_t\mathcal L_\gamma(\omega^t,\theta^t)
    =
    \sqrt2 S_k .
\end{align}

\medskip
\noindent
\textbf{Step 2: Control of the \(W_{\mathrm{kq}}\)-trajectory.}
By Lemma~\ref{lem:Lgamma-gradient-bound-by-Lgamma}, for the trajectory of $W_\mathrm{kq}$, we have
\[
\sigma_\mathrm{max}(\nabla_{W_{\mathrm{kq}}}\mathcal L_\gamma(\omega,\theta)) \leq \|\nabla_{W_{\mathrm{kq}}}\mathcal L_\gamma(\omega,\theta)\|_F\leq \frac{\sqrt{2}}{2}B_X^2B_{-1}\sigma_\mathrm{max}(W_{\mathrm{ov}})\mathcal L_\gamma(\omega,\theta)
\]
Then we have one-step iteration inequality
\[
\sigma_\mathrm{max}(W_{\mathrm{kq}}^{k+1}-W_{\mathrm{kq}}^k)=\eta_k \sigma_\mathrm{max}(\nabla_{W_{\mathrm{kq}}}\mathcal L_\gamma(\omega^k,\theta^k))\le \frac{\sqrt{2}}{2}B_X^2B_{-1}\sigma_\mathrm{max}(W_{\mathrm{ov}}^k)\mathcal L_\gamma(\omega^k,\theta^k).
\]
The recursion telescopes after summing the one-step operator-norm bound and use \ref{Wov}, we have
\begin{align*}
\sigma_\mathrm{max}(W_{\mathrm{kq}}^{k}-W_{\mathrm{kq}}^0)
&\le
\frac{\sqrt{2}}{2} B_X^2B_{-1}\sum_{t=0}^{k-1} \sigma_\mathrm{max}(W_{\mathrm{ov}}^t)\eta_t \mathcal L_\gamma(\omega^t,\theta^t)\\
&\le
\frac{\sqrt{2}}{2} B_X^2B_{-1}\sum_{t=0}^{k-1} \left(\sqrt{2}B_XS_t+\sigma_\mathrm{max}(W_{\mathrm{ov}}^0)\right)\eta_t \mathcal L_\gamma(\omega^t,\theta^t)\\
& \overset{(a)}{=} \frac{\sqrt{2}}{2}B_X^2B_{-1}\left(\frac{\sqrt{2}}{2}B_X S_k^2+\sigma_\mathrm{max}(W_{\mathrm{ov}}^0)S_k-\frac{\sqrt{2}}{2}B_X\sum_{t=0}^{k-1}\eta_t^2\mathcal L_\gamma(\omega^t,\theta^t)^2\right) \\
& \leq \frac{1}{2}B_X^2B_{-1}\left(B_XS_k^2+\sqrt{2}\sigma_\mathrm{max}(W_{\mathrm{ov}}^0)S_k\right)\numberthis\label{Wkq}
\end{align*}
where (a) is the identity \(2\sum_{s<t}a_sa_t=(\sum_ta_t)^2-\sum_ta_t^2\), applied with \(a_t=\eta_t\mathcal L_\gamma(\omega^t,\theta^t)\).

\medskip
\noindent
\textbf{Step 3: Control of the $H$-trajectory.}

Recall that
\[
H(\omega,\theta)
:=
\left[
\left\{
\frac{1}{\sqrt{N'D}}h_\theta(\widetilde X_d^j)
\right\}_{j\in[N'],\,d\in[D]},
\left\{
\sqrt{\frac{\gamma\sigma_i(\omega)}{ND}}\,
h_\theta(X_d^i)
\right\}_{i\in[N],\,d\in[D]}
\right]\in \mathbb{R}^{V\times(N+N')D}.
\]
Then
\begin{align*}
\|H(\widetilde\omega,\widetilde\theta)
      -H(\omega,\theta)\|_F\le
\|H(\widetilde\omega,\widetilde\theta)
      -H(\omega,\widetilde\theta)\|_F
+
\|H(\omega,\widetilde\theta)
      -H(\omega,\theta)\|_F .
\end{align*}

For the \(\theta\)-dependent term, by Lemma~\ref{lem:h-basic}, we have
\begin{align*}
\|h_{\widetilde\theta}(\widetilde X_d^j)
-
h_\theta(\widetilde X_d^j)\|_2
&\le
\frac12 B_X^2B_{-1}
\sigma_\mathrm{max}(\widetilde W_{\mathrm{kq}}-W_{\mathrm{kq}}),\\
\|h_{\widetilde\theta}(X_d^i)
-
h_\theta(X_d^i)\|_2
&\le
\frac12 B_X^2B_{-1}
\sigma_\mathrm{max}(\widetilde W_{\mathrm{kq}}-W_{\mathrm{kq}}).
\end{align*}
Therefore,
\begin{align*}
\|H(\omega,\widetilde\theta)-H(\omega,\theta)\|_F^2&=\frac{1}{N'D}
\sum_{j=1}^{N'}\sum_{d=1}^D
\|h_{\widetilde\theta}(\widetilde X_d^j)-h_\theta(\widetilde X_d^j)\|_2^2\\
&\quad+
\frac{\gamma}{ND}
\sum_{i=1}^{N}\sum_{d=1}^D
\sigma_i(\omega)
\|h_{\widetilde\theta}(X_d^i)-h_\theta(X_d^i)\|_2^2\\
&\le
\frac14 B_X^4B_{-1}^2
\left(1+\frac{\gamma}{N}\right)
\sigma_\mathrm{max}(\widetilde W_{\mathrm{kq}}-W_{\mathrm{kq}})^2.
\end{align*}
Equivalently, we have
\begin{align}\label{ineq:H-theta}
\|H(\omega,\widetilde\theta)-H(\omega,\theta)\|_F
\le
\frac12 B_X^2B_{-1}
\left(1+\frac{\gamma}{N}\right)^{1/2}
\sigma_\mathrm{max}(\widetilde W_{\mathrm{kq}}-W_{\mathrm{kq}}).
\end{align}

For the \(\omega\)-dependent term, the validation columns do not depend on
\(\omega\). Thus
\begin{align*}
\|H(\widetilde\omega,\widetilde\theta)
      -H(\omega,\widetilde\theta)\|_F^2
&=
\frac{\gamma}{ND}
\sum_{i=1}^N\sum_{d=1}^D
\left(
\sqrt{\sigma_i(\widetilde\omega)}
-
\sqrt{\sigma_i(\omega)}
\right)^2
\|h_{\widetilde\theta}(X_d^i)\|_2^2\\
&\le
\frac{\gamma}{N}B_X^2
\left\|
\sqrt{\sigma(\widetilde\omega)}
-
\sqrt{\sigma(\omega)}
\right\|_2^2.
\end{align*}
Since the map \(r(\omega):=\sqrt{\sigma(\omega)}\) is globally
\(\sqrt2/4\)-Lipschitz, we obtain
\[
\left\|
\sqrt{\sigma(\widetilde\omega)}
-
\sqrt{\sigma(\omega)}
\right\|_2
\le
\frac{\sqrt2}{4}\|\widetilde\omega-\omega\|_2.
\]
Therefore,
\begin{align}\label{ineq:H-omega}
\|H(\widetilde\omega,\widetilde\theta)
      -H(\omega,\widetilde\theta)\|_F
\le
\frac{\sqrt2}{4}\sqrt{\frac{\gamma}{N}}B_X
\|\widetilde\omega-\omega\|_2.
\end{align}

Combining the two estimates \eqref{ineq:H-theta} and \eqref{ineq:H-omega} gives
\begin{align*}
&\sigma_\mathrm{max}(H(\widetilde\omega,\widetilde\theta)
      -H(\omega,\theta))\\
&\le
\|H(\widetilde\omega,\widetilde\theta)
      -H(\omega,\theta)\|_F\\
&\le
\frac12 B_X^2B_{-1}
\left(1+\frac{\gamma}{N}\right)^{1/2}
\sigma_\mathrm{max}(\widetilde W_{\mathrm{kq}}-W_{\mathrm{kq}})
+
\frac{\sqrt2}{4}\sqrt{\frac{\gamma}{N}}B_X
\|\widetilde\omega-\omega\|_2\\
&\le
\left(1+\frac{\gamma}{N}\right)^{1/2}
\left(
\frac12B_X^2B_{-1}
\sigma_\mathrm{max}(\widetilde W_{\mathrm{kq}}-W_{\mathrm{kq}})
+
\frac{\sqrt2}{4}B_X\|\widetilde\omega-\omega\|_2
\right).
\end{align*}

Taking \((\widetilde\omega,\widetilde\theta)=(\omega^k,\theta^k)\)
and \((\omega,\theta)=(\omega^0,\theta^0)\), and using the trajectory bounds,
\[
\sigma_\mathrm{max}(W_{\mathrm{kq}}^k-W_{\mathrm{kq}}^0)
\le
\frac12B_X^3B_{-1}S_k^2+\frac{\sqrt{2}}{2}B_X^2B_{-1}\sigma_\mathrm{max}(W_{\mathrm{ov}}^0)S_k,
\quad
\|\omega^k-\omega^0\|_2\le \sqrt2 S_k,
\]
we obtain
\begin{align*}
\sigma_\mathrm{max}(H^k-H^0)
&\le
\left(1+\frac{\gamma}{N}\right)^{1/2}
\left(
\frac12B_X^2B_{-1}
\sigma_\mathrm{max}(W_{\mathrm{kq}}^k-W_{\mathrm{kq}}^0)
+
\frac{\sqrt2}{4}B_X\|\omega^k-\omega^0\|_2
\right)\\
&\le
\left(1+\frac{\gamma}{N}\right)^{1/2}
\left(
\frac14B_X^5B_{-1}^2S_k^2
+
\left(\frac{\sqrt{2}}{4}B_X^4B_{-1}^2\sigma_\mathrm{max}(W_{\mathrm{ov}}^0)+\frac12B_X\right)S_k
\right).
\end{align*}
Under condition \eqref{ass:sigma-H}, we have
\begin{align*}
    \sigma_\mathrm{min}(H^k)&\geq \sigma_\mathrm{min}(H^0)-\sigma_\mathrm{max}(H^k-H^0) \\
    &\geq \left(1+\frac{\gamma}{N}\right)^\frac{1}{2}\left(\sigma_0-\frac14B_X^5B_{-1}^2S_k^2- \left(\frac{\sqrt{2}}{4}B_X^4B_{-1}^2\sigma_\mathrm{max}(W_{\mathrm{ov}}^0)+\frac12B_X\right)S_k \right)
\end{align*}
\end{proof}

\subsubsection{Proof of Theorem~\ref{thm:offline_pgdc}}
\begin{lemma}[Descent recursion under shifted local PL and smoothness]\label{lem:descent-recursion}
Under Assumption~\ref{ass:token_geometry}, let 
\begin{align*}
    \mu_k &= \frac{\sigma_\mathrm{min}(H^k)^2}{2V}\exp(-\sqrt{2}B_X\sigma_\mathrm{max}(W_\mathrm{ov}^k))\\
    L_k&=4B_X^2(1+B_{-1})
\left(1+B_XB_{-1}\sigma_\mathrm{max}(W_\mathrm{ov}^k)\right)^2
\left(1+\frac{\gamma}{N}\right)\\
&\quad+
\frac{3\gamma}{\sqrt{N}}
\left(
\log V
+
B_X
+
B_X\left(1+B_XB_{-1}\right)\sigma_\mathrm{max}(W_\mathrm{ov}^k)
\right)
\end{align*}
Suppose the step size satisfies \(\eta_k\le 1/L_k\) and the shifted local PL bound
\[
    \mathcal{L}_\gamma(\omega^{k+1},\theta^{k+1})
    \le
    (1-\eta_k\mu_k)\mathcal{L}_\gamma(\omega^k,\theta^k),
\]
\end{lemma}
\begin{proof}
    It follows directly from Lemma ~\ref{lem:one-step-smoothness} and Lemma~\ref{lem:nonuniform_pl_offline}. 
\end{proof}

\begin{restatable}[Complete version of Theorem~\ref{thm:offline_pgdc}]{theorem}{OfflinePBGDComplete}
\label{thm:offline_pgdc_complete}
Consider the offline PBGD updates in \eqref{eq:offline_PBGD_det} for the
one-layer softmax Transformer in \eqref{eq:transformer}. Suppose Assumption \ref{ass:seperable} holds, and additionally, assume the upper-level dataset is separable. 
Assume Assumption~\ref{ass:token_geometry} holds and $V\ge (N+N')D$.
Let $k_0\ge 0$ be a chosen burn-in index and fix a target accuracy
$\epsilon>0$. Set the penalty parameter as $\gamma=\epsilon^{-1/2}$ and define
\[
R_{k_0}:=\sigma_{\max}(W_{\mathrm{ov}}^{k_0}),
\quad
\sigma_{\min}(H^{k_0})
=
\left(1+\frac{\gamma}{N}\right)^{1/2}\sigma_{k_0},
\quad
\mathcal{L}_\gamma(\omega^{k_0},\theta^{k_0})=\left(1+\frac{\gamma}{N}\right)l^{k_0}
\]
and assume the burn-in representation is non-degenerate, i.e., $\sigma_{k_0}>0$.

Define the certified representation radius
\[
S_H(k_0)
:=
\frac{
-\left(
\frac{\sqrt2}{2}B_X^4B_{-1}^2R_{k_0}
+B_X
\right)
+
\sqrt{
\left(
\frac{\sqrt2}{2}B_X^4B_{-1}^2R_{k_0}
+B_X
\right)^2
+
2B_X^5B_{-1}^2\sigma_{k_0}
}
}{
B_X^5B_{-1}^2
}.
\]

Assume that at the burn-in step $k_0$, we have
\begin{align}\label{cond:initialization}
\frac{
8Vl^{k_0}
\exp\left(\sqrt2B_XR_{k_0}\right)
}{
\sigma_{k_0}^2
}
\le
M_H(k_0),
\end{align}
where $M_H(k_0)
:=
\max_{0<S\le S_H(k_0)}
S\exp(-2B_X^2S)$.
Then the set 
\begin{align}\label{def:S(k_0)}
\mathcal{S}(k_0)=\{S\in (0,S_H(k_0)]: \frac{
8Vl^{k_0}
}{
\sigma_{k_0}^2\exp(-2B_X^2S-\sqrt2B_XR_{k_0})
}\leq S\}
\end{align}
is non-empty. We define $S_\epsilon(k_0)=\sup_{S\in \mathcal{S}(k_0)}S$. And further define
\begin{align*}
\mu_\epsilon(k_0)
&:=
\left(1+\frac{\gamma}{N}\right)
\frac{\sigma_{k_0}^2}{8V}
\exp\left(
-2B_X^2S_\epsilon(k_0)
-\sqrt2B_XR_{k_0}
\right), \numberthis\label{eq:mu-epsilon-k0}\\
L_\epsilon(k_0)
&:=
4B_X^2(1+B_{-1})
\left(
1+B_XB_{-1}R_{k_0}
+\sqrt2B_X^2B_{-1}S_\epsilon(k_0)
\right)^2
\left(1+\frac{\gamma}{N}\right)
\\
&\quad+
\frac{3\gamma}{\sqrt N}
\left(
\log V+B_X
+B_X(1+B_XB_{-1})R_{k_0}
+\sqrt2B_X^2(1+B_XB_{-1})S_\epsilon(k_0)
\right).\numberthis\label{eq:L-epsilon-k0}
\end{align*}
For every $t\ge k_0$, choose the fixed stepsize $\eta_t = \eta := \frac{1}{L_\epsilon(k_0)}$.

Define 
\begin{align}\label{def:A_k0}
    A(k_0)
:=
\log\left(
1+(V-1)
\exp\left(
-2B_X^2S_\epsilon(k_0)
-\sqrt2B_XR_{k_0}
\right)
\right),
\end{align}
For any $\epsilon>\epsilon_0$, where
\begin{align}
\label{cond:epsilon_lower_bound}
\epsilon_0
=
\max\left\{2A(k_0),\left(\frac{2A(k_0)}{N}\right)^{2/3}\right\},
\end{align}
the iterates of the offline \texttt{BDS} satisfy
$\tau_{\epsilon}(k_0)
:=
\inf\left\{
k\ge k_0:
\mathcal L_\gamma(\omega^k,\theta^k)\le \epsilon
\right\},$ and
\begin{align}
\label{eq:hitting-time-bound}
\tau_\epsilon(k_0)-k_0
\le
\left\lceil
\frac{1}{\eta\mu_\epsilon(k_0)}
\log\frac{\mathcal L_\gamma(\omega^{k_0},\theta^{k_0})}{\epsilon}
\right\rceil
<\infty .
\end{align}
\end{restatable}
\begin{proof}
We prove the result by showing that the trajectory remains inside the
certified local region determined by $S_\epsilon(k_0)$, where both the
local PL constant and the smoothness constant are uniformly controlled.

For $k\ge k_0$, define the accumulated trajectory length
\[
S_{k,k_0}
:=
\sum_{t=k_0}^{k-1}
\eta\,\mathcal L_\gamma(\omega^t,\theta^t).
\]

Suppose first that $S_{k,k_0}\le S_\epsilon(k_0)$. 
By Lemma~\ref{lem:trajectory-representation-drift}, applied from the
burn-in point $k_0$, we have
\begin{align}
\label{eq:wov-drift-from-k0}
\sigma_{\max}(W_{\mathrm{ov}}^k)
&\le
R_{k_0}+\sqrt2B_XS_{k,k_0},
\\
\label{eq:H-drift-from-k0}
\sigma_{\min}(H^k)
&\ge
\left(1+\frac{\gamma}{N}\right)^{1/2}
\left[
\sigma_{k_0}
-\frac14B_X^5B_{-1}^2S_{k,k_0}^2
-\left(
\frac{\sqrt2}{4}B_X^4B_{-1}^2R_{k_0}
+\frac12B_X
\right)S_{k,k_0}
\right]
\end{align}
Since $S_{k,k_0}\le S_\epsilon(k_0)\le S_H(k_0)$, the definition of
$S_H(k_0)$ gives
\[
\frac14B_X^5B_{-1}^2S_{k,k_0}^2
+
\left(
\frac{\sqrt2}{4}B_X^4B_{-1}^2R_{k_0}
+\frac12B_X
\right)S_{k,k_0}
\le
\frac12\sigma_{k_0}.
\]
Therefore,
\begin{align}
\label{eq:H-lower-certified}
\sigma_{\min}(H^k)
\ge
\frac12
\left(1+\frac{\gamma}{N}\right)^{1/2}
\sigma_{k_0}.
\end{align}
Similarly, \eqref{eq:wov-drift-from-k0} gives
\begin{align}
\label{eq:Wov-upper-certified}
\sigma_{\max}(W_{\mathrm{ov}}^k)
\le
R_{k_0}+\sqrt2B_XS_\epsilon(k_0).
\end{align}

Combining \eqref{eq:H-lower-certified} and
\eqref{eq:Wov-upper-certified} with
Lemma~\ref{lem:nonuniform_pl_offline}, we obtain the uniform shifted local
PL bound inside the certified region:
\begin{align}
\label{eq:uniform-local-pl}
\|\nabla\mathcal L_\gamma(\omega^k,\theta^k)\|_F^2
\ge
2\mu_\epsilon(k_0)
\mathcal L_\gamma(\omega^k,\theta^k),
\end{align}
where $\mu_\epsilon(k_0)$ is defined in
\eqref{eq:mu-epsilon-k0}.

Moreover, by Lemma~\ref{lem:one-step-smoothness} and
\eqref{eq:Wov-upper-certified}, the local smoothness constant satisfies
$L_k\le L_\epsilon(k_0)$.
Since $\eta=1/L_\epsilon(k_0)$, we have $\eta\le 1/L_k$ in the certified
region. Hence Lemma~\ref{lem:descent-recursion} gives
\begin{align}
\label{eq:linear-descent-certified}
\mathcal L_\gamma(\omega^{k+1},\theta^{k+1})
\le
\left(1-\eta\mu_\epsilon(k_0)\right)
\mathcal L_\gamma(\omega^k,\theta^k).
\end{align}

On the other hand, iterating \eqref{eq:linear-descent-certified} yields
\begin{align}\label{eq:loss-linear-bound}
\mathcal L_\gamma(\omega^k,\theta^k)
&\le
\left(1-\eta\mu_\epsilon(k_0)\right)^{k-k_0}
\mathcal L_\gamma(\omega^{k_0},\theta^{k_0})
\end{align}
Consequently,
\begin{align*}
S_{k,k_0}
&=
\eta
\sum_{t=k_0}^{k-1}
\mathcal L_\gamma(\omega^t,\theta^t)
\\
&\le
\eta
\sum_{t=k_0}^{k-1}
\left(1-\eta\mu_\epsilon(k_0)\right)^{t-k_0}
\mathcal L_\gamma(\omega^{k_0},\theta^{k_0})
\\
&\le
\frac{
\mathcal L_\gamma(\omega^{k_0},\theta^{k_0})
}{
\mu_\epsilon(k_0)
} \overset{(a)}{\le}
S_\epsilon(k_0),
\end{align*}
where (a) is from the burn-in initialization condition
\eqref{cond:initialization}; indeed, by the definitions of
$\mathcal L_\gamma(\omega^{k_0},\theta^{k_0})$ and $\mu_\epsilon(k_0)$,
\[
\frac{
\mathcal L_\gamma(\omega^{k_0},\theta^{k_0})
}{
\mu_\epsilon(k_0)
}
=
\frac{
8Vl^{k_0}
\exp\left(2B_X^2S_\epsilon(k_0)+\sqrt2B_XR_{k_0}\right)
}{
\sigma_{k_0}^2
}
\le S_\epsilon(k_0).
\]
Therefore, by induction, the whole
trajectory remains inside the certified tube
\[
S_{k,k_0}\le S_\epsilon(k_0),
\qquad k\ge k_0.
\]
For any step $t\ge k_0$ such that $\sigma_\mathrm{max}(W_\mathrm{ov}^t)\leq \sigma_\mathrm{max}(W_\mathrm{ov}^{k_0})+\sqrt{2}B_XS_\epsilon(k_0)$, the lower bound $\mathcal{L}_\gamma(\omega^{t},\theta^t)\geq \left(1+\frac{\gamma}{N}\right)A(k_0):=\mathcal{L}_\gamma^*(k_0)$ where $A(k_0)$ is defined in \eqref{def:A_k0}.

With the penalty choice $\gamma=\epsilon^{-1/2}$, by the condition \eqref{cond:epsilon_lower_bound},
\begin{align*}
    A(k_0)\le \frac{\epsilon}{2},
    \qquad
    \frac{A(k_0)}{N\sqrt{\epsilon}}\le \frac{\epsilon}{2}.
\end{align*}
Consequently,
\[
    \mathcal{L}_\gamma^*(k_0)
    =
    \left(1+\frac{\gamma}{N}\right)A(k_0)
    =
    A(k_0)+\frac{A(k_0)}{N\sqrt{\epsilon}}
    \le \epsilon.
\]

When the condition \eqref{cond:epsilon_lower_bound} holds true, the uniform PL \eqref{eq:mu-epsilon-k0} and smoothness \eqref{eq:L-epsilon-k0} bounds above hold.

Now, from \eqref{eq:loss-linear-bound} and the inequality
$1-x\le e^{-x}$, we have
\[
\mathcal L_\gamma(\omega^k,\theta^k)
\le
\exp\left(
-\eta\mu_\epsilon(k_0)(k-k_0)
\right)
\mathcal L_\gamma(\omega^{k_0},\theta^{k_0}).
\]
Thus it is sufficient to take
\[
k-k_0
\ge
\left\lceil
\frac{1}{\eta\mu_\epsilon(k_0)}
\log\frac{\mathcal{L}_\gamma(\omega^{k_0},\theta^{k_0})}{\epsilon}
\right\rceil .
\]
This proves the hitting-time bound
\[
\tau_\epsilon(k_0)-k_0
\le
\left\lceil
\frac{1}{\eta\mu_\epsilon(k_0)}
\log\frac{\mathcal L_\gamma(\omega^{k_0},\theta^{k_0})}{\epsilon}
\right\rceil
<\infty.
\]
\end{proof}

\subsection{Proof of Theorem \ref{thm:eqv}}\label{sec:proof_thm1}
Before proceeding, we first define the bilevel linear scalarization problem as 
\begin{align}\label{eq:BLS}
&\texttt{BLS}:~ \min_{\lambda\in\Delta^N, \theta}~ \frac{1}{N^\prime}\sum_{i=1}^{N^\prime}\mathcal{L}_{\rm val} (\theta;\tilde x^i,\tilde y^i), ~\text{ s.t. }~ \theta\in\mathcal{S}(\lambda):=\argmin_{\theta^\prime} \frac{1}{N}\sum_{i=1}^N \lambda_i\mathcal{L}_{\text{SFT}} (\theta^\prime;x^i,y^i)
\end{align} 
where $\Delta^N:=\{\lambda\in\mathbb{R}^N: \lambda_i\geq 0, \sum_{i=1}^N \lambda_i=1\}$ is the simplex. Clearly, by setting $\lambda_i=\sigma_i(\omega)$, $\texttt{BLS}$ is equivalent to $\texttt{BDS}$. So the remaining task is to prove the equivalence of $\texttt{BLS}$ and $\texttt{BMO}$. 

\begin{proof}
Define \(f_i(\theta):=\mathcal{L}_{\text{SFT}}(\theta;x^i,y^i)\) for each \(i\in[N]\), and write \(f(\theta):=(f_1(\theta),\ldots,f_N(\theta))\). By switching the optimization order in the upper level, the equivalence of \(\texttt{BLS}\) and \(\texttt{BMO}\) reduces to proving
\[
\operatorname{WP}(f)=\bigcup_{\lambda\in\Delta^N}\mathcal{S}(\lambda).
\]

By Assumption \ref{ass:seperable}, there exists \(\bar{\theta}\in\mathbb{R}^h\) such that
\[
f_i(\bar{\theta})=0,\qquad \forall i\in[N].
\]
Since each \(f_i(\theta)\geq 0\), for any \(\lambda\in\Delta^N\) the scalarized lower-level objective $\frac{1}{N}\sum_{i=1}^N \lambda_i f_i(\theta)$
is nonnegative and reaches value \(0\) at \(\bar{\theta}\). Therefore,
\begin{align}
\theta\in\mathcal{S}(\lambda)
&\Longleftrightarrow
\frac{1}{N}\sum_{i=1}^N \lambda_i f_i(\theta)=0 \nonumber\\
&\Longleftrightarrow
f_i(\theta)=0 \quad \text{for every } i \text{ with } \lambda_i>0.
\label{eq:scalarization_characterization_neurips}
\end{align}
Hence
\begin{align}
\bigcup_{\lambda\in\Delta^N}\mathcal{S}(\lambda)
=
\left\{
\theta\in\mathbb{R}^h:
\exists i\in[N]\text{ such that } f_i(\theta)=0
\right\}.
\label{eq:union_characterization_neurips}
\end{align}

We next characterize the weak Pareto set of \(f\). First, suppose \(f_i(\theta)>0\) for every \(i\in[N]\). Then \(\bar{\theta}\) strictly improves all coordinates because
\[
f_i(\bar{\theta})=0<f_i(\theta),\qquad \forall i\in[N].
\]
Hence \(\theta\notin \operatorname{WP}(f)\).

Conversely, if \(f_j(\theta)=0\) for some \(j\in[N]\), then there does not exist any \(\theta^\prime\) such that
\[
f_j(\theta^\prime)<f_j(\theta)=0,
\]
because every SFT loss is nonnegative. Therefore no \(\theta^\prime\) can strictly improve all coordinates simultaneously, which implies \(\theta\in\operatorname{WP}(f)\). We have thus proved
\begin{align}
\operatorname{WP}(f)
=
\left\{
\theta\in\mathbb{R}^h:
\exists i\in[N]\text{ such that } f_i(\theta)=0
\right\}.
\label{eq:wp_characterization_neurips}
\end{align}
Combining \eqref{eq:union_characterization_neurips} and \eqref{eq:wp_characterization_neurips}, we obtain
\[
\operatorname{WP}(f)=\bigcup_{\lambda\in\Delta^N}\mathcal{S}(\lambda).
\]
Therefore \(\texttt{BMO}\) and \(\texttt{BLS}\) minimize the same validation objective over the same feasible set, so their global minimizers coincide. Since \(\texttt{BLS}\) is equivalent to \(\texttt{BDS}\) through the change of variables \(\lambda_i=\sigma_i(\omega)\), the global solutions of \(\texttt{BMO}\) and \(\texttt{BDS}\) coincide.
\end{proof}

\section{Implicit weight assigned by $\texttt{BMO}$}
\label{sec:BMO_implicit_weight}

In this section, we provide the proof of Lemma \ref{lemma:implicit_BMO}. First we derive the generalized implicit weight assigned by bilevel multi-objective optimization problem and then specify each objective as the response-level SFT loss so that we can obtain Lemma \ref{lemma:implicit_BMO}. 

In previous section, we showed that the optimal LLM model $\theta$ given by $\texttt{BDS}$ and $\texttt{BMO}$ are the same. This inspires the question that 
\begin{center}
  \emph{How to interpret the implicit weights assigned by algorithm for solving  $\texttt{BMO}$?}   
\end{center}
To solve $\texttt{BMO}$ in \eqref{eq:BMOL}, we consider converting weak Pareto set optimization into a scalar objective. In  \citep{tanabe2024new}, it is shown that weakly Pareto set $\operatorname{WP}(\theta)$ for any lower semicontinuous multi-function $\mathcal{L}(\theta)$ can be equivalently expressed by a merit function  
\begin{align}\label{merit_function_max}
u(\theta)=\sup_{\theta^\prime}\min_{m\in [M]}\left\{\mathcal{L}_m(\theta)-\mathcal{L}_m(\theta^\prime)\right\}
\end{align}
where $u(\theta)\geq 0$ and the equality holds if and only if $\theta\in \operatorname{WP}(\theta)$. This merit function is easy to check by the definition of weak Pareto set in Definition \ref{def:WP} as 
\begin{align*}
\operatorname{WP}(\mathcal{L})&=\{\theta \mid \not\exists \theta' \text{ s.t. } \forall m\in [M], \mathcal{L}_m(\theta') < \mathcal{L}_m(\theta)\}\\
&=\{\theta \mid \forall \theta' \text{ s.t. } \exists m\in [M], \mathcal{L}_m(\theta') \geq  \mathcal{L}_m(\theta)\}
\end{align*}
while the latter can be expressed in terms of the supremum and infimum in \eqref{merit_function_max}. Consequently, $\texttt{BMO}$ can be reformulated to $\min_\theta \mathcal{L}_{\rm val}(\theta), \text{ s.t. } u(\theta)\leq 0$, and can be solved sequentially by its penalty reformulation \citep{chen2025efficient} 
\begin{align}\label{eq:penalty_BMO}
\texttt{PMO}: ~\min_\theta~ \mathcal{L}_{\rm val}(\theta)+\gamma_k u(\theta)
\end{align}
with enlarging $\gamma_k\rightarrow\infty$ similar to \texttt{BDS}. Then the bottleneck of designing a gradient-based algorithm on \eqref{eq:penalty_BMO} is $u(x)$, which is usually non-differentiable due to $\max$ operator. 

One popular choice to estimate  $u(x)$ is to rewrite it in a $\operatorname{max}$ form and estimate the $\operatorname{max}$ operator by log-sum-exponential (LSE) function \citep{rockafellar2009variational} 
\begin{align*}
\operatorname{LSE}(q;\tau)=\frac{1}{\tau}\log\left(\sum_{m=1}^M\exp{(\tau q_m)}\right). 
\end{align*}
It is well-known that for any vector $q\in\mathbb{R}^M$, $\max_{m\in [M]} q_m\leq \operatorname{LSE}(q;\tau)\leq \max_{m\in [M]} q_m+\frac{\log M}{\tau}$
\citep{boyd2004convex,rockafellar2009variational}. Therefore, $\min_{m\in [M]} q_m\approx -\operatorname{LSE}(-q;\tau)$ can be approximated by a smooth function. Applying these results into $\texttt{PMO}$ in \eqref{eq:penalty_BMO} leads to 
{\small
\begin{align}\label{eq:penalty_BMO_lse}
\texttt{PMO}_{\operatorname{LSE}}: ~\min_\theta~ \mathcal{L}_{\rm val}(\theta)-\frac{\gamma_k}{\tau} \inf_{\theta^\prime}\operatorname{LSE}\left(\mathcal{L}(\theta^\prime)-\mathcal{L}(\theta);\tau\right)
\end{align}}

When leveraging  $\texttt{PMO}_{\operatorname{LSE}}$ to solve the $\texttt{BDS}$ problem, the formulation can be further simplified. 
\begin{lemma}\label{lemma:LSE_selection}
Letting $\mathcal{L}_m(\theta)=\mathcal{L}_{\text{SFT}}(\theta;x^m,y^m)$ and under Assumption \ref{ass:seperable}, \eqref{eq:penalty_BMO_lse} can be simplified as 
\begin{align}\label{eq:penalty_BMO_lse_reform}
\texttt{PMO}_{\operatorname{LSE}}^{\operatorname{S}}: ~~\min_\theta~ \mathcal{L}_{\rm val}(\theta)-\frac{\gamma_k}{\tau} \operatorname{LSE}(-\mathcal{L}(\theta);\tau)
\end{align}
where $\mathcal{L}(\theta):=\left[\orange{\mathcal{L}_{\text{SFT}}(\theta;x^1,y^1)},\cdots, \orange{\mathcal{L}_{\text{SFT}}(\theta;x^M,y^M)}\right]$ is the per-sample SFT loss vector. 
\end{lemma}

\begin{proof}
Letting $\theta^*$ be the shared minimizer in Assumption \ref{ass:seperable}, we first want to prove 
\begin{align}\label{sup_min}
\sup_{\theta^\prime}\min_{m\in [M]}\left(\mathcal{L}_m(\theta)-\mathcal{L}_m(\theta^\prime)\right)=\min_{m\in [M]}\mathcal{L}_m(\theta). 
\end{align}
which is equivalent to 
\begin{align*}
\sup_{\theta^\prime}\min_{\lambda\in\Delta^M}\sum_{m=1}^M\lambda_m\left(\mathcal{L}_m(\theta)-\mathcal{L}_m(\theta^\prime)\right)=\min_{\lambda\in\Delta^M}\sum_{m=1}^M\lambda_m\mathcal{L}_m(\theta) \end{align*}
because $\min_{m \in [M]}q_m=\min_{\lambda\in\Delta^M}\lambda_m q_m$. 

First we have 
{\small 
\begin{align}
\sup_{\theta^\prime}\min_{\lambda\in\Delta^M}\sum_{m=1}^M\lambda_m\left(\mathcal{L}_m(\theta)-\mathcal{L}_m(\theta^\prime)\right)&\geq \min_{\lambda\in\Delta^M}\sum_{m=1}^M\lambda_m\left(\mathcal{L}_m(\theta)-\mathcal{L}_m(\theta^*)\right)= \min_{\lambda\in\Delta^M}\sum_{m=1}^M\lambda_m\mathcal{L}_m(\theta). \label{medium_5_0}
\end{align}}
On the other hand, for any $m\in [M]$, 
\begin{align}\label{medium_5_1}
\sup_{\theta^\prime}\min_{\lambda\in\Delta^M}\sum_{m=1}^M\lambda_m\left(\mathcal{L}_m(\theta)-\mathcal{L}_m(\theta^\prime)\right)&\leq \sup_{\theta^\prime}\left(\mathcal{L}_m(\theta)-\mathcal{L}_m(\theta^\prime)\right)=\mathcal{L}_m(\theta)
\end{align}
Since \eqref{medium_5_1} holds for any $m\in [M]$, 
\begin{align}
\sup_{\theta^\prime}\min_{\lambda\in\Delta^M}\sum_{m=1}^M\lambda_m\left(\mathcal{L}_m(\theta)-\mathcal{L}_m(\theta^\prime)\right)&\leq\min_{m\in [M]}\mathcal{L}_m(\theta)=\min_{\lambda\in\Delta^M}\sum_{m=1}^M\lambda_m\mathcal{L}_m(\theta). 
\end{align}
where the last equality is because $\min_{m\in [M]} q_m=\min_{\lambda\in\Delta^M}\sum_{m=1}^M q_m$. Then with \eqref{medium_5_0}, we know 
\begin{align*}
\sup_{\theta^\prime}\min_{\lambda\in\Delta^M}\sum_{m=1}^M\lambda_m\left(\mathcal{L}_m(\theta)-\mathcal{L}_m(\theta^\prime)\right)=\min_{\lambda\in\Delta^M}\sum_{m=1}^M\lambda_m\mathcal{L}_m(\theta). \end{align*}
so that \eqref{sup_min} holds. Then 
\begin{align}\label{lse_medium}
\inf_{\theta^\prime}\operatorname{LSE}\left(\mathcal{L}(\theta^\prime)-\mathcal{L}(\theta);\tau\right)\leq \operatorname{LSE}\left(\mathcal{L}(\theta^*)-\mathcal{L}(\theta);\tau\right)=\operatorname{LSE}\left(-\mathcal{L}(\theta);\tau\right). 
\end{align}
On the other hand, 
\begin{align}\label{lse_medium_1}
\inf_{\theta^\prime}\operatorname{LSE}\left(\mathcal{L}(\theta^\prime)-\mathcal{L}(\theta);\tau\right)&\stackrel{(a)}{\geq} \inf_{\theta^\prime}\left\{\max_{m\in [M]} \left(\mathcal{L}_m(\theta^\prime)-\mathcal{L}_m(\theta)\right)\right\}\stackrel{(b)}{=}\max_{m\in [M]}-\mathcal{L}_m(\theta)\nonumber\\
&\geq \operatorname{LSE}\left(-\mathcal{L}(\theta);\tau\right)-\frac{\log M}{\tau}
\end{align}
where (a) is because $\max_{m\in [M]} q_m\leq \operatorname{LSE}(q;\tau)$, (b) is earned by \eqref{sup_min} and $\min_{m\in [M]} q_m=\min_{\lambda\in\Delta^M}\sum_{m=1}^M q_m$, and (c) is because $\operatorname{LSE}(q;\tau)\leq \max_{m\in [M]} q_m+\frac{\log M}{\tau}$. Then since $\frac{\log M}{\tau}$ is a constant, we get the conclusion. 
\end{proof}

\noindent\textbf{Importance ratio is proportional to the implicit weight of $\texttt{BMO}$ with $\tau=1$. } 
Taking the gradient over $\texttt{PMO}_{\operatorname{LSE}}^{\operatorname{S}}$ in \eqref{eq:penalty_BMO_lse_reform} gives the update direction as 
\begin{align}\label{update-LSE}
\nabla\mathcal{L}_{\rm val}(\theta)+\gamma_k\sum_{m=1}^M\lambda_m\nabla\mathcal{L}_{\text{SFT}}(\theta;x^m,y^m)
\end{align}
where $\lambda_m=\frac{\exp(-\tau\orange{\mathcal{L}_{\text{SFT}}(\theta;x^m,y^m)})}{\sum_{i=1}^M\exp(-\tau\mathcal{L}_{\text{SFT}}(\theta;x^i,y^i))}$ is given by the softmax policy of negative per-sample SFT loss. Compared with \eqref{PBGD-1}, instead of directly learning the data weight, $\texttt{PMO}_{\operatorname{LSE}}^{\operatorname{S}}$ chooses a special logit of data weight as 
\begin{align}\label{explicit_data_logit}
\omega^*(\theta)=-\tau\left[\orange{\mathcal{L}_{\text{SFT}}(\theta;x^1,y^1)},\cdots, \orange{\mathcal{L}_{\text{SFT}}(\theta;x^M,y^M)}\right]
\end{align}

which is also inversely determined by the per-sample SFT loss. If the per-sample SFT loss of a data sample remains low after joint descent with the validation function, it is likely drawn from the same distribution as the validation set, and thus we assign a higher weight to this sample. 

\noindent\textbf{Application to Lemma \ref{lemma:implicit_BMO}. } Consider $i$-th question and let $y^m, m\in [G]$ be the group of generations to $i$-th  question. Then the implicit weight of $g$-th generation is $$\lambda_{i,g}=\operatorname{softmax}(-\tau\mathcal{L}_{\text{SFT}}(\theta;x^i,y^{i,g}_{\text{old}}))\propto \exp(-\tau\mathcal{L}_{\text{SFT}}(\theta;x^i,y^{i,g}_{\text{old}}))=\pi_\theta(x^i,y^{i,g}_{\text{old}})$$ 
which is proportional to the importance ratio of $g$-th response. In this way, we are not only assigning question-level validation score, but also enforce response-level validation score to prioritize response adhere to the validation dataset in the proposed algorithm.

\section{Additional experiments}\label{sec:additional_experimernts}

\subsection{Details of experimental setup}

In this section, we present the detailed experimental setup and the hyperparameter choices. 

\noindent\textbf{General Setup.} For both fine-tuning tasks, we use $56000$ samples for the \textsc{Llama-3-8b} \citep{dubey2024llama} model and $9600$ samples for the \textsc{Pythia}-1b \citep{biderman2023pythia} model. We only process the question with length shorter than $2048$. 
Both models are adapted with Low-Rank Adaptation (LoRA) (\textsc{alpha} $16$, \textsc{rank} $16$). The learning rates are set to $5\times 10^{-6}$ for the LoRA parameter $\theta$ update and $1\times 10^{-4}$ for the selector $\omega$ update, using Adam \citep{kingma2014adam} as the optimizer. Fine-tuning was performed using PyTorch with the DeepSpeed library \url{https://github.com/deepspeedai/DeepSpeed} to optimize memory usage. We use effective train batch size $32$ and micro batch size $8$ for both tasks, and use zero stage $2$ with gradient checkpointing in DeepSpeed. We train $3$ epochs for all algorithms. 

\noindent\textbf{Algorithm hyperparameters.} We use penalty constant of $\gamma_k=\frac{\rho_k}{1-\rho_k}$ with $\rho_k$ initialized as $0.1$ and increased by $0.1$ after every epoch for both offline and online selection, as suggested by \citep{shen2024seal}. For the online selection approach, we generate responses to the masked questions using a batch size of $64$ for the first quality enhancement task and $16$ for the second safety-aware fine-tuning task. We generate new responses every $K_{\text{gen}}=500$ iterations, with a maximum of $512$ tokens per response and a generation temperature of $0.8$ to allow moderate exploration. 

\subsection{Quality comparisons of model outputs}\label{sec:generated_response}

Table \ref{tab:sft-eval-examples} shows the responses generated by different methods using \textsc{Llama-3-8b Instruct} \citep{grattafiori2024llama} as the base model for the upper-level \textsc{OpenOrca} dataset. The responses generated by online self-refining are shorter yet clearer, which is consistent with the better performance on the tidy lower-level dataset. It also follows instructions more faithfully: for example, only online self-refining correctly rates the restaurant as $3/5$ rather than five stars. Moreover, while offline \texttt{BDS} provides a partially correct answer to the third question, it fails to adhere to the given options. 

Table \ref{tab:ranking_question} gives an overview of the top $10\%$ and bottom $10\%$ questions ranked by offline \texttt{BDS} and online self-refining. The online self-refining approach shows a clear trend of ranking easier questions higher at the beginning and processing harder ones later. We also present template responses for some of the questions, sorted by difficulty level, in Table \ref{tab:ranking_question_response}, which shows that response length can partially indicate question difficulty.


\begin{table}[t]
    
    \begin{minipage}{0.96\linewidth}
    \renewcommand{\arraystretch}{1.6} 
    \begin{tabular}{p{0.31\linewidth} p{0.31\linewidth} p{0.31\linewidth}}
    \hline\hline
    \multicolumn{3}{c}{\textbf{Examples of generated responses from \textsc{Llama-3-8b} model fine-tuned by different methods}} \\
    \hline\hline
    \multicolumn{3}{p{\linewidth}}{\textbf{Human}: John was a terrible writer.  To practice, his teacher suggest that he consider people he knows and do what?
    Options: - write novels
    - advertising firm
    - write letter
    - write notes
    - write poems
    Let's think now! Step-by-step reasoning:} \\
    \textbf{\textsc{Llama-3-8b} \citep{grattafiori2024llama}}:\newline
    If John is a terrible writ... (read more)
    \red{Option: B - advertising firm. Explanation: The correct answer is option B - "advertising firm". This is because the question states that John's teacher suggests that he "consider people he know" and "do what". In an advertising firm, one of the tasks would be to create ads for products or services, which requires writing. By doing so, John can practice his writing skills without having to write long, elaborate pieces like novels or poems. He can start with simple texts like advertisements, which are shorter and more focused on conveying information. } 
    & \textbf{Offline Selection}:\newline 
    1. John is a terrible at writing.
   2. His teacher suggests practicing to improve. 
    3. To practice writing, he can think about people  \orange{he know. }

    \olivegreen{The answer would be "write letters".} Because it's easier for him to start by thinking of someone he knows (like family or friends) and then writing a letter to them. This way he can practice his writing skills while also being more comfortable with the topic.
    
    \orange{Now let's see if we're right:
    Write your own explanation below in the space provided. }
    & \textbf{Online Self-refining}:\newline
    1. John is a terrible at writing.
    2. His teacher suggests doing something to improve writing skills.
    3. To improve writing, one can focus on writing about familiar topics or people \olivegreen{they know. }
    
    \olivegreen{So the best answer is option C) write letters, as this will allow him to practice writing by thinking of people he know and writing a letter to them.} This can be an easy way for him to get started with improving his writing.
    \\
    \hline
    \multicolumn{3}{p{\linewidth}}{\textbf{Human}: Generate an approximately fifteen-word sentence that describes all this data: Midsummer House eatType restaurant; Midsummer House food Chinese; Midsummer House priceRange moderate; Midsummer House customer rating 3 out of 5; Midsummer House near All Bar One} \\
    \textbf{\textsc{Llama-3-8b} \citep{grattafiori2024llama}}:\newline
    Midsummer House is a moderately priced Chinese restaurant with a \red{3-star rating nearby major hotels. }  
    & \textbf{Offline Selection}:\newline
    Midsummer House is a moderately priced, \red{three-star-rated} Chinese eatery \orange{located near All-Bar-One. } 
    & \textbf{Online Self-refining}:\newline
    Midsummer House, a moderately-priced Chinese restaurant with a \olivegreen{3/5 rating, is located near All bar one. } 
    \\
    \hline
    \multicolumn{3}{p{\linewidth}}{\textbf{Human}: Premise: A man is inline skating in front of a wooden bench.

    Hypothesis: A man is having fun skating in front of a bench.
    
    Choose the correct answer: Given the premise, can we conclude the hypothesis?
    
    Select from: a). yes b). it is not possible to tell c). no.} \\
    \textbf{\textsc{Llama-3-8b} \citep{grattafiori2024llama}}:\newline
    Correct answer: \red{a) yes}
    & \textbf{Offline Selection}:\newline
        \orange{d). maybe}
    & \textbf{Online Self-refining}:\newline
    \olivegreen{b). it is not possible to tell }
    \\
    \hline
    \hline
    \end{tabular}
    \vspace{0.1cm}
    \end{minipage}
    \caption{Examples of generated response on \textsc{OpenOrca} dataset using \textsc{Llama-3-8b} \citep{grattafiori2024llama} and fine-tuned with offline \texttt{BDS} and online self-refining. Text marked in \red{red} indicates incorrect outputs, \orange{orange} indicates partially correct outputs or irrelevant information, and \olivegreen{green} indicates fully correct outputs that match the expected instructions.}
    \label{tab:sft-eval-examples}
        \vspace{-0.6cm}
\end{table}

\begin{table*}[t]
\centering
\small
\begin{tabular}{ll p{0.45\linewidth} p{0.42\linewidth}}
\toprule
\textbf{Method} & \textbf{Rank} & \textbf{Epoch 1} & \textbf{Epoch 3} \\
\midrule
\multirow{2}{*}{Offline} & Top & Classify the following text as either satire or non-satire. "The last presidential election was a great affair with exciting twists and turns that kept us all on our toes." \orange{[Medium]} 

Generate a formal invitation for a networking event. We invite you to join us for an informal networking event. \red{[Hard]}

Use a variety of language and words to rewrite the given sentence. He was very tired and cann't go any further. \orange{[Medium]} 

Rewrite this code in C++. public static boolean isAnagram(String str1, String str2) {  
    char[] charArray1 = str1.toCharArray();  
    char[] charArray2 = str2.toCharArray();  
    Arrays.sort(charArray1);  
    Arrays.sort(charArray2);  
    return Arrays.equals(charArray1, charArray2);  
} \red{[Hard]} 
& Find information about the primary schools in Johannesburg \red{[Hard]}

Create a haiku poem using the provided words. Wind, Clouds, Sky \red{[Hard]}

Identify the cause of this issue. The computer is not working. \red{[Hard]}

Write a story using the given words in your story. desert, moonlit, violin. \red{[Hard]} 

Calculate the average speed of a car traveling 120 miles in 2 hours. 120 miles in 2 hours. \orange{[Medium]} 
\\
\cmidrule(l){2-4}
& Bottom & Delete all words with more than 5 letters from this sentence. This sentence has many long words like 'sentence' and 'instruction'. \cryan{[Easy]} 

Rearrange the words in the sentence to form a question. Reading is difficult. \cryan{[Easy]} 

Is the following sentence structured correctly? We went for a walk in the park and played hide and seek. \cryan{[Easy]} 

Convert the following number in scientific notation: 0.567. \cryan{[Easy]} 

& Provide a list of materials needed for the given project. A school project to build a model of a volcano. \red{[Hard]} 

Is the following sentence structured correctly? We went for a walk in the park and played hide and seek. \cryan{[Easy]} 

Translate the following sentence from English to Spanish. Output less than 25 words. I am learning Spanish. \cryan{[Easy]} 

Write a code that prints the following string in all lowercase letters. Mixed Case Characters \red{[Hard]} 
\\
\midrule
\multirow{2}{*}{Online} & Top & Find and replace all instances of the word "great" in the sentence with synonyms. The teacher's great approach in teaching helped the students to understand the lessons better. \cryan{[Easy]}

Re-write the following sentence omitting the word "comfortable". We were quite comfortable with our decision. \cryan{[Easy]} 

Evaluate the following sentence and provide feedback on the spelling and punctuation errors. The frog jumpted acros the road. \cryan{[Easy]} 

Combine the sentences below into an essay. There are many ways to reduce waste. For example, reducing the use of plastic. Reusing materials is also important. \cryan{[Easy]}  & Create a short story in the horror genre based on the given setting. Setting: An abandoned island \red{[Hard]} 

Given a set of data points, create an equation for the linear regression line. Data points: (1,1), (2,2), (3,4) \red{[Hard]} 

Come up with an experiment that tests the given concept. The effect of heavy metals on plant growth \red{[Hard]} 

Edit the given text such that its length is not changed but it clearly conveys the meaning. The app has got a very good user interface and is really nice to use. \red{[Hard]} 
\\
\cmidrule(l){2-4}
& Bottom & Write a script for a movie scene with the following input. Setting: Airport
Characters: Rebecca and Jacob \red{[Hard]} 

Pick a random song and create a dance routine for it. "Just a Dream" by Nelly. \red{[Hard]}

Write an email to a customer thanking them for their purchase. Customer Name – John Smith
Product Purchased – Nike Shoes
Delivery Date – 6/20/2021. \red{[Hard]}

Explain the meaning of the following phrase. "out of the box". \orange{[Medium]} 
& Assign the following verb a category of either transitive, intransitive, or linking. Jumped  \cryan{[Easy]} 

Add a transition sentence at the end of this passage: Jack was so excited for his birthday party. He had spent months planning it. He had invited all his friends and family members. \cryan{[Easy]} 

Reverse the order of words in the sentence. Alice visited the museum. \cryan{[Easy]} 

Delete all words with more than 5 letters from this sentence. This sentence has many long words like 'sentence' and 'instruction'. \cryan{[Easy]} 
\\
\bottomrule
\end{tabular}
\caption{Examples of top $10\%$ questions ranked by offline \texttt{BDS} and online self-refining. We annotate each question as \cryan{[Easy]}, \orange{[Medium]} or \red{[Hard]} based on whether they need creativity and deep thinking or just data analysis and reorganization. Online self-refining presents a clear trend of learning from simple to hard questions, while offline \texttt{BDS} tends to keep emphasizing harder questions. Learning from simple questions like 'combine the sentences below into an essay' serves as a good foundation for harder creativity questions like movie-script and poem writing. } 
\label{tab:ranking_question}
\vspace{-0.3cm}
\end{table*}

\begin{table*}[t]
\centering
\small
\begin{tabular}{l p{0.25\linewidth}  p{0.75\linewidth}}
\toprule
\textbf{Level}  &\textbf{Question} & \textbf{Responses} \\
\midrule
\cryan{Simple} & Delete all words with more than 5 letters from this sentence. This sentence has many long words like 'sentence' and 'instruction'.   & This has long like and.   \\
\midrule
\cryan{Simple} & Assign the following verb a category of
either transitive, intransitive, or linking.
Jumped.   & Jumped is an intransitive verb.    \\ 
\midrule
\orange{Medium} & Classify the following text as either satire or non-satire. "The last presidential election was a great affair with exciting twists and turns that kept us all on our toes."   & This text would classify as non-satire.     \\ 
\midrule
\red{Hard} & Find information about the primary schools in Johannesburg.   & Johannesburg is the largest city in South Africa and is home to a number of primary schools. These schools cater to students from the ages of 5-12 and provide education from Grade R (the Reception year) to Grade 7. Johannesburg’s primary schools are either public (government-funded) or independent (privately-funded) institutions, and they offer the National Curriculum or the independent schools’ curriculum.

Some well-known primary schools in Johannesburg include:
- Parkview Senior Primary School: A public school located in Parkview, Johannesburg.
- Greenside Primary School: A public school located in Greenside, Johannesburg.
- St Katharine's School : A private Anglican school located in Parktown, Johannesburg.
- The Ridge School: An independent boys' school located in Westcliff, Johannesburg.
- Pridwin Preparatory School: A private boys' school located in Melrose, Johannesburg.
- Auckland Park Preparatory School: A private girls' school located in Auckland Park, Johannesburg.

Admission requirements vary depending on the school, but most schools require an application form, an interview, and an assessment of the child’s abilities. Potential students may also be placed on a waiting list. It is recommended to research the individual schools in Johannesburg and their specific admission requirements, fees, and curricula before making a decision on which school to enroll a child.  \\
\midrule
\red{Hard} & Write a script for a movie scene with the follow-ing input. Setting: Airport Characters: Rebecca
and Jacob   & Fade in:
EXT. AIRPORT - DAY

We open on a wide shot of a busy airport. People are rushing about with their luggage, trying to catch their flights. Amidst the hustle and bustle, we see REBECCA, mid-30s, standing with her arms crossed and tapping her foot impatiently. A moment later, JACOB, late-30s, appears, out of breath.

JACOB: (Out of breath) Sorry, sorry. The traffic was insane.

REBECCA: You’re lucky I didn’t leave without you.

JACOB: (Smiling) You wouldn’t do that to me.

REBECCA: (Rolling her eyes) Don’t be so sure.

Jacob chuckles and the two make their way to the check-in counter.

JACOB: Are you excited?

REBECCA: (Sighs) I don’t know. I’m nervous.

JACOB: It’ll be fine. You’ll see.

REBECCA: (Smiling) I’m glad you’re here.

JACOB: (Playfully nudging her) Always.
The two check-in their luggage and make their way to their gate.

REBECCA: Do you have the tickets?

JACOB: (Pats his pockets, looking worried) Wait, I thought you had them.

REBECCA looks at him sternly, but then bursts out laughing.

REBECCA: I’m kidding. I have them. (Pulls out the 
tickets)

JACOB: (Shakes his head, smiling) You’re terrible.
Their flight is called and the two join the line to board the plane, ready for their next adventure.
Fade to black.   \\
\bottomrule
\end{tabular}
\caption{Corresponding responses to part of questions in Table \ref{tab:ranking_question}. Harder questions tend to have longer responses than simple and medium questions, but the distinction for simple and medium questions is not clear. } 
\label{tab:ranking_question_response}
\vspace{-0.3cm}
\end{table*}



\end{document}